\documentclass{article}

\usepackage[numbers]{natbib}
\usepackage{authblk}
\usepackage{hyperref}       

\usepackage[utf8]{inputenc} 
\usepackage[T1]{fontenc}    

\usepackage{url}            
\usepackage{booktabs}       
\usepackage{amsfonts}       
\usepackage{nicefrac}       
\usepackage{microtype}      
\usepackage{xcolor}         

\usepackage{comment}
\usepackage{graphicx}
\usepackage{epstopdf}
\usepackage{bm} 
\usepackage{rotating}
\usepackage{color}
\usepackage{amsfonts}
\usepackage{amsmath}
\usepackage{amssymb}
\usepackage{amsthm}
\usepackage{graphicx}
\usepackage{listings}
\usepackage{float}
\usepackage{makecell}
\usepackage{algorithm}
\usepackage{algpseudocode}
\usepackage{soul}
\usepackage[justification=centering]{caption}
\usepackage{subcaption}
\usepackage[normalem]{ulem}

\usepackage[textsize=tiny]{todonotes}
\usepackage{cleveref}
\usepackage{wrapfig}

\newtheorem{thm}{Theorem}[section]
\newtheorem{lemma}[thm]{Lemma}
\newtheorem{defn}{Definition}
\newtheorem{proposition}{Proposition}
\newtheorem{cor}[thm]{Corollary}

\def\E{\mathbb{E}}

\definecolor{codegreen}{rgb}{0,0.6,0}
\definecolor{codegray}{rgb}{0.5,0.5,0.5}
\definecolor{codepurple}{RGB}{50,157,168}
\definecolor{mycolor}{RGB}{51,177,255}
\definecolor{backcolour}{gray}{0.95}

\newcommand*\dd{\mathop{}\!\mathrm{d}}

\lstdefinestyle{mystyle}{
  backgroundcolor=\color{backcolour},   commentstyle=\color{codegreen},
  keywordstyle=\color{mycolor},
  numberstyle=\tiny\color{codegray},
  stringstyle=\color{codepurple},
  basicstyle=\ttfamily\footnotesize,
  breakatwhitespace=false,         
  breaklines=true,                 
  captionpos=b,                    
  keepspaces=true,                 
  numbers=left,                    
  numbersep=5pt,                  
  showspaces=false,                
  showstringspaces=false,
  showtabs=false,                  
  tabsize=2
}

\newcommand{\ab}[1]{\textcolor{red}{AB:  #1}}
\newcommand{\jz}[1]{\textcolor{blue}{JZ:  #1}}

\title{On the design of scalable, high-precision \\spherical-radial Fourier features}

\author{Ayoub Belhadji$^{\dagger}$\footnote{Corresponding author: abelhadj@mit.edu}}
\author{Qianyu Julie Zhu$^{\dagger}$}
\author{Youssef Marzouk$^{\dagger}$}

\affil{$\phantom{a}^{\dagger}$ Massachusetts Institute of Technology}

\begin{document}

\maketitle

\begin{abstract}
Approximation using Fourier features is a popular technique for scaling kernel methods to large-scale problems, with myriad applications in machine learning and statistics. This method replaces the integral representation of a shift-invariant kernel with a sum using a quadrature rule. The design of the latter is meant to reduce the number of features required for high-precision approximation. Specifically, for the squared exponential kernel, one must design a quadrature rule that approximates the Gaussian measure on $\mathbb{R}^d$. Previous efforts in this line of research have faced difficulties in higher dimensions. We introduce a new family of quadrature rules that accurately approximate the Gaussian measure in higher dimensions by exploiting its isotropy. These rules are constructed as a tensor product of a radial quadrature rule and a spherical quadrature rule. Compared to previous work, our approach leverages a thorough analysis of the approximation error, which suggests natural choices for both the radial and spherical components. We demonstrate that this family of Fourier features yields improved approximation bounds.

\end{abstract}

\section{Introduction}
Fourier features provide a powerful technique to overcome the scalability issues inherent in traditional kernel methods.
The very first instance of this approach was proposed by Rahimi and Recht \cite{Rahimi2007RFF}, which consists in approximating the integral that appears in the Bochner representation of a shift-invariant positive definite kernel $
\kappa: \mathbb{R}^{d} \times \mathbb{R}^{d} \rightarrow \mathbb{R}$:
\begin{equation}\label{eq:bochner_approx}
\kappa(x,y) = \int_{\mathbb{R}^{d}} \cos(\langle \omega,x-y\rangle) \mathrm{d} \Lambda(\omega) \approx \frac{1}{M}\sum_{j=1}^M \cos(\langle \omega_j,x-y\rangle),
\end{equation}

where $\Lambda$ is a probability measure supported on $\mathbb{R}^{d}$ and $\omega_{1}, \dots, \omega_{M}$ are i.i.d.~samples from $\Lambda$. In particular, in the case of the squared exponential kernel, the distribution $\Lambda$ corresponds to a Gaussian distribution on $\mathbb{R}^{d}$. 
However, several works have pointed out that the vanilla Monte Carlo approximation in~\eqref{eq:bochner_approx} is sub-optimal. For instance, quasi-Monte Carlo rules were shown to improve upon the Monte Carlo method in low-dimensional domains \cite{YaSiAvMa14,AvSiYaMa16}. Alternatively,  Gaussian quadrature was shown to significantly improve upon Monte Carlo methods in \cite{DaDeRe17,ShAv22}. However, attempting to adapt this approach to high dimensions through tensorization has yielded  poor empirical results. In contrast, \emph{spherical-radial quadrature rules} appear to offer better empirical performance \cite{MuKaBuOs18}. This construction makes use of the isotropy of the Gaussian distribution and is based on the transformation
\begin{equation}\label{eq:spherical_radial_decomposition_gaussian}
  \Big(\frac{\sigma^2}{2 \pi}\Big)^{d/2}  \int_{\mathbb{R}^{d}} \varphi(\omega) e^{-\frac{\sigma^2 \|\omega\|^2}{2}} \mathrm{d}\omega = \Big(\frac{\sigma^2}{2 \pi}\Big)^{d/2}\int_{\mathbb{R}_{+}} \int_{\mathbb{S}^{d-1}}  \varphi(r n) r^{d-1}e^{-\frac{\sigma^2 r^2}{2}} \mathrm{d} \pi_{\mathbb{S}^{d-1}}(n) \mathrm{d} r
\end{equation}
where $\phi(\omega) = \cos(\langle \omega,x-y\rangle)$. The r.h.s.\ of~\eqref{eq:spherical_radial_decomposition_gaussian} is then approximated by a tensor product of two quadrature rules: a radial quadrature rule and a spherical quadrature rule.
While this approach yields strong empirical results, there has been 
little systematic study of how to design the two quadrature rules. 
Indeed, the analysis given in~\cite{MuKaBuOs18} was restricted to radial quadrature rules of order $3$ and the spherical quadrature rules were restricted to a very specific family of constructions. Moreover, the importance of the bandwidth $\sigma$ was overlooked.





In this work, we propose a refined analysis of spherical-radial Fourier features when used in approximating the squared exponential kernel. Our approach makes use of a decomposition of the approximation error that shows the error of each component (spherical or radial). Moreover, we use a change of variable $\xi = \sigma^2 r^2 /2$ in \eqref{eq:spherical_radial_decomposition_gaussian} to obtain a slightly different transformation
\begin{equation}\label{eq:spherical_radial_decomposition_gaussian_2}
  \Big(\frac{\sigma^2}{2 \pi}\Big)^{d/2}  \int_{\mathbb{R}^{d}} \varphi(\omega) e^{-\frac{\sigma^2 \|\omega\|^2}{2} }\mathrm{d}\omega = \frac{1}{2}\Big(\frac{1}{2 \pi}\Big)^{d/2}\int_{\mathbb{R}_{+}} \int_{\mathbb{S}^{d-1}}  \varphi\Big(\frac{\sqrt{2\xi}}{\sigma} n \Big) \xi^{d/2-1}e^{-\xi} \mathrm{d} \pi_{\mathbb{S}^{d-1}}(n) \mathrm{d} \xi .
\end{equation}
Using this modification, we show that we can take advantage of the Gaussian quadrature associated to the weight function 
\begin{equation}\label{eq:radial_density}
p_{\Xi}(\xi) =\frac{1}{\Gamma(d/2)} \xi^{d/2 -1} e^{-\xi} \chi_{[0,+\infty[}(\xi),
\end{equation}
for which the algorithmic construction is classical. Moreover, this approach allows us to study the importance of the number of nodes in the radial quadrature rule by studying the coefficients of the function $\xi \mapsto \int_{\mathbb{S}^{d-1}} \cos\big( \langle \sqrt{2 \xi} n, x-y \rangle /\sigma \big) \mathrm{d} \pi_{\mathbb{S}^{d-1}}(n)$ in the orthonormal basis formed by the generalized Laguerre polynomials, which is a well-studied family of orthogonal polynomials. Similarly, we study the coefficients of the functions $n \mapsto \cos\big( \langle \sqrt{2 \xi} n, x-y \rangle /\sigma \big)$ in the basis of the spherical harmonics defined on $\mathbb{S}^{d-1}$. This refined analysis of the coefficients enables us to highlight the difference in approximation error between two families of spherical quadrature rules: one based on Monte Carlo (SR-MC) on $\mathbb{S}^{d-1}$ and the second based on i.i.d.\ orthogonal matrices (SR-OMC). Combining the two analyses allows us to derive upper bounds for the expected value of the squared approximation error that show the importance of the number of both the radial quadrature rule and the spherical quadrature rule, the bandwidth $\sigma$, and the dimension $d$. Contrary to \cite{MuKaBuOs18}, our analysis gives insight on how these parameters influence the approximation error.

Besides these theoretical findings, our numerical simulations show that SR-OMC delivers strong empirical results, especially  when the number of features $M$ is close to the dimension $d$. Moreover, the numerical simulations show how some interpolative quadrature rules on $\mathbb{S}^{d-1}$ could yield even better quadrature rules that are worth studying in the future.

\paragraph{Notation}
Denote by $\kappa_{\sigma}$ the Gaussian kernel on $\mathbb{R}^{d} \times \mathbb{R}^{d}$ defined by 
\begin{equation}
 \kappa_{\sigma}(x,y) := e^{-\frac{\|x-y\|^2}{2 \sigma^2}}.   
\end{equation}
Moreover, denote by $\Lambda_{\sigma}$ the Gaussian density corresponding to the Fourier transform of $\kappa_{\sigma}$, given by
\begin{equation}
\Lambda_{\sigma}(\omega) := \Big(\frac{\sigma^2}{2 \pi}\Big)^{d/2} e^{- \frac{\sigma^2 \|\omega \|^2}{2}}.
\end{equation}
%
%
We denote a quadrature rule $Q$ defined on the space of continuous function on some domain  $\mathcal{X}$ as a measure: $Q = \sum_{i=1}^{M} w_{i} \delta_{n_i}$, where the $w_i$ are the weights and the $n_{i}$ are the nodes. Given a quadrature rule of non-negative weights, defined on $\mathbb{R}^{d}$, define the
approximate kernel $\hat{\kappa} : \mathbb{R}^d \times \mathbb{R}^d \rightarrow \mathbb{R}$ associated to $Q$ to be 
\begin{equation}\label{eq:kappa_hat}
\hat{\kappa}(x,y) := \sum\limits_{i=1}^{M} w_{i} \cos \big( \langle n_{i}, x-y \rangle \big).
\end{equation}

This article is structured as follows. \Cref{sec:maiM_{R}esults} is dedicated to the details of our constructions along with the main theoretical results. \Cref{sec:related_work} contains an overview of existing Fourier feature approximations built using quadrature rules. Numerical simulations illustrating the theoretical results are gathered in~\Cref{sec:num_sim}.

\section{Main results}\label{sec:maiM_{R}esults}

In this section, we detail our theoretical contributions. In \Cref{sec:the_construction}, we define the family of quadrature rules that we adopt for constructing Fourier features and present the main theoretical results. In \Cref{sec:on_the_design}, we justify our choices of the radial and spherical quadrature rules, along with their respective approximation errors. 

\subsection{A family of spherical-radial quadrature rules}\label{sec:the_construction}
We start with the construction of the approximate kernel $\hat{\kappa}$, for which a study of the approximation error will be given later.
Let $Q^{R}= \sum_{i=1}^{M_R} a_{i} \delta_{\xi_i}$ be the Gaussian quadrature associated to \eqref{eq:radial_density}.
Given a \emph{spherical quadrature rule} $Q^{S} = \sum_{j=1}^{M_S} b_{j} \delta_{\theta_j}$, where $b_j \in \mathbb{R}_{+}$ and $\theta_{j} \in \mathbb{S}^{d-1}$, we consider the approximate kernel $\hat{\kappa}$ associated to the 
 quadrature rule 
\begin{equation}
\sum\limits_{i = 1}^{M_R} \sum\limits_{j = 1}^{M_S} a_{i}b_{j} \delta_{r_{i}\theta_{j}}; \quad \quad  r_{i} := \frac{\sqrt{2\xi_i}}{\sigma}.
\end{equation}


The calculation of the weights $a_i$ and the nodes $\xi_i$ can be performed classically through the eigendecomposition of the tridiagonal Jacobi matrix associated with the orthogonal polynomials with respect to $p_{\Xi}$, which coincide with generalized Laguerre polynomials. This classical approach is briefly reviewed in \Cref{sec:app_laguerre}.
One could argue for the use of the Gauss-Hermite quadrature rule in the initial spherical-radial transformation~\eqref{eq:spherical_radial_decomposition_gaussian}. 
However, this approach becomes highly unstable in higher dimensions due to the calculation of the products $r_{i}^{d-1}$, where the $r_{i}$ are the non-negative nodes of the Gauss-Hermite quadrature, which can take extremely large values as the dimension $d$ increases. Otherwise, we could use the Gaussian quadrature associated to the density $\propto r^{d-1}e^{-r^2/2}$ in the initial transformation \eqref{eq:spherical_radial_decomposition_gaussian}. However, this entails the manipulation of polynomials which are orthogonal with respect to this density and for which few properties are known.


\subsection{On the design of the spherical-radial quadrature rule}\label{sec:on_the_design}
As was shown in \cite{MuKaBuOs18}, there are multiple ways to design a radial-spherical quadrature rule. In this section, we discuss a principled way to design such a quadrature rule, which justifies the choice adopted in~\Cref{sec:the_construction}. An important ingredient in the analysis is the following definition.
\begin{defn}
Let $x,y \in \mathbb{R}^{d}$. Define $f_{x-y}$ on $\mathbb{R}^d$ by
\begin{equation}
    f_{x-y}(\omega) := \cos \big( \langle \omega, x-y \rangle \big),
\end{equation}
and its spherical average $\bar{f}_{x-y}$ defined on $\mathbb{R}_{+}$ by
\begin{equation}\label{eq:spherical_average_def}
    \bar{f}_{x-y}(\xi):= \int_{\mathbb{S}^{d-1}} f_{x-y}\Big(\frac{\sqrt{2\xi}}{\sigma} n\Big) \mathrm{d} \pi_{\mathbb{S}^{d-1}}(n). 
\end{equation}
\end{defn}


The crux of our approach is the observation that the squared error
\begin{equation*}
    \big|\kappa(x,y) - \hat{\kappa}(x,y) \big|^2 = \Big|\int_{\mathbb{R}^{d}} f_{x-y}(\omega) \mathrm{d} \Lambda_{\sigma}(\omega) - \sum\limits_{i=1}^{M_R}\sum\limits_{j=1}^{M_S} a_{i} b_{j} f_{x-y}\big(r_i \theta_{j} \big)\Big|^2
\end{equation*}
is upper bounded by the sum of two terms
\begin{equation*}
    2\Big(\underbrace{\Big| \int_{\mathbb{R}_{+}} \bar{f}_{x-y}(\xi) p_{\Xi}(\xi) \mathrm{d}\xi - \sum\limits_{i=1}^{M_R} a_{i} \bar{f}_{x-y}(\xi_i) \Big|^2}_{\textbf{radial error}} + \underbrace{\sum\limits_{i \in [M_R]}a_i\Big|\bar{f}_{x-y}(\xi_i) - \sum\limits_{j=1}^{M_S} b_{j} f_{x-y}\Big(\frac{\sqrt{2\xi_i}}{\sigma} \theta_{j} \Big) \Big|^2 }_{\textbf{spherical error}} \Big).
\end{equation*}

In other words, to upper bound the approximation error $|\kappa(x,y) - \hat{\kappa}(x,y)|^2$, it is sufficient to establish upper bounds for the approximation errors of both the radial and spherical quadrature rules. The design of an efficient quadrature rules boils down to understanding the structure and the smoothness of the integrands. We discuss this point in detail next.









\subsubsection{The radial quadrature rule}\label{sec:radial_quadrature_design}
The radial quadrature rule has to be tailored to the functions $\bar{f}_{x-y}$, whose expression is given in the following result.

\begin{proposition}\label{prop:f_bar_representation}
For $x,y \in \mathbb{R}^{d}$, we have 
\begin{equation}
    \bar{f}_{x-y}(\xi) =\sum_{n=0}^{+\infty}\beta_n(d, x-y)\xi^n,
\end{equation}
where
\begin{equation}\label{eq:def_beta}
    \beta_n(d,x-y) := \frac{(-\|x-y\|^2/(2\sigma^{2}))^n\Gamma\left(\frac{d}{2}\right)}{\Gamma(n+1)\Gamma\left(d/2+n\right)}.
\end{equation}
\end{proposition}


In other words, $\bar{f}_{x-y}$ is an analytic function despite the fact that the expression \eqref{eq:spherical_average_def} involves a dependency on the square root of $\xi$.
The Gaussian quadrature is an effective method for approximating uni-dimensional integrals that involve `simple' weight functions and smooth integrands. In our case, we have to deal with the weight function \eqref{eq:radial_density},
for which the generalized Gauss-Laguerre quadrature rule offers the most natural quadrature rule.
We show in the following, that the functions $\bar{f}_{x-y}$ are well approximated in the orthonormal basis of the generalized Laguerre polynomials $(\ell^{\alpha}_m)_{m \in \mathbb{N}}$, where $\alpha := d/2 -1$, which is orthogonal with respect to the scalar product $\langle . , . \rangle_{p_{\Xi}}$. A definition of these polynomials is given in \Cref{sec:app_laguerre}. 
\begin{proposition}\label{prop:radial_quadrature_error}
Let $x,y \in \mathbb{R}^{d}$. We have 
\begin{equation}
\big| \langle \bar{f}_{x-y}, \ell_{m}^{\alpha} \rangle_{p_{\Xi}} \big| \leq \sqrt{\frac{1}{m!\,\Gamma(m+d/2)}} c^{2m}e^{-c^2}; \:\: c:= \frac{\|x-y\|}{\sqrt{2} \sigma}.
\end{equation}
In particular, we have 
\begin{equation}
\forall M \in \mathbb{N}, \:\:  \sum\limits_{m = 2M+1}^{+\infty} m |\langle \bar{f}_{x-y}, \ell_{m}^{\alpha} \rangle_{p_{\Xi}}| \leq \frac{c^2}{\sqrt{\Gamma(d/2)}} \Big(\frac{c^{2}}{2M-1}\Big)^{2M-1} ,
\end{equation}
%
and $\exists L>0$ which is independent of the choice of $M_R$, $x$ and $y$ such that
\begin{equation}\label{eq:laguerre_quad_err_bound}
    \Big|\int_{\mathbb{R}_{+}} \bar{f}_{x-y}(\xi) p_{\Xi}(\xi) \mathrm{d}r - \sum\limits_{i=1}^{M_R} a_{i} \bar{f}_{x-y}(\xi_{i}) \Big| \leq L \frac{c^2}{\sqrt{\Gamma(d/2)}}  \Big(\frac{c^{2}}{2M_R-1}\Big)^{2M_R-1}.
\end{equation}

\end{proposition}
The proof of \Cref{prop:radial_quadrature_error} is given in \Cref{proof:radial_quadrature_error}. 



\subsubsection{The spherical quadrature rule}\label{sec:spherical_quadrature_rule}
For the design of the spherical quadrature rule, we adopt a strategy similar to that in \Cref{sec:radial_quadrature_design}, and we study how well the functions $f_{r_{i}(x-y)}$, seen as functions defined on \(\mathbb{S}^{d-1}\), are approximated in the orthonormal basis of the spherical harmonics. The latter is the most natural basis for representing functions on $\mathbb{S}^{d-1}$. We start with the following result.
\begin{proposition}\label{prop:harmonic_decomposition_f}
Let $x,y \in \mathbb{R}^{d}$, and $r \in \mathbb{R}_{+}$. We have
    \begin{equation}\label{eq:harmonic_decompose_f}
    f_{r(x-y)}(n) = \sum_{k=0}^{+\infty} N(d,k) \lambda_{k} P_{k}(\langle v, n \rangle),\:\: v:= \frac{x-y}{\|x-y\|}
\end{equation}
where the $P_{k}$ are the Gegenbauer polynomials defined in~\Cref{sec:spherical_harmonics}, and the $\lambda_{k}$ are defined as
\begin{equation}
    \lambda_{k} := \frac{\Gamma\left(d/2\right)}{\sqrt{\pi}\Gamma\left(d/2-1/2\right)} \int_{-1}^{1} \cos(r\|x-y\|t) P_{k}(t) (1-t^2)^{(d-3)/2} \mathrm{d} t,
\end{equation}
and satisfies
\begin{equation}\label{eq:lambda_k_bound}
\forall k \in \mathbb{N}, \:\: |\lambda_{k}| \leq 
\frac{\Gamma((d-1)/2)}{\Gamma(k+(d-1)/2)}\bigg(\frac{r \|x-y\|}{2}\bigg)^{k}.
\end{equation}
\end{proposition}
The proof of~\Cref{prop:harmonic_decomposition_f} is given in~\Cref{sec:proof_prop:harmonic_decomposition_f}. In the following, we study the choice of the spherical quadrature rule.

\paragraph{Monte Carlo on $\mathbb{S}^{d-1}$} Taking the spherical weights $b_j = 1/M_{S}$ and the spherical nodes $\theta_{1}, \dots, \theta_{M_{S}}$ to be i.i.d.\ samples from the uniform distribution on $\mathbb{S}^{d-1}$, we get the following result.



\begin{proposition}\label{prop:iid_spherical_quadrature}
Let $x,y \in \mathbb{R}^{d}$, and let $r \in \mathbb{R}_{+}$. The expected squared error under Monte Carlo sampling on $\mathbb{S}^{d-1}$ applied to the function $f_{r(x-y)}$ is bounded as follows

\begin{equation}\label{eq:spherical_MC_bound}
   \mathbb{E} \Big| \int_{\mathbb{S}^{d-1}} f_{r(x-y)}(n) \mathrm{d}\pi_{\mathbb{S}^{d-1}}(n) - \frac{1}{M_S} \sum\limits_{j=1}^{M_S} f_{r(x-y)}(\theta_j) \Big|^2 = \frac{1}{M_{S}}   
     \sum\limits_{k = 2}^{+\infty} N(d,k) \lambda_{k}^2.
\end{equation}
In particular, we have 
\begin{equation}\label{eq:spherical_MC_bound__2}
   \mathbb{E} \Big| \int_{\mathbb{S}^{d-1}} f_{r(x-y)}(n) \mathrm{d}\pi_{\mathbb{S}^{d-1}}(n) - \frac{1}{M_S} \sum\limits_{j=1}^{M_S} f_{r(x-y)}(\theta_j) \Big|^2 \leq  \frac{1}{2M_S}\Big(\frac{r^2 \|x-y\|^2}{d-1} \Big)^2 e^{\frac{r^2 \|x-y\|^2}{d-1}} .
\end{equation} 

\end{proposition}
The proof of \Cref{prop:iid_spherical_quadrature} is given in \Cref{proof:prop_iid_spherical_quadrature}.
The upper bound \eqref{eq:spherical_MC_bound__2} converges to $0$ at the typical Monte Carlo rate, and the constant depends on the ratio $r \|x-y\|/\sqrt{d-1}$. Now, remember that in our construction, $r \leq \sqrt{2\xi_{M_R}}/\sigma$, where $\xi_{M_R}$ is the largest node of the Gaussian quadrature associated to the weight function defined by \eqref{eq:radial_density}. By 
Theorem 6.31.2. in \cite{Sze39} we have
\begin{equation}\label{eq:szego_bound}
\forall i \in [M_{R}],   \xi_i \leq 4 M_{R} +2 \frac{d-2}{2} + 2 = 4M_{R} +d,
\end{equation}
and we get the following result.



\begin{thm}\label{thm:main_theorem_MC}
Let $\kappa$ be the squared exponential kernel of bandwidth $\sigma$, and let $\hat{\kappa}$ be the empirical kernel of the spherical-radial Fourier feature map, as defined in~\Cref{sec:the_construction}, associated to the Monte Carlo approximation on $\mathbb{S}^{d-1}$.
Then, we have
\begin{equation}\label{eq:upper_bound_error_thm_1}
    \mathbb{E} \big|\kappa(x,y) - \hat{\kappa}(x,y) \big|^2 \leq 2 \Bigg( L \frac{c^2}{\sqrt{\Gamma(d/2)}}  \Big(\frac{c^{2}}{2M_R-1}\Big)^{2M_R-1} +  \frac{8}{M_{S}}  \left(\frac{4M_{R} + d}{d-1}\right)^2c^4 e^{\frac{4 (4M_{R} + d)  }{d-1} c^2} \Bigg),
\end{equation}
where $c = \|x-y\|/(\sqrt{2}\sigma)$, and $L$ is the same constant as in \Cref{prop:radial_quadrature_error}.


    


\end{thm}
The upper bound in~\eqref{eq:upper_bound_error_thm_1} underscores the importance of the parameters $M_R$, $M_S$ and $\sigma$. Increasing the number of nodes in the radial quadrature, $M_R$, reduces the radial error. However, to prevent a significant rise in spherical error, $M_R$ should be kept as small as possible. Furthermore, the spherical part of the approximation error depends on the ratio $(4 (4M_{R} + d) c^2 )/(d-1)$, which scales as $c^2$ for a fixed value of $M_R$ and large value of $d$. 
In the following, we study how we can improve the design of the spherical quadrature rule.

\begin{figure}
    \centering
    \includegraphics[width=\textwidth]{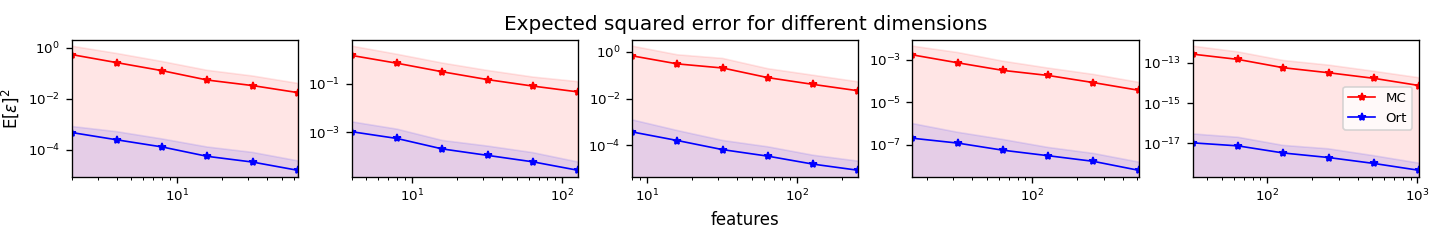}
    \vspace{-15pt}
    \caption{
    Monte Carlo on $\mathbb{S}^{d-1}$ versus Orthogonal Monte Carlo on $\mathbb{S}^{d-1}$. \\
    Dimensions from left to right: 2, 4, 8, 16, 32.}
    \label{fig:MC_vs_OMC} 
\end{figure}

\paragraph{The orthogonal Monte Carlo quadrature on $\mathbb{S}^{d-1}$}
In this section, we assume that $M_{S}$ is a multiple of $d$, and we take $\theta_{1}, \dots, \theta_{M_{S}}$ to be the columns of the matrix obtained by the concatenation of $M_{S}/d$ orthonormal i.i.d.\ random matrices $\bm{B}_{1}, \dots, \bm{B}_{M_{S}/d}$ from the Haar distribution of $\mathbb{O}_{d}(\mathbb{R})$, so that the columns of each $\bm{B}_{\ell}$ follows the distribution of the uniform distribution on $\mathbb{S}^{d-1}$. By taking uniform weights equal to $1/M_{S}$, we define a quadrature rule that we call the \emph{orthogonal Monte Carlo quadrature on $\mathbb{S}^{d-1}$}. For such a construction, we have the following result.

\begin{proposition}\label{prop:block_iid_spherical_quadrature}
Let $x,y \in \mathbb{R}^{d}$, and let $r \in \mathbb{R}_{+}$. The expected squared error under the orthogonal Monte Carlo quadrature on $\mathbb{S}^{d-1}$ applied to the function $f_{r(x-y)}$ is bounded as follows 
\begin{equation}\label{eq:block_iid_spherical_quadrature}
   \mathbb{E} \Big| \int_{\mathbb{S}^{d-1}} f_{r(x-y)}(n) \mathrm{d}\pi_{\mathbb{S}^{d-1}}(n) - \frac{1}{M_S} \sum\limits_{j=1}^{M_S} f_{r(x-y)}(\theta_j) \Big|^2  \leq \frac{3}{M_{S}}   
     \sum\limits_{k = 4}^{+\infty} N(d,k) \lambda_{k}^2.
\end{equation}
In particular, we have
\begin{equation}\label{eq:block_iid_spherical_quadrature_bis}
   \mathbb{E} \Big| \int_{\mathbb{S}^{d-1}} f_{r(x-y)}(n) \mathrm{d}\pi_{\mathbb{S}^{d-1}}(n) - \frac{1}{M_S} \sum\limits_{j=1}^{M_S} f_{r(x-y)}(\theta_j) \Big|^2  \leq \frac{1}{16M_S}(\frac{r^2\|x-y\|^2}{d-1})^4 e^{\frac{r^2\|x-y\|^2}{d-1}}.
\end{equation}

\end{proposition}

The upper bound in \Cref{prop:block_iid_spherical_quadrature} converges to $0$ at the same rate as in~\eqref{eq:spherical_MC_bound}, which is $1/M_{S}$. However, the involved constant is lower, which is crucial in our context where $M_{S} = m_{S} d$, and $m_{S}$ is not very large. Indeed, the ratio between the r.h.s. of \eqref{eq:block_iid_spherical_quadrature_bis} and the r.h.s of \eqref{eq:spherical_MC_bound__2} is equal to 
 \begin{equation}
      \frac{1}{8} \Big(\frac{r^{2} \|x-y\|^2}{d-1}\Big)^2,
\end{equation}


and we get an improvement over \Cref{thm:main_theorem_MC} (the Monte Carlo approximation) when $\|x-y\|^2 \leq 2(d-1) r^{-2}$.
\Cref{fig:MC_vs_OMC} illustrates this improvement for $\|x-y\| = 1$ and for $d \in \{2,4,8,16,32\}$. 
\Cref{prop:block_iid_spherical_quadrature} leads to the following result.


\begin{thm}\label{thm:main_theorem_OMC}
Let $\kappa$ be the squared exponential kernel of bandwidth $\sigma$, and let $\hat{\kappa}$ be the empirical kernel of the spherical-radial Fourier feature map, as defined in~\Cref{sec:the_construction}, associated to the orthogonal Monte Carlo quadrature on $\mathbb{S}^{d-1}$.
Then, we have
\begin{equation}\label{eq:main_theorem_OMC}
    \mathbb{E} \big|\kappa(x,y) - \hat{\kappa}(x,y) \big|^2 \leq 2 \Bigg(  L \frac{c^2}{\sqrt{\Gamma(d/2)}}  \Big(\frac{c^{2}}{2M_R-1}\Big)^{2M_R-1} + \frac{2}{M_{S}}  \Big(\frac{4M_{R} + d}{d-1}\Big)^4 c^8 e^{\frac{4 (4M_{R} + d)  }{d-1} c^2}  \Bigg).
\end{equation}

    


\end{thm}

Despite the fact that the approximation error of this method converges to zero at a rate similar to the Monte Carlo rate, the numerical simulations presented in~\Cref{sec:num_sim} demonstrate the superiority of this method. Compared to \Cref{thm:main_theorem_MC}, the error term depends on the ratio $c$ raised to the power of 4, which yields a significant improvement when $c \leq 1$. The proof of \Cref{prop:block_iid_spherical_quadrature} is given in \Cref{proof:block_iid_spherical_quadrature}.



\paragraph{The optimal kernel quadrature}
The previous constructions, which make use of uniform weights, can be combined with a procedure that optimizes the weights in some sense. Given a configuration of nodes, one could seek the weights that minimizes the worst-case approximation error (WCE) of the quadrature on the unit ball of an RKHS $\mathcal{H}_{S}$ associated to a positive definite kernel $\kappa_{\mathbb{S}^{d-1}}: \mathbb{S}^{d-1} \times \mathbb{S}^{d-1} \rightarrow \mathbb{R}$. This quantity is equal to 
\begin{equation}\label{eq:WCE}
    \Big\| \mu - \sum\limits_{i=1}^{M_{S}}w_{i}\kappa_{\mathbb{S}^{d-1}}(\theta_{i},.) \Big\|_{\mathcal{H}_S}
\end{equation}
where $\mu \in \mathcal{H}$ is defined by $\mu(.) := \int_{\mathbb{S}^{d-1}} \kappa_{\mathbb{S}^{d-1}}(.,\theta) \mathrm{d}\pi_{\mathbb{S}^{d-1}}(\theta)$. The vector $\hat{\bm{w}}$ of the weights $w_i$ that minimizes \eqref{eq:WCE} is given by $\hat{\bm{w}} = \bm{K}_{S}(\bm{\theta})^{-1} \mu(\bm{\theta})$, where $\bm{K}_{S}(\bm{\theta}):= (\kappa_{\mathbb{S}^{d-1}}(\theta_{i},\theta_j))_{i,j \in [M_{S}]} \in \mathbb{R}^{M_{S} \times M_{S}}$, and $\mu(\bm{\theta}) := (\mu(\theta_i))_{i \in [M_{S}]} \in \mathbb{R}^{M_{S}}$. When $\kappa_{\mathbb{S}^{d-1}}$ is rotation-invariant, $\mu$ is a constant function with an explicit formula. 
The study of the convergence of the resulting quadrature rule, also called \emph{the optimal kernel-based quadrature}, was conducted for random configurations of points \cite{Bac17,EhGrOa19, BeBaCh19, BeBaCh20,Bel21}. 
In particular, if the kernel $\kappa_{\mathbb{S}^{d-1}}$ is the squared exponential kernel, the defined quadrature rule is guaranteed to converge at an exponential rate. However, the design of this quadrature rule requires a sharp theoretical analysis that allows tuning the bandwidth of the kernel $\kappa_{\mathbb{S}^{d-1}}$, and the resulting weights are not guaranteed to be positive.

\section{Related work}\label{sec:related_work}
In this section we compare our constructions with existing work. 

\paragraph{Vanilla random Fourier features (RFF)}
The frequencies used in the initial form of Fourier features are $M$ i.i.d.\ samples from the distribution $\Lambda_{\sigma}$. The study of the approximation error of this method gave birth to an abundant literature \cite{SrSz15,SuSc15}; see \cite{LiHuChSuJo21} for a survey. In particular, these works are concerned with the study of the uniform error bound 
\begin{equation}
    \sup\limits_{\|x-y\| \leq B} |\kappa(x,y)- \hat{\kappa}(x,y)|,
\end{equation}
for a given $B>0$. The uniform error bound was proved when $M = \Omega(d/\sigma)$ with constants that depends on $\log (B)$.  
In practice, when $M$ gets closer to the dimension $d$, the approximation deteriorates strongly, as we will show in \Cref{sec:num_sim}.



\paragraph{QMC Fourier features}
Methods based on quasi-Monte Carlo rules make use of the fact that the density $\Lambda_{\sigma}$ can be factorized $\Lambda_{\sigma} = \prod_{i=1}^d \Lambda_{\sigma,i}$, and take the $i$-th component of the frequency $\omega_{m} \in \mathbb{R}^{d}$ to be equal to $\Phi^{-1}(v_{m,i})$, where $(v_{m,.})_{m \in \mathbb{N}^*}$ is a low-discrepancy sequence, and $\Phi$ is the cumulative distribution function of the Gaussian \cite{YaSiAvMa14,AvSiYaMa16,HuSuHu2024}.  As shown in \cite{HuSuHu2024}, when $(v_{m,.})_{m \in \mathbb{N}^*}$ corresponds to the Halton sequence, we get a uniform bound of the form
\begin{equation}
    \sup\limits_{\|x-y\| \leq B} |\kappa(x,y)- \hat{\kappa}(x,y)| = O\Big( \frac{\log(M)^{d}}{M} \Big),
\end{equation}
which is a typical rate of convergence for QMC rules: the improvement over the Monte Carlo only happens when $M$ is exponential on the dimension $d$. This is corroborated by numerical simulations that show that there is no advantage in using QMC Fourier features in high dimensional settings. 

\paragraph{Stochastic spherical rules and orthogonal Fourier features}



Another class of closely-related methods are stochastic spherical rules (SSR) introduced by Munkhoeva et al.\ in \cite{MuKaBuOs18}, which are based on the spherical-radial decomposition \cite{GeMo99}. The corresponding quadrature rule writes, for a continuous function $\varphi :\mathbb{R}^{d} \rightarrow \mathbb{R}$, as follows
\begin{equation}\label{eq:ssr_estimator}
  \frac{1}{\tilde{M}} \sum\limits_{i_t = 1}^{\tilde{M}}  \sum\limits_{i_r=1}^{M_R} \sum\limits_{i_s=1}^{M_S}a_{i_r}b_{i_s} \frac{\varphi(r_{i_t}Q_{i_t}z_{i_s})+ \varphi(-r_{i_t}Q_{i_t}z_{i_s})}{2},
\end{equation}
where the $z_{i_s} \in \mathbb{S}^{d-1}$, and the $r_{i_t}$ are i.i.d.\  samples from the distribution supported on the real line, and the $Q_{i_t}$ are i.i.d.\ samples from the Haar distribution of $\mathbb{O}_{d}(\mathbb{R})$. This construction ~\eqref{eq:ssr_estimator} was shown to extend orthogonal Fourier features (ORF), which is a popular sampling scheme that reduces the approximation error by enforcing orthogonal structure on the features\cite{YuSuChHoKu16,ChRoWe17}. There are multiple ways of choosing the weights $a_{i_r}$ and $b_{i_s}$. The adopted strategy in \cite{MuKaBuOs18} was to calculate the weights in such a way that the quadrature rule is exact for low-degree polynomials, so that the quadrature~\eqref{eq:ssr_estimator} yield an unbiased estimator of the integral $\int_{\mathbb{R}^{d}} \varphi(\omega) \Lambda(\omega) \mathrm{d}\omega $ \cite{MuKaBuOs18}. While they showed this scheme performs better than state-of-the-art methods, no comprehensive analysis has been done to derive the error bound when the number of radial nodes is larger than $3$. Moreover, the design of the $z_i$ was restricted: the $z_i$ were taken to belong to a regular $d$-simplex, and the sensitivity of the method with respect to this choice is not clear.

\section{Numerical simulations}\label{sec:num_sim}

In this section, we present numerical simulations that corroborate our theoretical findings. First, in \Cref{sec:num_sim_1}, we study the influence of the choice of the spherical quadrature rule along with the influence of the order of the radial quadrature rule. Then, in \Cref{sec:num_sim_2}, we compare our constructions with existing methods for entry-wise and spectral kernel matrix approximation. Finally, in \Cref{sec:num_sim_3}, we compare our constructions with existing methods for two learning tasks. The code for this paper is available at \href{https://github.com/qianyu-zhu/SRFF}{\texttt{https://github.com/qianyu-zhu/SRFF}}.

\subsection{The importance of the spherical and radial designs}\label{sec:num_sim_1}

In this section, we study the influence of the design of the spherical quadrature rule and the number of nodes of the radial quadrature rule. For this purpose, we conduct two experiments on the dataset Powerplant $(d=4)$. 
The first experiment compares vanilla RFF and stochastic spherical-radial quadratures (SSR) as defined in \cite{MuKaBuOs18}, to six spherical-radial (SR) quadrature rules as defined in \Cref{sec:the_construction}: (1) SR-MC takes $Q^S$ to be Monte Carlo on $\mathbb{S}^{d-1}$; (2) SR-OMC takes $Q^S$ to be Orthogonal Monte Carlo on $\mathbb{S}^{d-1}$ as defined in \Cref{sec:spherical_quadrature_rule}; (3) SR-SOMC takes $Q^S$ to be symmetrized Orthogonal Monte Carlo on $\mathbb{S}^{d-1}$, where $M_S$ is an even integer and the $\theta_{i}$ are the columns of the matrix obtained by the concatenation of $\bm{B}_{1}, -\bm{B}_{1}, \dots, \bm{B}_{M_S/(2d)}, -\bm{B}_{M_S/(2d)}$ and the $\bm{B}_{i}$ are i.i.d.\ matrices from the Haar distribution of $\mathbb{O}_{d}(\mathbb{R})$; (4) SR-OKQ-MC is the optimal kernel quadrature (OKQ) associated with the nodes of SR-MC, where the kernel $\kappa_{\mathbb{S}^{d-1}}$ is taken to be the Gaussian kernel on $\mathbb{S}^{d-1}$; (5) SR-OKQ-OMC is the optimal kernel quadrature associated with the nodes of SR-OMC; (6) SR-OKQ-SOMC is the optimal kernel quadrature associated with the nodes of SR-SOMC. For this experiment, we randomly select $5000$ samples from the dataset and report the relative error $\|\bm{K}-\hat{\bm{K}}\|_{\mathrm{Fr}}/\|\bm{K}\|_{\mathrm{Fr}}$ averaged over $20$ runs. The kernel bandwidth is chosen to be $\sigma = 2.83$.





\begin{wrapfigure}{r}{0.5\textwidth}
    \includegraphics[width=\linewidth]{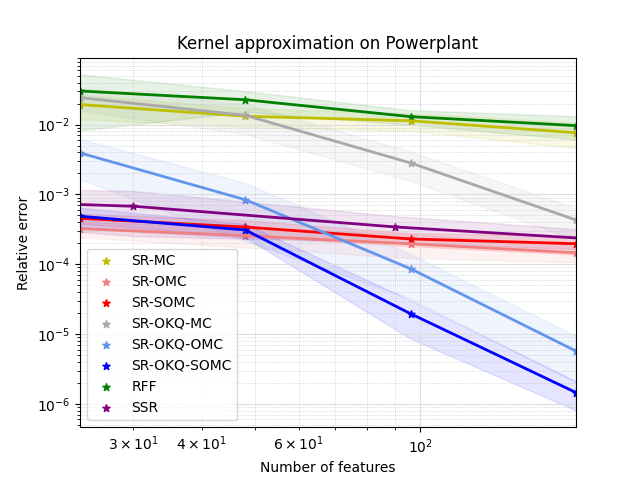}
\caption{Relative error of different kernel approximation schemes for the dataset Powerplant. Shaded regions indicate sample standard deviation of the relative error, computed over 20 independent runs of each method. Radial nodes $M_R$ are fixed and the number of spherical nodes $M_S$ changes.} 
\label{fig:sim_1}
\end{wrapfigure}

\Cref{fig:sim_1} shows the results for the different quadrature methods as the number of features increases. We observe that SR-MC has a similar performance compared to RFF, while optimizing the weights in SR-OKQ-MC gives slightly better results. On the other hand, SR-OMC yields significantly better results compared to RFF, and slightly better than SSR. The symmetrization of the orthogonal matrix in SR-SOMC yields slightly worse performance than SR-OMC. Finally, optimizing the weights in SR-OKQ-OMC leads to better results when the number of features is large, and combining the optimization of weights and the symmetrization of the nodes in SR-OKQ-SOMC yields the best performance compared to every other method. The experiment shows that enforcing orthogonality in the spherical quadrature rule notably reduces the approximation error, which is in line with our theoretical analysis in \Cref{sec:spherical_quadrature_rule}. Moreover, it shows that optimizing the weights through OKQ significantly improves the rates of convergence and yields exponential convergence, as predicted by existing theoretical results in the literature.

\begin{figure}[h]
    \begin{subfigure}[b]{.33\linewidth}
  \centering
    \includegraphics[width=.99\linewidth]{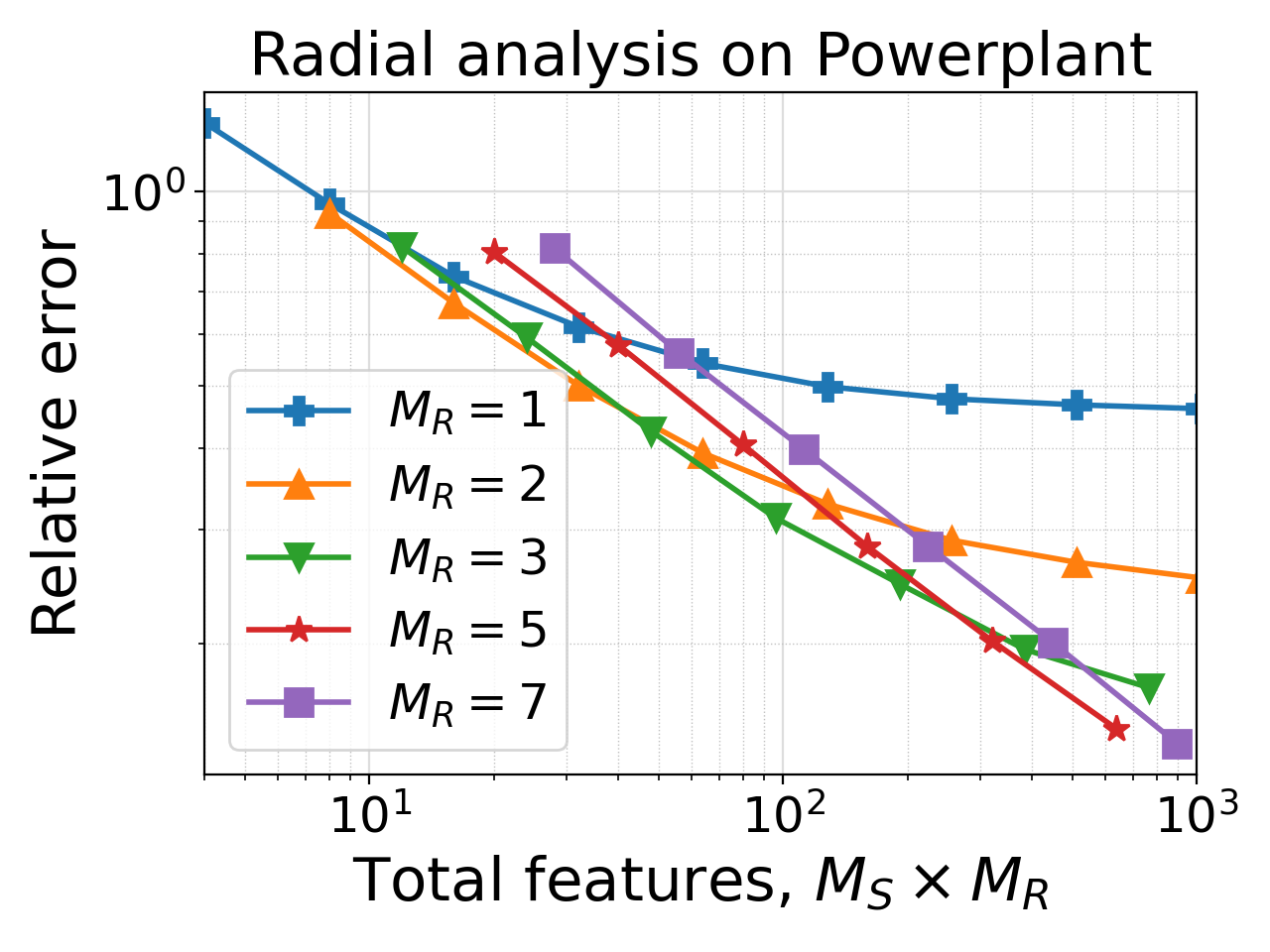}
    \vspace{-10pt}
    \caption{$\sigma = 0.28$}
    \label{fig:radial_design_smallbandwidth}
  \end{subfigure} 
\begin{subfigure}[b]{.33\linewidth}
  \centering
    \includegraphics[width=.99\linewidth]{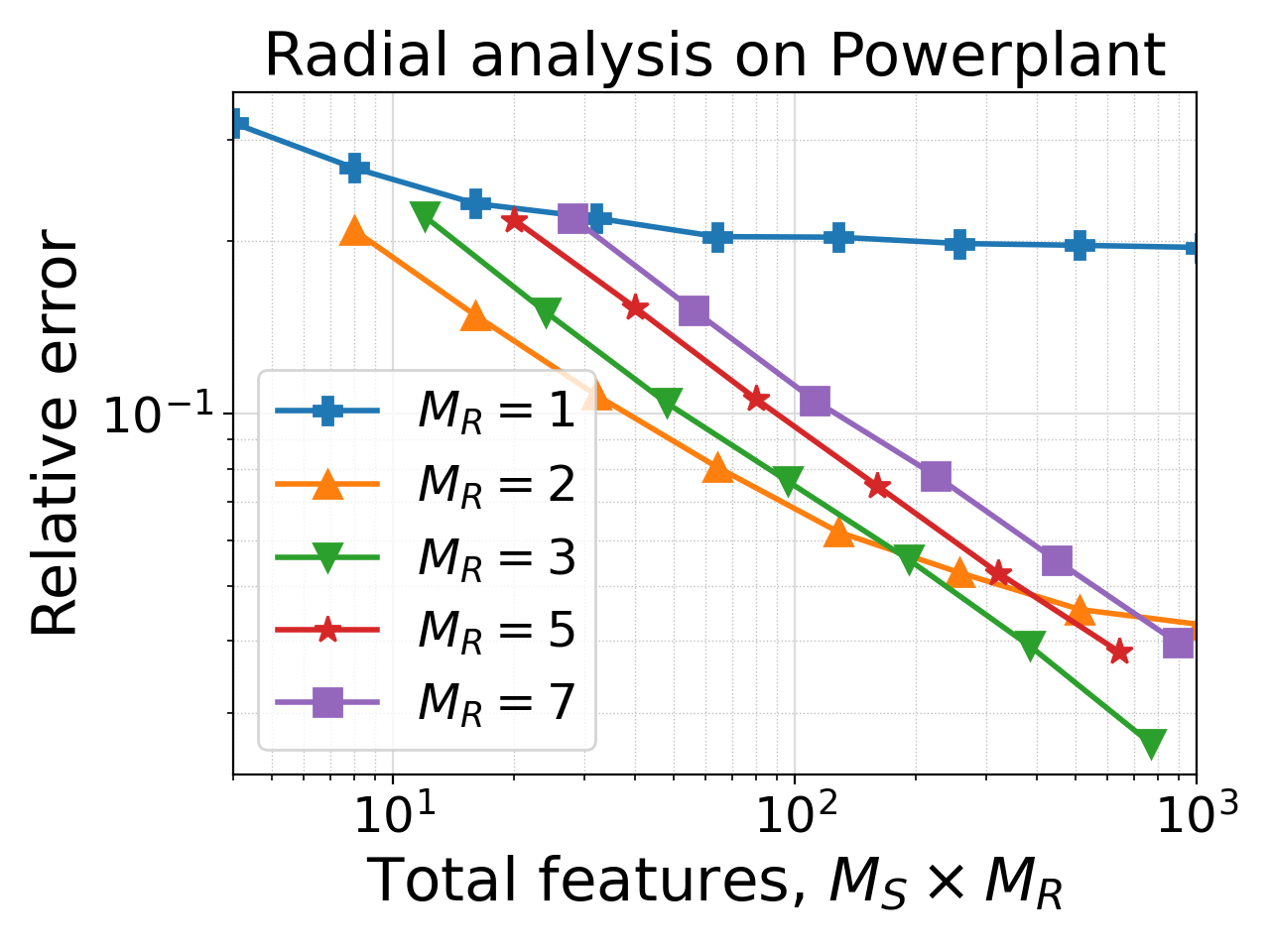}
    \vspace{-10pt}
    \caption{$\sigma = 0.57$}
    \label{fig:radial_design_smallbandwidth}
  \end{subfigure} 
  \begin{subfigure}[b]{0.33\linewidth}
  \centering
    \includegraphics[width=.99\linewidth]{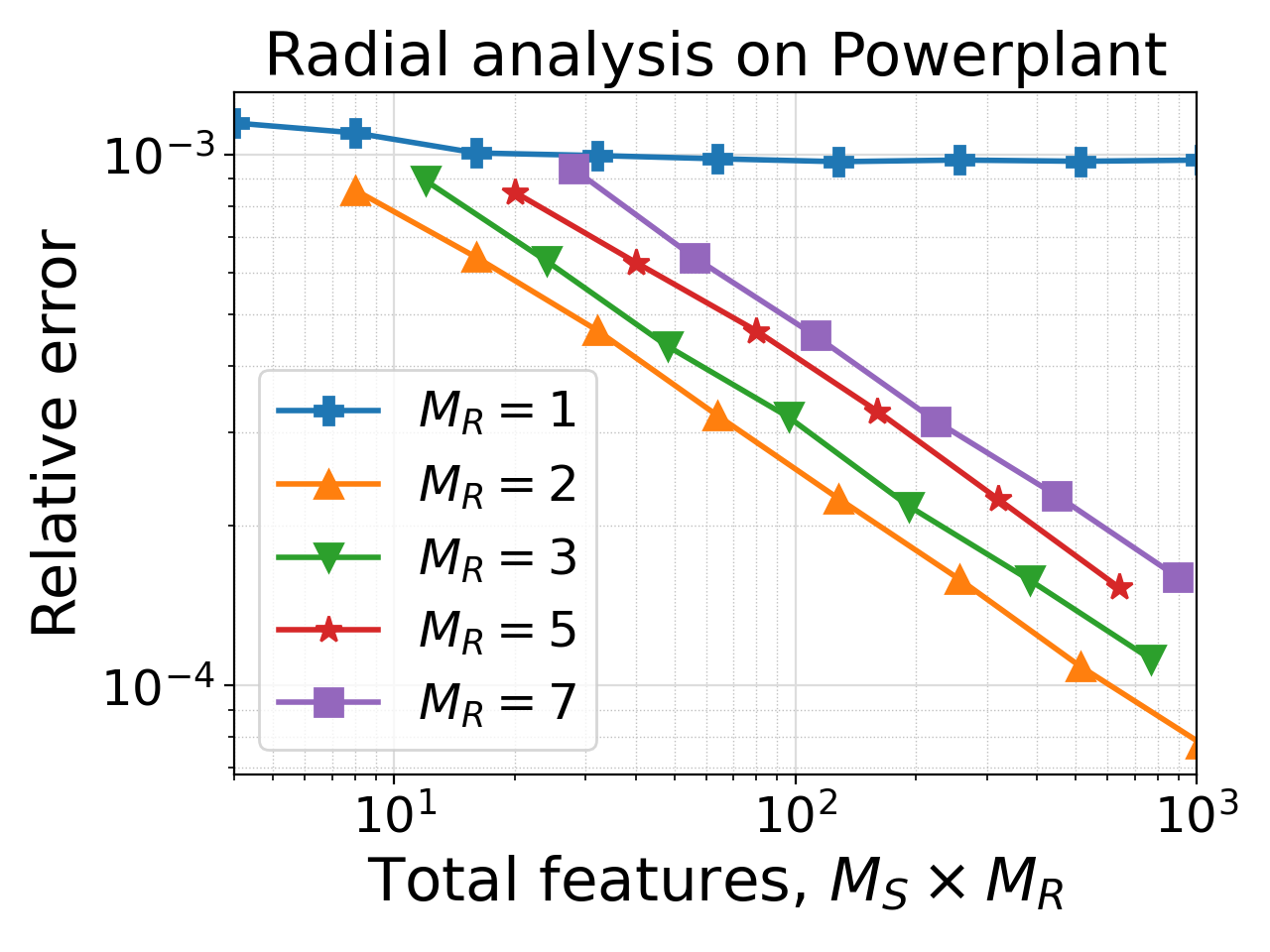}
    \vspace{-10pt}
    \caption{$\sigma = 2.83$}
    \label{fig:radial_design_largebandwidth}
  \end{subfigure} 

\caption{Relative error of kernel approximation schemes for increasing number of radial nodes $M_R$. \label{fig:sim_1.2}}
\end{figure}

In the second experiment, we investigate the impact of the number of nodes of the radial quadrature rule. \Cref{fig:sim_1.2} illustrates the relative error when SR-OMC is used to approximate the kernel matrix corresponding to $5000$ random samples from Powerplant. We conduct experiments with $M_R \in \{1, 2, 3, 5, 7\}$ using three different values of $\sigma$. For $\sigma = 0.28$, the relative error corresponding to larger $M_R$ decreases consistently but with a worse initial constant, eventually reaching a saturation point. This can be explained by the trade-off between spherical and radial error in \eqref{eq:main_theorem_OMC}. Initially, the spherical error dominates, but as the total number of features increases, the radial error, becomes more significant, and the error curve plateaus for a fixed value of $M_R$. The plateau is higher and appears earlier for smaller $M_R$. 
A similar behavior is observed when $\sigma = 2.83$, but this time the optimal $M_R$ is equal to $2$ when $M_S . M_R \leq 10^3$. For a fixed dataset and $\sigma$, the optimal value of $M_R$ is achieved when the spherical and radial errors are balanced. This experiment shows the importance of the choice of $M_R$ and its dependency on the kernel bandwidth $\sigma$. Further analysis on synthetic datasets and other real datasets is done in \Cref{sec:optimal_radial_nodes}.

\subsection{Kernel approximation on real-world datasets}\label{sec:num_sim_2}
In this section, we compare our construction SR-OMC to existing methods: (1) random Fourier features (RFF) \cite{Rahimi2007RFF}; (2) orthogonal random features (ORF) \cite{YuSuChHoKu16}; (3) QMC-based Fourier features (QMC) \cite{AvSiYaMa16}; and (4) stochastic spherical-radial quadratures (SSR) \cite{MuKaBuOs18}, on 4 datasets: Powerplant $(d=4, \sigma = 1.41)$, Letter $(d=16, \sigma = 1.0)$, USPS $(d=256, \sigma = 11.31)$, and MNIST $(d=784, \sigma = 10.58)$. We study two quantities: the relative Frobenius norm $d_{F}(K, \hat{K}) = \|K-\hat{K}\|_F/ \|K\|_F$, and a quantity evaluating spectral deviation $G(K,\hat{K}) = \|K^{-\frac{1}{2}}\hat{K}K^{-\frac{1}{2}}-\mathrm{Id}\|_2$.  

\Cref{fig:sim_2} illustrates the relative Frobenius error for various methods as a function of the number of features. Across all scenarios, SR-OMC and SSR consistently achieve superior accuracy compared to other state-of-the-art methods. In low-dimensional cases, ORF exhibits larger errors than SR-OMC and SSR, followed by QMC and RFF. As the dimensionality increases, ORF's performance converges to that of SR-OMC and SSR, while QMC and RFF continue to exhibit significantly larger errors.

Regarding the spectral deviation error, \Cref{fig:sim_2_2} demonstrates that our method consistently outperforms others across all datasets. As the dimensionality increases, OMC initially exhibits comparable approximation power on the Powerplant dataset but incurs a larger error on the LETTER dataset. ORF and SSR do not maintain stable spectral approximations on either dataset. In high-dimensional settings, all four methods perform similarly on the USPS and MNIST datasets, with our method showing a slight advantage. \Cref{fig:sim_2_2} implies that SR-OMC offers numerical guarantees for spectral approximation, and exploring the underlying theory presents a promising direction for future research.



Although our method, SR-OMC, has a relative Frobenius error similar to that of SSR, it offers greater flexibility and robustness in terms of radial quadrature rule. Compared to the radial design in \cite{MuKaBuOs18} which includes exactly 3 nodes (the origin and two symmetric random nodes), our method allows for easy adjustment of the radial design based on changes in the kernel bandwidth, dataset dimension, and total number of Fourier features. This adaptability enhances our algorithm's robustness across different problem settings. Particularly for high-dimensional datasets, minimizing the number of Fourier features is crucial to reduce computational complexity. Our method operates efficiently with minimal feature count of $d$, whereas SSR needs at least $2(d+1)$ features. 




\begin{figure}[ht]
\captionsetup{justification=centering}
\label{fig:kernel_approx}
  \begin{subfigure}[b]{0.25\linewidth}
    \includegraphics[width=.99\linewidth]{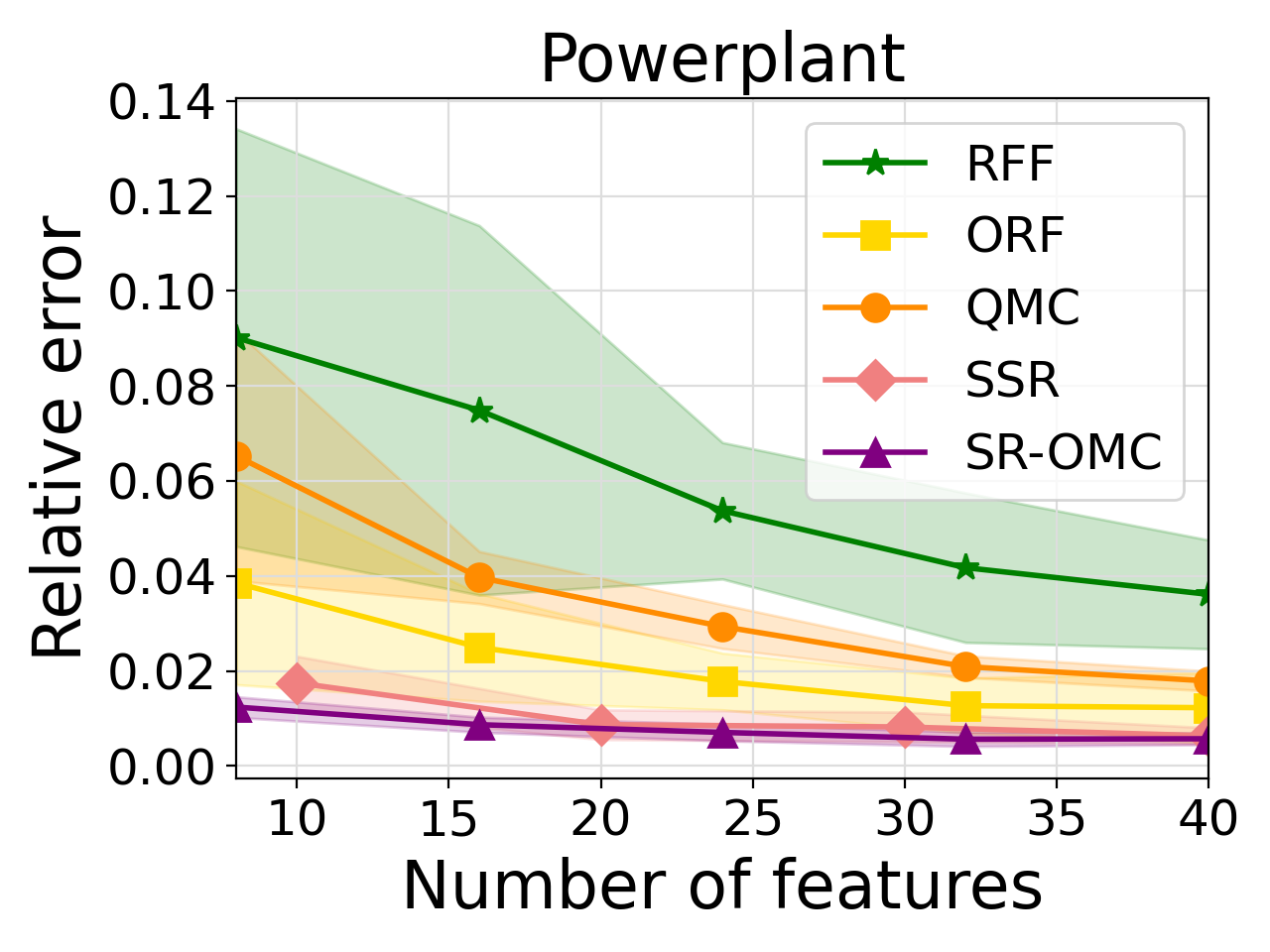}
    \vspace{-10pt}
    \caption{Powerplant, $d=4$\\$M_R = 2, \sigma = 1.41$} 
  \end{subfigure} 
  \hspace{-5pt}
  \begin{subfigure}[b]{0.25\linewidth}
    \includegraphics[width=.99\linewidth]{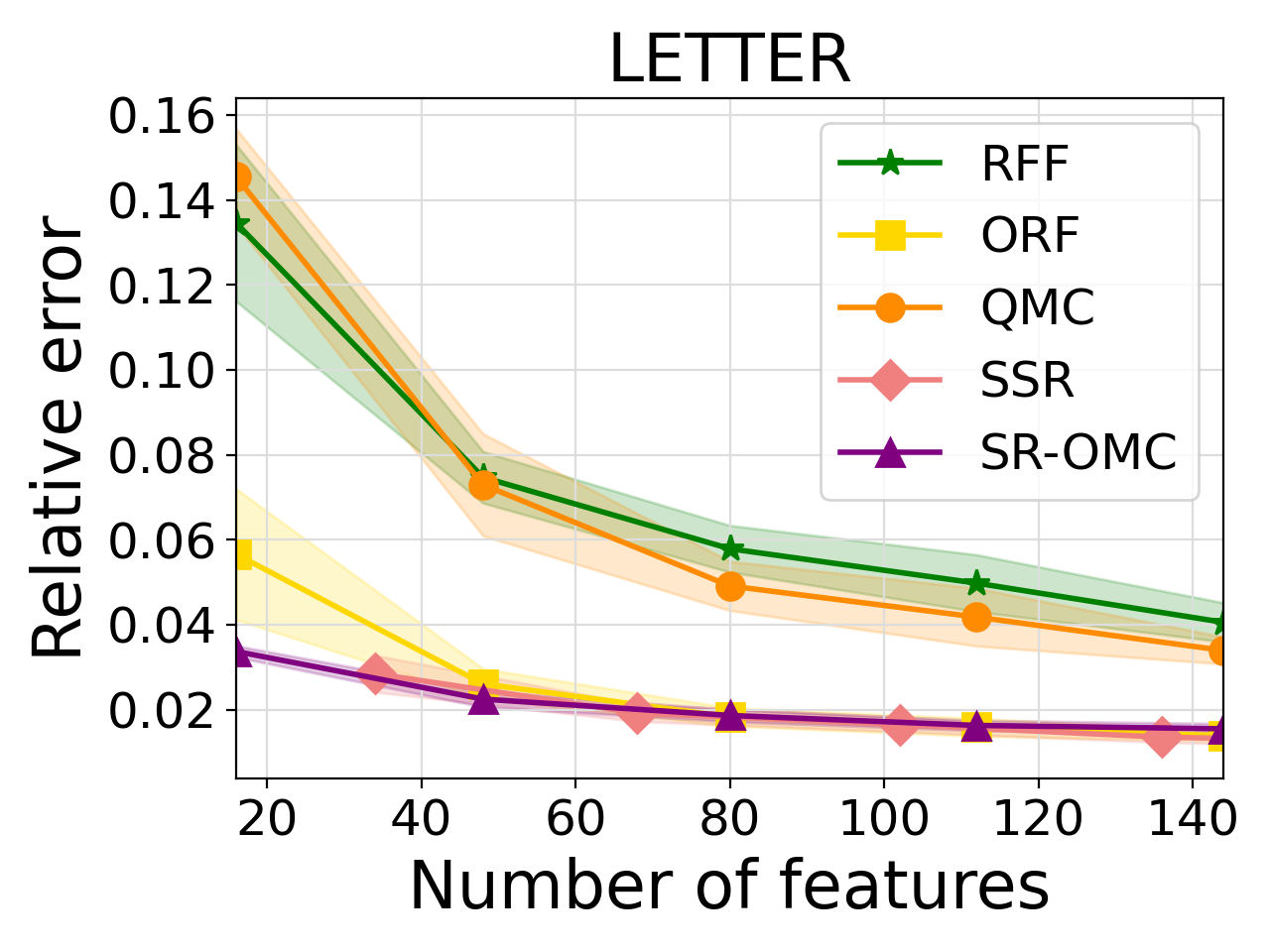}
    \caption{Letter, $d=16$\\$M_R = 1, \sigma = 1.0$}
  \end{subfigure} 
  \hspace{-5pt}
  \begin{subfigure}[b]{0.25\linewidth}
    \includegraphics[width=.99\linewidth]{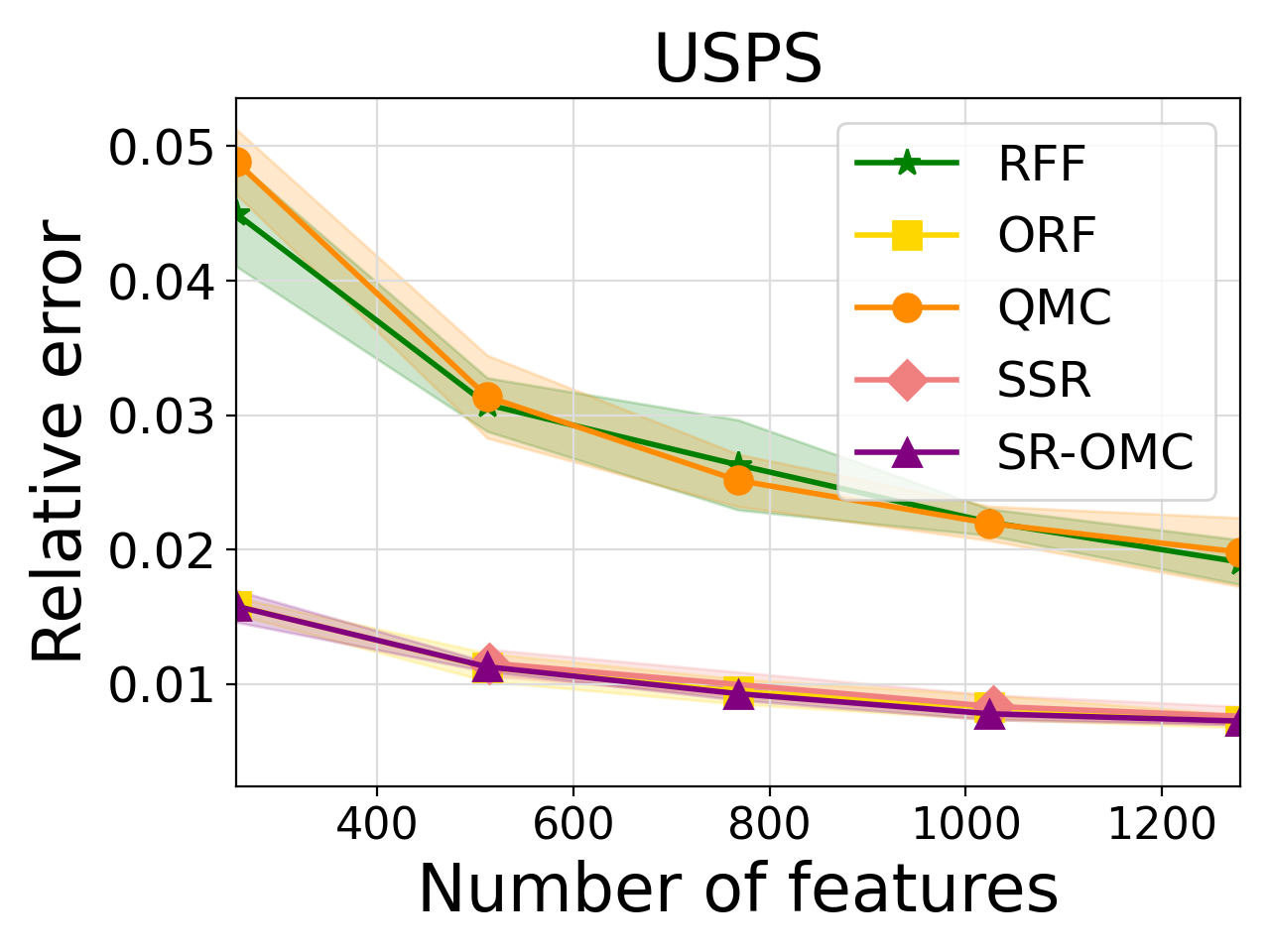}
    \caption{USPS, $d=256$\\ $M_R = 1, \sigma = 11.31$}
  \end{subfigure}   
  \hspace{-5pt}
  \begin{subfigure}[b]{0.25\linewidth}
    \includegraphics[width=.99\linewidth]{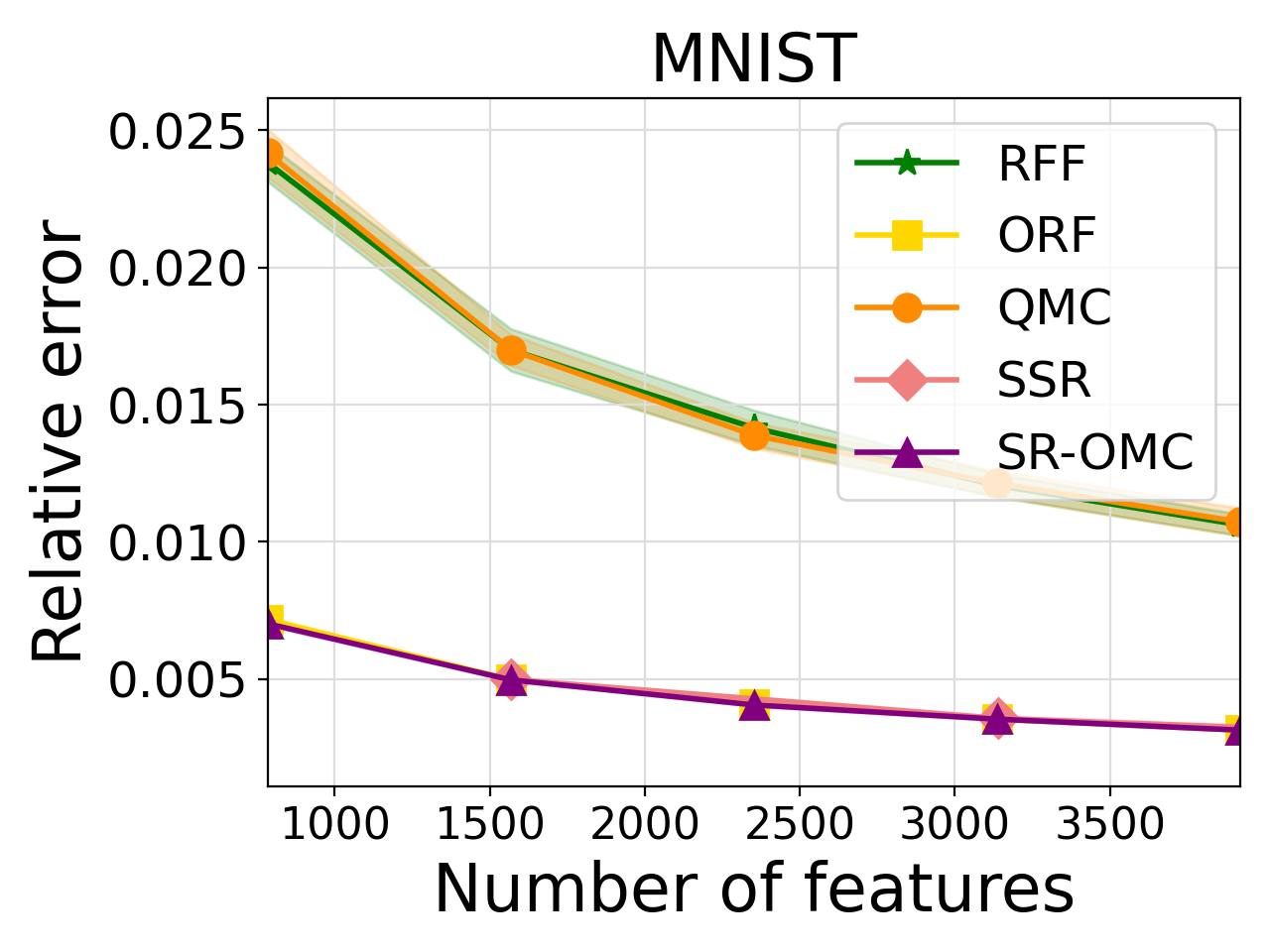}
    \caption{MNIST, $d=784$\\$M_R = 1, \sigma = 10.58$}
  \end{subfigure} 
\caption{Kernel approximation error on 4 datasets. SSR has slightly different bins on the x-axis due to its specific spherical-radial construction.  \label{fig:sim_2}}
\end{figure}

\begin{figure}[ht]
\captionsetup{justification=centering}
\label{fig:kernel_approx}
  \begin{subfigure}[b]{0.25\linewidth}
    \includegraphics[width=.99\linewidth]{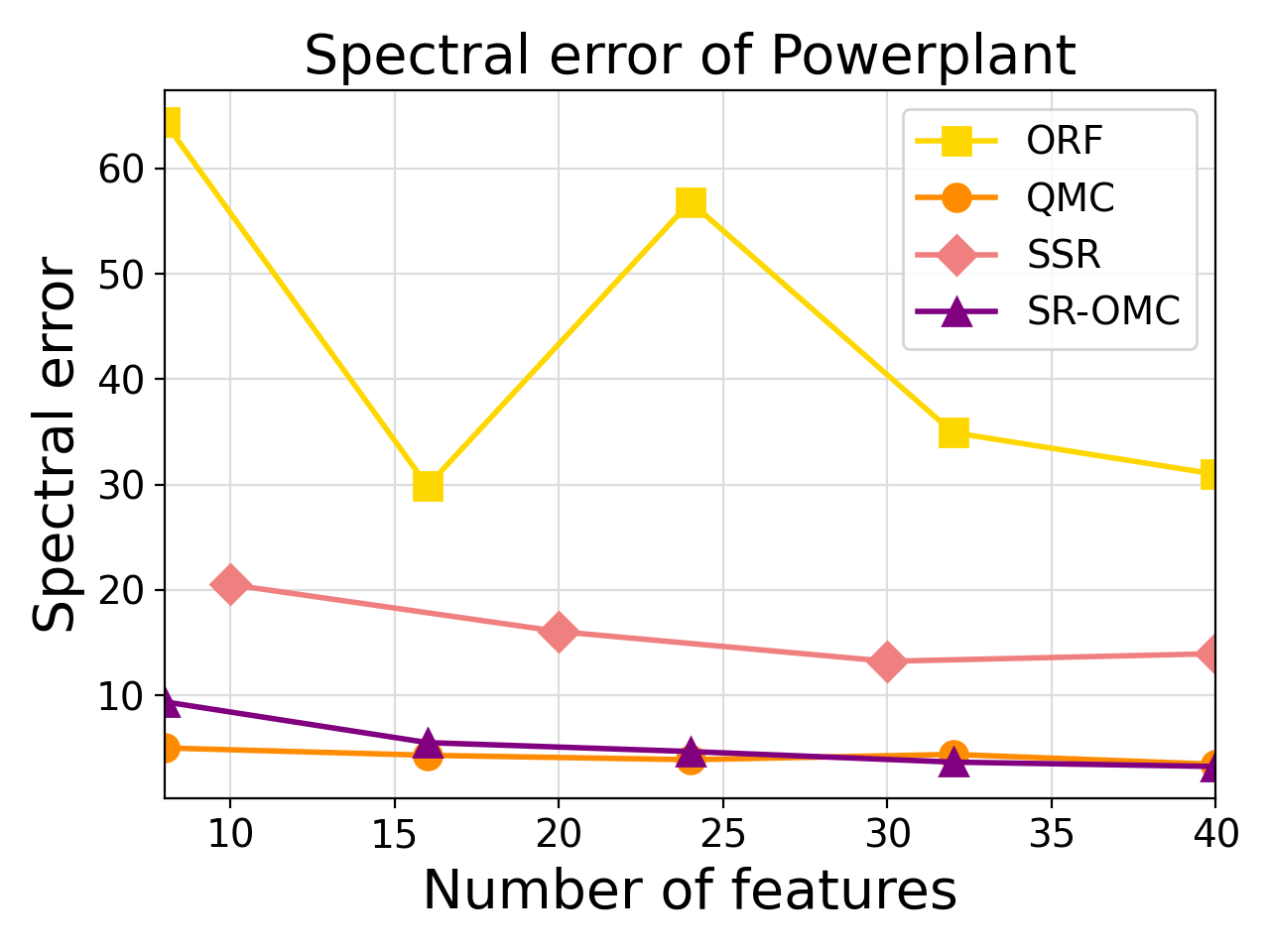}
    \vspace{-10pt}
    \caption{Powerplant, $d=4$\\$M_R = 2, \sigma = 1.41$} 
  \end{subfigure} 
  \hspace{-5pt}
  \begin{subfigure}[b]{0.25\linewidth}
    \includegraphics[width=.99\linewidth]{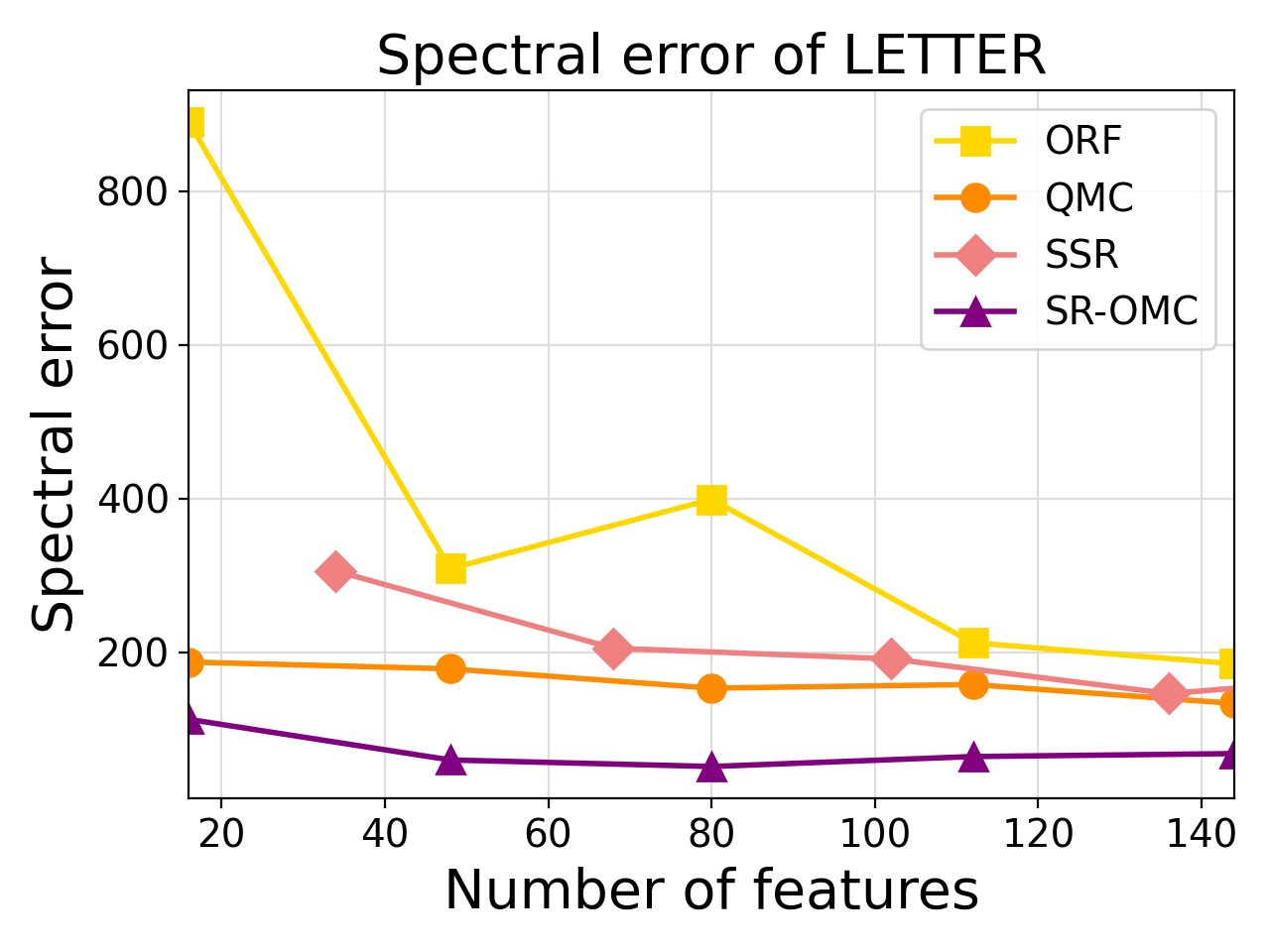}
    \caption{Letter, $d=16$\\$M_R = 1, \sigma = 1.0$}
  \end{subfigure} 
  \hspace{-5pt}
  \begin{subfigure}[b]{0.25\linewidth}
    \includegraphics[width=.99\linewidth]{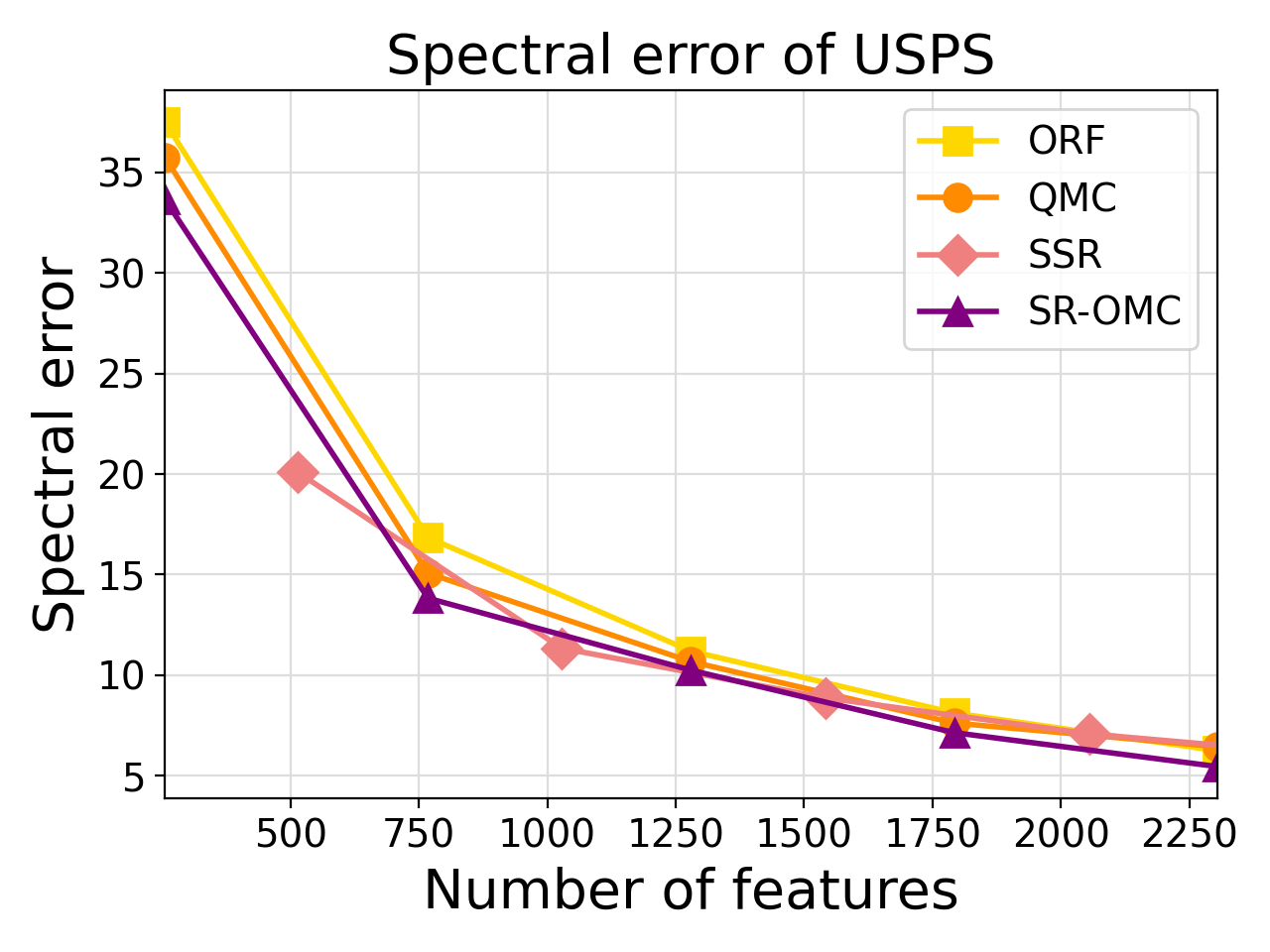}
    \caption{USPS, $d=256$\\ $M_R = 1, \sigma = 11.31$}
  \end{subfigure}   
  \hspace{-5pt}
  \begin{subfigure}[b]{0.25\linewidth}
    \includegraphics[width=.99\linewidth]{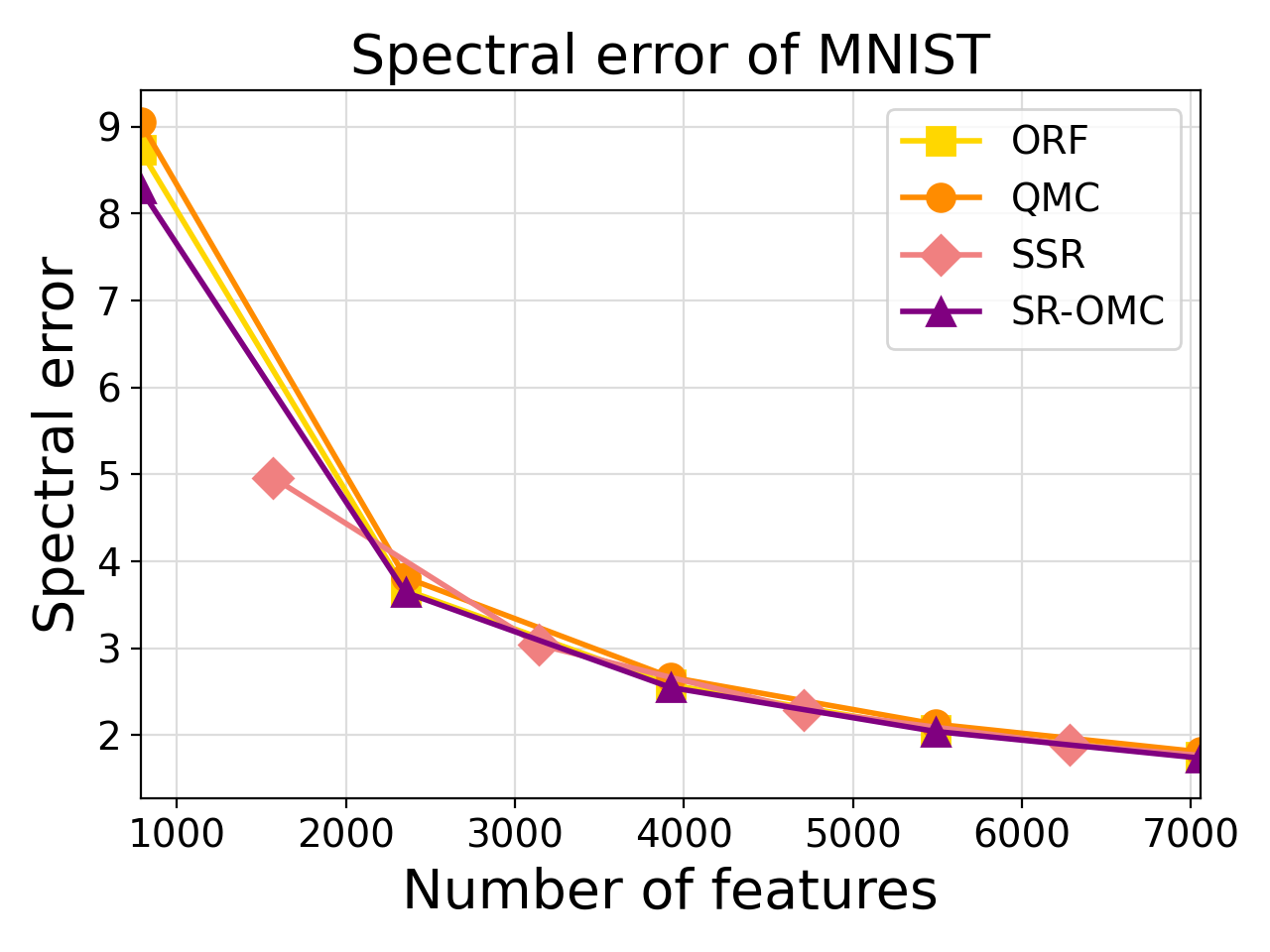}
    \caption{MNIST, $d=784$\\$M_R = 1, \sigma = 10.58$}
  \end{subfigure} 
\caption{Kernel approximation error on 4 datasets. SSR has slightly different bins on the x-axis due to its specific spherical-radial construction.  \label{fig:sim_2_2}}
\end{figure}

\subsection{Prediction on real-world datasets}\label{sec:num_sim_3}
To evaluate our method in practical learning tasks, we compare the performance of different kernel-approximation schemes on a regression task on Powerplant, and classification tasks on Letter and USPS respectively. For the regression task, we implement support vector regression (SVR) and report its $ R^2 $ score ($ R^2 = 1 $  indicates that the regression predictions perfectly fit the data); while for the classification task, we train a support vector classifier (SVC) and test its prediction accuracy.

\Cref{fig:sim_3} compares the result of different methods along with the standard deviation computed from 10 repeats. We observe that SR-OMC and SSR both achieve similar high-precision results. ORF demonstrates comparable performance to SR-OMC and SSR on high-dimensional dataset USPS but shows inferior performance on the others. On the contrary, QMC behaves the best on Powerplant but the performance drops significantly when dimension gets bigger. 

Combine the numerical analysis in \Cref{sec:num_sim_2} and \Cref{sec:num_sim_3}, we summarize that the prediction power of different methods does not fully align with kernel approximation accuracy in general. Our method, SR-ORC, consistently performs the best in all dimensions for both kernel approximation and prediction. SSR achieves comparable results to our method but admits notable spectral deviation in low-dimension kernel approximation tasks and requires 2 times more features for similar performance. ORF performs well in high-dimensional settings, while QMC behaves the opposite.

\begin{figure}[ht]
\captionsetup{justification=centering}
\label{fig:prediction}
  \begin{subfigure}[b]{0.33\linewidth}
  \centering
    \includegraphics[width=.99\linewidth]{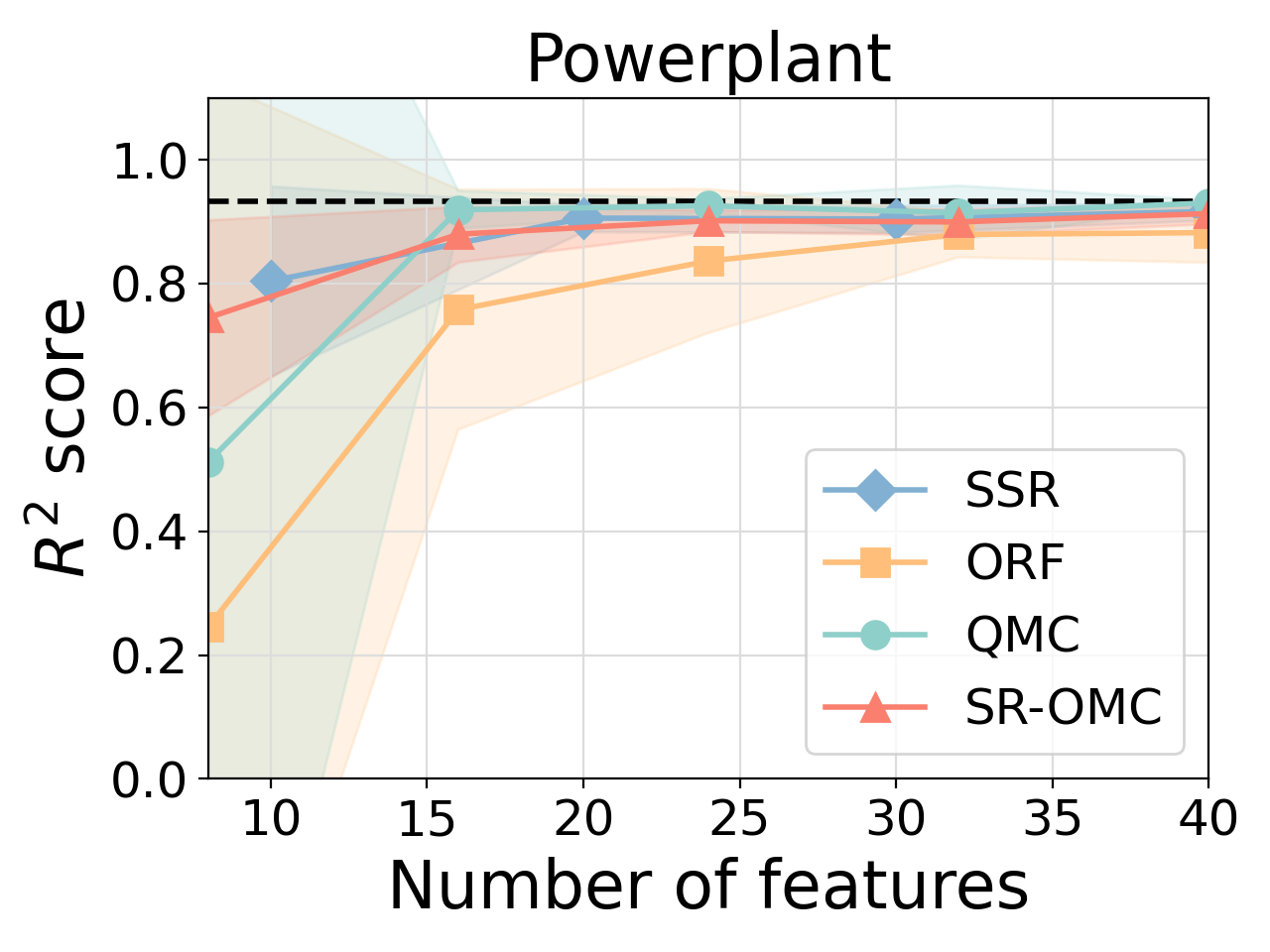}
    \vspace{-10pt}
    \caption{$R^2$ error for Powerplant\\$d=4, M_R=2, \sigma=1.41$} 
  \end{subfigure} 
  \begin{subfigure}[b]{0.33\linewidth}
  \centering
    \includegraphics[width=.99\linewidth]{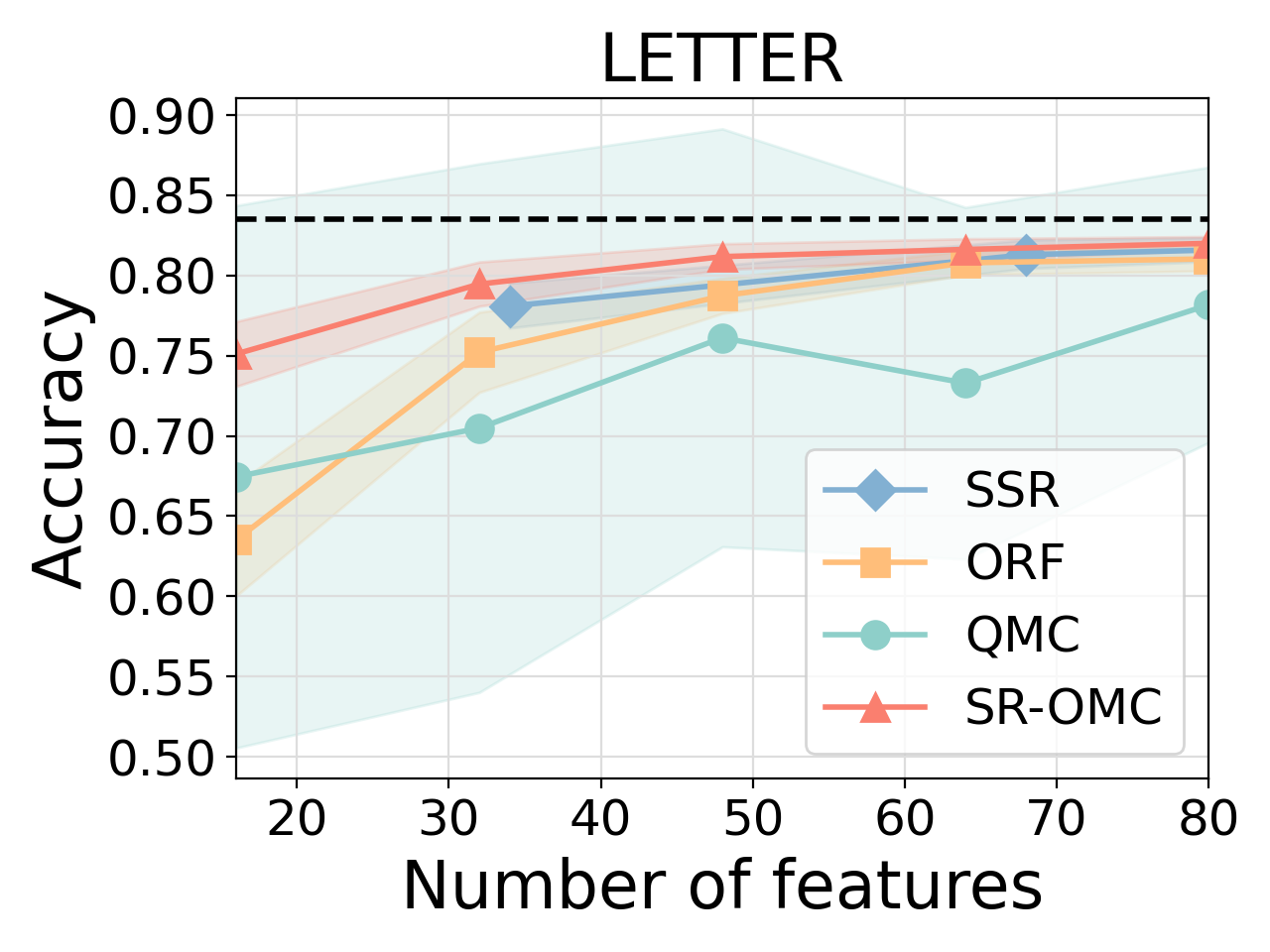}
    \vspace{-10pt}
    \caption{Accuracy for LETTER\\$d=16, M_R=1, \sigma=1.0$}
  \end{subfigure} 
  \begin{subfigure}[b]{0.33\linewidth}
  \centering
    \includegraphics[width=.99\linewidth]{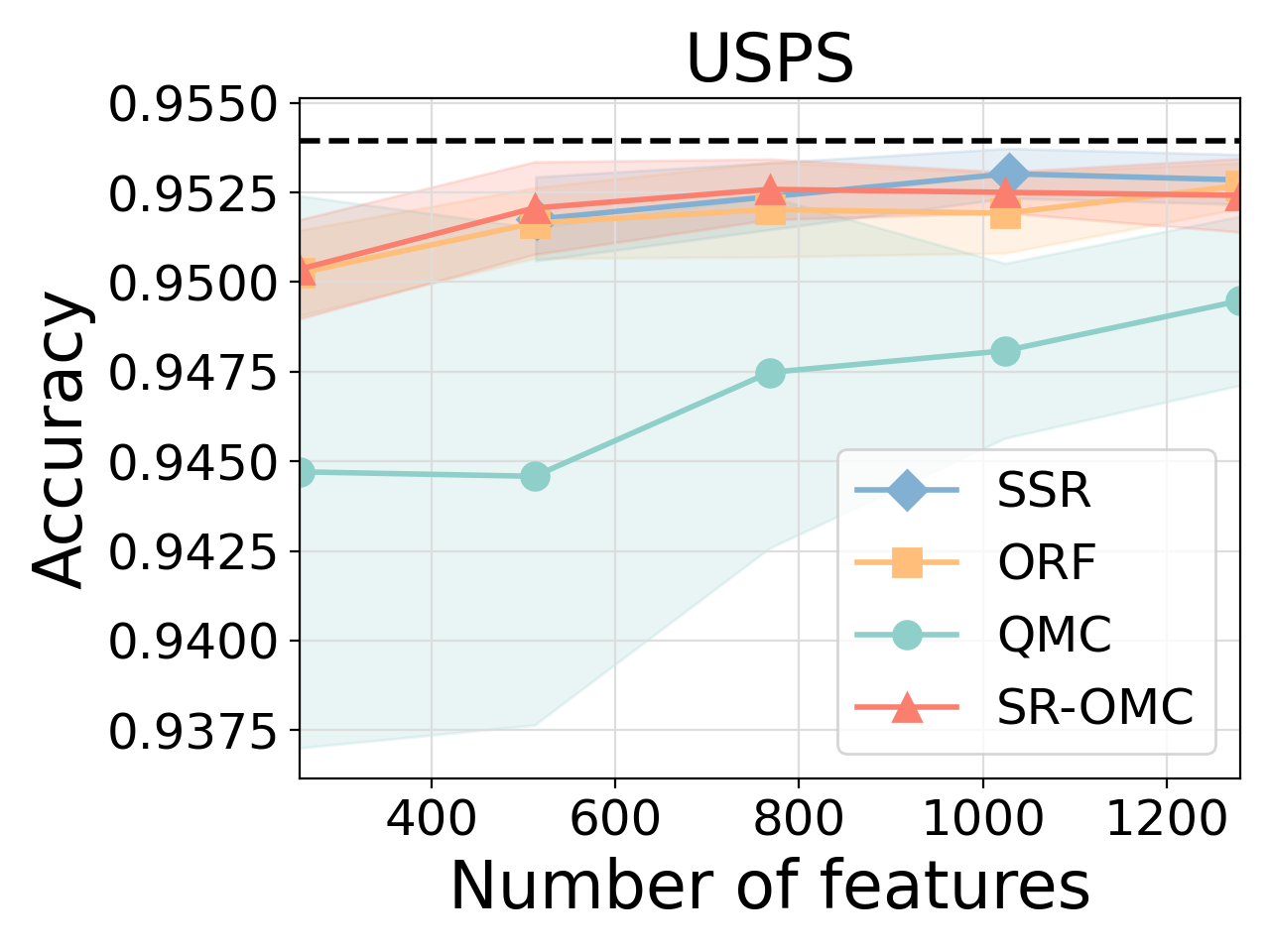}
    \vspace{-10pt}
    \caption{Accuracy for USPS\\$d=256, M_R=1, \sigma=11.31$} 
  \end{subfigure} 
  \vspace{-10pt}
\caption{Performance of SVM on regression and prediction tasks using different kernel approximation schemes.\label{fig:sim_3}}
\end{figure}

\section{Conclusion}

In this work, we provide an exhaustive study of spherical-radial Fourier features for approximating the squared exponential kernel when Gaussian quadrature is used in the radial part. Our analysis is based on a decomposition of the approximation error into spherical and radial components, allowing us to quantify the contribution of each term. In particular, our analysis highlights the interplay between 
the number of nodes in both quadratures and the bandwidth of the kernel. We also show that the approximation error is sensitive to the the choice of the spherical quadrature rule; specifically, we demonstrate that enforcing orthogonality in the nodes of the spherical quadrature rule yields better results compared to Monte Carlo on $\mathbb{S}^{d-1}$, both theoretically and empirically. Finally, our numerical simulations suggest that optimizing the weights of the spherical quadrature rule further improves the quality of approximation. Proving uniform approximation bounds on compact sets, and extending these results to other shift-invariant kernels, is left for future work.

\section{acknowledgement}
AB, QZ, and YMM are grateful for support from the US Department of Energy (DOE), Office of Science, Office of Advanced Scientific Computing Research (ASCR), under award number DE-SC0023188.

\newpage

\bibliographystyle{plain}
\bibliography{references}

\begin{thebibliography}{10}

\bibitem{AtHa12}
K.~Atkinson and W.~Han.
\newblock {\em Spherical harmonics and approximations on the unit sphere: an
  introduction}, volume 2044.
\newblock Springer Science \& Business Media, 2012.

\bibitem{AvSiYaMa16}
H.~Avron, V.~Sindhwani, J.~Yang, and M.~Mahoney.
\newblock Quasi-monte carlo feature maps for shift-invariant kernels.
\newblock {\em Journal of Machine Learning Research}, 17(120):1--38, 2016.

\bibitem{Bac17}
F.~Bach.
\newblock On the equivalence between kernel quadrature rules and random feature
  expansions.
\newblock {\em The Journal of Machine Learning Research}, 18(1):714--751, 2017.

\bibitem{Bel21}
A.~Belhadji.
\newblock An analysis of ermakov-zolotukhin quadrature using kernels.
\newblock {\em Advances in Neural Information Processing Systems},
  34:27278--27289, 2021.

\bibitem{BeBaCh19}
A.~Belhadji, R.~Bardenet, and P.~Chainais.
\newblock Kernel quadrature with {DPP}s.
\newblock In {\em Advances in Neural Information Processing Systems 32}, pages
  12907--12917. 2019.

\bibitem{BeBaCh20}
A.~Belhadji, R.~Bardenet, and P.~Chainais.
\newblock Kernel interpolation with continuous volume sampling.
\newblock {\em Proceedings of the 37th International Conference on Machine
  Learning}, pages 725--735, 2020.

\bibitem{CaCr}
M.R. Capobianco and G~Criscuolo.
\newblock Quadrature rules on unbounded interval.

\bibitem{ChRoWe17}
Krzysztof~M Choromanski, Mark Rowland, and Adrian Weller.
\newblock The unreasonable effectiveness of structured random orthogonal
  embeddings.
\newblock In I.~Guyon, U.~Von Luxburg, S.~Bengio, H.~Wallach, R.~Fergus,
  S.~Vishwanathan, and R.~Garnett, editors, {\em Advances in Neural Information
  Processing Systems}, volume~30. Curran Associates, Inc., 2017.

\bibitem{DaDeRe17}
T.~Dao, C.~De~Sa, and C.~R{\'e}.
\newblock Gaussian quadrature for kernel features.
\newblock {\em Advances in neural information processing systems}, 30, 2017.

\bibitem{EhGrOa19}
M.~Ehler, M.~Gr{\"a}f, and Ch.~J. Oates.
\newblock Optimal monte carlo integration on closed manifolds.
\newblock {\em Statistics and Computing}, 29(6):1203--1214, 2019.

\bibitem{GeMo99}
A.~Genz and J.~Monahan.
\newblock A stochastic algorithm for high-dimensional integrals over unbounded
  regions with gaussian weight.
\newblock {\em Journal of computational and applied mathematics},
  112(1-2):71--81, 1999.

\bibitem{GoLeTuZa17}
S.~Goldstein, J.~L. Lebowitz, R.~Tumulka, and N.~Zanghi.
\newblock Any orthonormal basis in high dimension is uniformly distributed over
  the sphere.
\newblock 2017.

\bibitem{HuSuHu2024}
Z.~Huang, J.~Sun, and Y.~Huang.
\newblock Quasi-monte carlo features for kernel approximation.
\newblock 2024.

\bibitem{LiHuChSuJo21}
F.~Liu, X.~Huang, Y.~Chen, and Johan~A.K. Suykens.
\newblock Random features for kernel approximation: A survey on algorithms,
  theory, and beyond.
\newblock {\em IEEE Transactions on Pattern Analysis and Machine Intelligence},
  44(10):7128--7148, 2021.

\bibitem{Mul06}
C.~M{\"u}ller.
\newblock {\em Spherical harmonics}, volume~17.
\newblock Springer, 2006.

\bibitem{MuKaBuOs18}
Marina Munkhoeva, Yermek Kapushev, Evgeny Burnaev, and Ivan Oseledets.
\newblock Quadrature-based features for kernel approximation.
\newblock In S.~Bengio, H.~Wallach, H.~Larochelle, K.~Grauman, N.~Cesa-Bianchi,
  and R.~Garnett, editors, {\em Advances in Neural Information Processing
  Systems}, volume~31. Curran Associates, Inc., 2018.

\bibitem{Rahimi2007RFF}
Ali Rahimi and Benjamin Recht.
\newblock Random features for large-scale kernel machines.
\newblock In J.~Platt, D.~Koller, Y.~Singer, and S.~Roweis, editors, {\em
  Advances in Neural Information Processing Systems}, volume~20. Curran
  Associates, Inc., 2007.

\bibitem{ShAv22}
P.~F. Shustin and H.~Avron.
\newblock Gauss-legendre features for gaussian process regression.
\newblock {\em Journal of Machine Learning Research}, 23(92):1--47, 2022.

\bibitem{SrSz15}
B.~Sriperumbudur and Z.~Szab{\'o}.
\newblock Optimal rates for random fourier features.
\newblock {\em Advances in neural information processing systems}, 28, 2015.

\bibitem{SuSc15}
Danica~J. Sutherland and Jeff~G. Schneider.
\newblock On the error of random fourier features.
\newblock {\em ArXiv}, abs/1506.02785, 2015.

\bibitem{Sze39}
G.~Szego.
\newblock {\em Orthogonal polynomials}, volume~23.
\newblock American Mathematical Soc., 1939.

\bibitem{YaSiAvMa14}
J.~Yang, V.~Sindhwani, H.~Avron, and M.~Mahoney.
\newblock Quasi-monte carlo feature maps for shift-invariant kernels.
\newblock In {\em International Conference on Machine Learning}, pages
  485--493. PMLR, 2014.

\bibitem{YuSuChHoKu16}
Felix Xinnan~X Yu, Ananda~Theertha Suresh, Krzysztof~M Choromanski, Daniel~N
  Holtmann-Rice, and Sanjiv Kumar.
\newblock Orthogonal random features.
\newblock In D.~Lee, M.~Sugiyama, U.~Luxburg, I.~Guyon, and R.~Garnett,
  editors, {\em Advances in Neural Information Processing Systems}, volume~29.
  Curran Associates, Inc., 2016.

\end{thebibliography}
\newpage
\appendix 

\section{Orthogonal families}
\subsection{Generalized Laguerre polynomials and the associated Gaussian quadrature}\label{sec:app_laguerre}

The family of generalized Laguerre polynomials orthogonal with respect to the weight function $p_{\alpha}(x) = \Gamma(\alpha+1)^{-1} x^\alpha e^{-x}\chi_{[0,+\infty[}(x)$ are given explicitly by the formula
\begin{equation}
    L^{\alpha}_n(x) = \sum_{i=0}^n\frac{1}{i!}\begin{pmatrix}
        n+\alpha\\n-i
    \end{pmatrix}(-x)^i.
\end{equation}
More precisely, we have
\begin{equation}
    \int_{\mathbb{R}_{+}}L^{\alpha}_n(x)L^{\alpha}_m(x)p_{\alpha}(x) \mathrm{d}x = \frac{\Gamma(n+\alpha+1)}{n!}\delta_{mn}.
\end{equation}
Moreover, by \Cref{lem:moment_laguerre} we have the following identity holds for general $\alpha'$
\begin{equation}
    \int_{\mathbb{R}_{+}}x^{\alpha'-1}L^\alpha_n(x)e^{-x}\chi_{[0,+\infty[}(x)\mathrm{d}x = \begin{pmatrix}
        \alpha-\alpha'+n\\n
    \end{pmatrix}\Gamma(\alpha').
\end{equation}
After proper scaling, we can derive an orthonormal basis of $L_{2}(p_{\alpha})$
\begin{equation}
    \ell^{\alpha}_n(x) := \sqrt{\frac{n!}{\Gamma(n+\alpha+1)}}L^{\alpha}_n(x),
\end{equation}
such that 
\begin{equation}
    \int_{\mathbb{R}_{+}}\ell^{\alpha}_n(x)\ell^{\alpha}_m(x)p_{\alpha}(x)\mathrm{d}x = \delta_{mn}.
\end{equation}
These polynomials satisfy the following recurrence
\begin{equation}
\ell_{n+1}^{\alpha}(X) = \frac{2n+1+\alpha -x}{\sqrt{(n+1)(n+\alpha+1)}} \ell_{n}^{\alpha}(X) - \sqrt{\frac{n(n+\alpha)}{(n+1)(n+\alpha+1)}} \ell_{n-1}^{\alpha}(X).
\end{equation}
In particular, the Gaussian quadrature associated to this family of orthogonal polynomials can be calculated by the eigendecomposition of the Jacobi matrix
\begin{equation*}
J = \begin{pmatrix}
1+\alpha & \sqrt{1+\alpha} & 0 & \cdots & 0 & 0 \\
\sqrt{1+\alpha} & 3+\alpha & \sqrt{2 \times (2+\alpha)} & \cdots & 0 & 0 \\
0 & \sqrt{2 \times (2+\alpha)}& 5+\alpha & \cdots & 0 & 0 \\
\vdots & \vdots & \vdots & \ddots & \vdots & \vdots \\
0 & 0 & 0 & \cdots & 2n - 1+ \alpha & \sqrt{(n+1)(n+1+\alpha)} \\
0 & 0 & 0 & \cdots & \sqrt{(n+1)(n+1+\alpha)} & 2n + 1+ \alpha
\end{pmatrix}
\end{equation*}
Moreover, they satisfy the identity
\begin{equation}
    \int_{\mathbb{R}_{+}}x^{\alpha'-1}\ell^{\alpha}_n(x)e^{-x}\chi_{[0,+\infty[}(x)\mathrm{d}x = \sqrt{\frac{n!}{\Gamma(n+\alpha+1)}}\begin{pmatrix}
        \alpha-\alpha'+n\\n
    \end{pmatrix}\Gamma(\alpha').
\end{equation}




\subsection{Spherical harmomics}\label{sec:spherical_harmonics}

We define $Y_{0,1}: \mathbb{S}^{d-1} \rightarrow \mathbb{R}$ to be the constant function equal to $1$.
For any $k\in \mathbb{N}^{*}$, denote by $\{Y_{k,i}: \mathbb{S}^{d-1} \rightarrow \mathbb{R}, i=1,\dots, N(d,k) \}$ the family of spherical harmonics of exact degree $k$, where
\begin{equation}\label{eq:dim_spherical_harmonics_k}
    N(d,k) := \frac{(2k +d -2) }{k} \binom{k+d-3}{d-2},
\end{equation}
and we adopt the notation $N(d,0) := 1$.

We refer to \cite{Mul06} for an explicit construction of this family of functions. 

The family $(Y_{k,i})$ forms an orthonormal basis in $L_{2}(\pi_{\mathbb{S}^{d-1}})$. In particular, we have for $k,k' \in \mathbb{N}$, and for $(i,i') \in [N(d,k)] \times [N(d,k')]$, we have
\begin{equation}\label{eq:orthogonality_spherical}
  \langle Y_{k,i}, Y_{k',i'} \rangle_{\pi_{\mathbb{S}^{d-1}}} = \delta_{k,k'} \delta_{i,i'}.
\end{equation}
Moreover, they satisfy the addition formula 
\begin{equation}\label{eq:spherical_addition_formula}
\forall k \in \mathbb{N}, \:\: \forall x,y \in \mathbb{S}^{d-1}, \:\:    \sum\limits_{i =1}^{N(d,k)} Y_{k,i}(x)Y_{k,i}(y) = N(d,k) P_{k}(\langle x,y \rangle),
\end{equation}
where $P_{k}$ is $k$-th Gegenbauer polynomial of parameter $\alpha = d/2-1$ defined by
\begin{equation}\label{def:P_k}
    P_k(t) = \Big(-\frac{1}{2}\Big)^k\frac{\Gamma\big(\frac{d-1}{2}\big)}{\Gamma\big(k+\frac{d-1}{2}\big)}(1-t^2)^{\frac{3-d}{2}}\left(\frac{\dd}{\dd t}\right)^k(1-t^2)^{k+\frac{d-3}{2}},
    \end{equation}
see\footnote{In \cite{AtHa12}, the authors opted for the \emph{un-normalized} uniform measure on $\mathbb{S}^{d-1}$, hence the difference between 2.24 in \cite{AtHa12} and \eqref{eq:spherical_addition_formula}} Theorem 2.9 in \cite{AtHa12}. For instance, we have
\begin{align}\label{eq:polys_examples}
    P_{0}(t) & = 1, \nonumber\\
    P_{1}(t) & = t, \nonumber\\
    P_{2}(t) & = \frac{1}{d-1}\big(d t^2 - 1\big) .
\end{align}
Moreover, by Section 2.7.3. in \cite{AtHa12}, the polynomials satisfy the recursion relation
\begin{equation}\label{eq:polys_rec_relation}
    P_{k+1}(t) = \frac{2k+d-2}{k+d-2} t P_{k}(t) - \frac{k}{k+d-2} P_{k-1}(t).
\end{equation}
In particular, using \eqref{eq:polys_rec_relation} along with \eqref{eq:polys_examples} we prove that for $k \in \mathbb{N}$, $P_{2k}$ is an even function while $P_{2k+1}$ is an odd function. In addition, by inequality (2.39) in \cite{AtHa12}, we have
\begin{equation}\label{eq:bound_polys_by_1}
    \forall k \in \mathbb{N}, \forall t \in [-1,1], \:\: |P_{k}(t)| \leq 1,
\end{equation}
and using the addition formula \eqref{eq:spherical_addition_formula}, and the orthogonality of the spherical harmonics \eqref{eq:orthogonality_spherical}, we prove that
\begin{equation}\label{eq:int_Sd_P_k_2}
 \forall v \in \mathbb{S}^{d-1}, \:\:   \int_{\mathbb{S}^{d-1}} P_{k}\big(\langle v, \theta \rangle \big)^2 \mathrm{d} \pi_{\mathbb{S}^{d-1}}(\theta) =  \frac{1}{N(d,k)}.
\end{equation}
Finally, for a function $\varphi: \mathbb{R} \rightarrow \mathbb{R}$, we have the Hecke-Funk formula 
\begin{equation}\label{eq:hecke_funk}
\forall v \in \mathbb{S}^{d-1}, \:\:    \int_{\mathbb{S}^{d-1}} \varphi(\langle v,n \rangle) Y_{k,i}(n) \dd \pi_{\mathbb{S}^{d-1}}(n) =  \frac{\mathcal{V}_{d-2}}{\mathcal{V}_{d-1}}  Y_{k,i}(v) \int_{-1}^{1} \varphi(t) P_{k}(t) (1-t^2)^{(d-3)/2} \mathrm{d}t,
\end{equation}
see Theorem 2.22 in \cite{AtHa12}.






\section{Proofs}

\subsection{Proof of \Cref{prop:f_bar_representation}}


When $x-y =0$, we have for $\xi \in \mathbb{R}_{+}$ and $n \in \mathbb{S}^{d-1}$
\begin{equation}
f_{x-y}\Big(\frac{\sqrt{2\xi}}{\sigma}\;n \Big) = \cos \Big( \big\langle \frac{\sqrt{2\xi}}{\sigma}\;n, x-y \big\rangle \Big) = 1,
\end{equation}
so that $ \bar{f}_{x-y}(\xi) = \beta_{0}(d,x-y)$, by observing that 
\begin{equation}
     \beta_0(d,x-y) = \frac{\mathcal{V}_{d-2}}{\mathcal{V}_{d-1}} \frac{\sqrt{\pi} \Gamma\left(\frac{d-1}{2}\right)}{\Gamma\left(\frac{d}{2}\right)} = \frac{2 \pi^{(d-1)/2}}{\Gamma((d-1)/2)} \frac{\Gamma(d/2)}{2 \pi^{d/2}}  \frac{\sqrt{\pi} \Gamma\left(\frac{d-1}{2}\right)}{\Gamma\left(\frac{d}{2}\right)} =1.
\end{equation}
Now, assume that $x-y \neq 0$, and define
\begin{equation}
v := \frac{1}{\|x-y\|}(x-y) \in \mathbb{S}^{d-1}.
\end{equation}
We have, for $\xi \in \mathbb{R}_{+}$ and $n \in \mathbb{S}^{d-1}$, 
\begin{equation}
f_{x-y}\Big(\frac{\sqrt{2\xi}}{\sigma}n \Big) = \cos \Big( \big\langle \frac{\sqrt{2\xi}}{\sigma}n, x-y \big\rangle \Big) = \cos \Big( \frac{\sqrt{2\xi}}{\sigma} \|x-y\| \langle n,v \rangle \Big) = \varphi_{\frac{\sqrt{2\xi}}{\sigma} \|x-y\|}(\langle n,v \rangle ),
\end{equation}
where, for a given $\gamma >0$, the function $\varphi_{\gamma} : \mathbb{R} \rightarrow \mathbb{R}$ is defined as follows
\begin{equation}
\varphi_{\gamma}(t) := \cos(\gamma t).
\end{equation}
Thus, by Hecke-Funk formula \eqref{eq:hecke_funk}, we have
\begin{align}\label{eq:int_cos_poly}
\bar{f}_{x-y}(\xi) &= \int_{\mathbb{S}^{d-1}} \varphi_{\frac{\sqrt{2\xi}}{\sigma}\|x-y\|}(\langle n, v \rangle) \mathrm{d} \sigma(n) \nonumber \\
& = \frac{\mathcal{V}_{d-2}}{\mathcal{V}_{d-1}} \int_{-1}^{1} \varphi_{\frac{\sqrt{2\xi}}{\sigma}\|x-y\|}(t) (1-t^2)^{(d-3)/2} \mathrm{d}t \nonumber \\
& = \frac{\mathcal{V}_{d-2}}{\mathcal{V}_{d-1}} \int_{-1}^{1} \cos\Big(\frac{\sqrt{2\xi}}{\sigma}\|x-y\| t\Big) (1-t^2)^{(d-3)/2} \mathrm{d}t.
\end{align}
Now, observe that
\begin{align}\label{eq:cos_taylor_alpha}
    \cos\Big(\frac{\sqrt{2\xi}}{\sigma}\|x-y\| t \Big)&=\sum_{n=0}^{\infty}\frac{(-1)^n (\sqrt{2\xi}\|x-y\| t)^{2n}}{\sigma^{2n} (2n)!}=\sum_{n=0}^{+\infty}\alpha_n(\xi)\; t^{2n}, \\
    &\alpha_n(\xi) := \frac{(-1)^n (\sqrt{2\xi}\|x-y\|)^{2n}}{(2n)! \sigma^{2n}}. \nonumber
\end{align}
To switch the summation and integral in \Cref{eq:int_cos_poly}, we show that
\begin{align}
    \int_{-1}^1\sum_{n=0}^{+\infty}|\alpha_n(\xi)\;| |t|^{2n}(1-t^2)^{(d-3)/2} \mathrm{d}t & \leq \int_{-1}^1\sum_{n=0}^{+\infty}\frac{ (\sqrt{2\xi}\|x-y\|)^{2n}}{(2n)! \sigma^{2n}} (1-t^2)^{(d-3)/2} \mathrm{d}t\\
    &\leq \int_{-1}^1\text{cosh}\Big(\frac{\sqrt{2\xi}}{\sigma}\|x-y\| t \Big) (1-t^2)^{(d-3)/2} \mathrm{d}t\\
    &\leq \text{cosh}(1)\int_{-1}^1(1-t^2)^{(d-3)/2} \mathrm{d}t\\
    &=\text{cosh}(1)\frac{\Gamma(1/2)\Gamma(d/2-1/2)}{\Gamma(d/2)}<\infty,
\end{align}
where the last equality is proved using the identity
\begin{equation}\label{eq:beta_function_property}
\forall \alpha>-\frac{1}{2}, \forall \beta>-1, \:\:    \int_{-1}^{1}t^{2\alpha}(1-t^2)^{\beta} \dd t = B(\alpha+1/2,\beta+1) = \frac{\Gamma(\alpha +1/2) \Gamma(\beta +1)}{\Gamma(\alpha +\beta +3/2)}, 
\end{equation}
where $B$ is the Beta function.

By Fubini-Tonelli Theorem, the integration in \eqref{eq:int_cos_poly} can be decomposed into sequence of integration, with each term computed as follows

\begin{align}\label{eq:int_alpha_n}
    \int_{-1}^{1} \alpha_n(\xi) t^{2n} (1-t^2)^{(d-3)/2} \mathrm{d}t
    &=\frac{(-1)^n (\sqrt{2\xi}\|x-y\|)^{2n}}{(2n)!\sigma^{2n}}\int_{-1}^{1} t^{2n}(1-t^2)^{(d-3)/2} \mathrm{d}t \nonumber\\
    &=\frac{(-1)^n (\sqrt{2\xi}\|x-y\|)^{2n}}{(2n)! \sigma^{2n}}\frac{\Gamma(n +1/2) \Gamma((d-1)/2)}{\Gamma(n +d/2)},
\end{align}
since
\begin{equation}
    \int_{-1}^{1}t^{2n}(1-t^2)^{(d-3)/2} \dd t  = \frac{\Gamma(n +1/2) \Gamma((d-1)/2)}{\Gamma(n +d/2)},
\end{equation}
which again follows from the identity \eqref{eq:beta_function_property}.

Now, observe that
\begin{equation}
    \frac{\Gamma(n +1/2) }{(2n)!} = \frac{\overbrace{(2n-1) \cdot (2n-3) \dots \cdot 1}^{\text{n factors}} \cdot \Gamma(1/2)}{2^n (2n)!} = \frac{\Gamma(1/2) }{2^n \cdot 2^n \cdot n!} = \frac{ \sqrt{\pi}}{2^{2n} n!}.
\end{equation}
Thus, \eqref{eq:int_alpha_n} yields
\begin{align}\label{eq:int_alpha_n_2}
   \int_{-1}^{1} \alpha_n(\xi) t^{2n} (1-t^2)^{(d-3)/2} \mathrm{d}t &= \frac{(-1)^n \|x-y\|^{2n}}{ \sigma^{2n}} \frac{2^{n}  \Gamma(n +1/2)}{(2n)!} \frac{ \Gamma((d-1)/2)}{\Gamma(n +d/2)} \xi^n \nonumber\\ &=\frac{\sqrt{\pi}(-\|x-y\|^2/(2\sigma^2))^n\Gamma\left(\frac{d-1}{2}\right)}{\Gamma(n+1)\Gamma\left(\frac{d}{2}+n\right)}\xi^n.
\end{align}
Combining \eqref{eq:int_cos_poly}, \eqref{eq:cos_taylor_alpha}, and \eqref{eq:int_alpha_n_2}, we get
\begin{equation}
    \bar{f}_{x-y}(\xi)  =  \sum_{n=0}^{+\infty}\beta_n(d,x-y) \xi^n ,
\end{equation}
where
\begin{equation}\label{eq:f_bar_expansion}
    \beta_n(d,x-y) := \frac{\mathcal{V}_{d-2}}{\mathcal{V}_{d-1}} \frac{\sqrt{\pi}(-\|x-y\|^2/(2\sigma^{2}))^n\Gamma\left(\frac{d-1}{2}\right)}{\Gamma(n+1)\Gamma\left(d/2+n\right)}.
\end{equation}
Now, by using the fact that $\frac{\mathcal{V}_{d-2}}{\mathcal{V}_{d-1}} = \frac{\Gamma\left(d/2\right)}{\sqrt{\pi}\Gamma \left(d/2-1/2\right)}$, we get
\begin{equation}
    \beta_n(d,x-y) =  \frac{(-\|x-y\|^2/(2\sigma^{2}))^n\Gamma\left(\frac{d}{2}\right)}{\Gamma(n+1)\Gamma\left(d/2+n\right)}.
\end{equation}
Moreover, the infinite series in \eqref{eq:f_bar_expansion} is absolutely convergence because the ratio
\begin{align}
    \Big|\frac{\beta_{n+1}(d,x-y)\xi^{n+1}}{\beta_n(d,x-y)\xi^n}\Big| &= \frac{\sqrt{\pi}\|x-y\|^2\xi/(2\sigma^{2})}{(n+1)(d/2+n)}\leq\frac{\sqrt{\pi}\|x-y\|^2\xi/(2\sigma^{2})}{n^2}
\end{align}
converges to 0 at the rate $n^{-2}$. Therefore, the series converge.





\subsection{Proof of \Cref{prop:radial_quadrature_error}}\label{proof:radial_quadrature_error}
In the following we look for an expression of the coefficient $\gamma_m (d, x-y)$ defined as 
\begin{equation}\label{eq:gamma_m}
   \gamma_m (d, x-y):=  \langle \bar{f}_{x-y}, \ell_{m}^{\alpha} \rangle_{p_{\Xi}}= \int_{\mathbb{R}_{+}} \bar{f}_{x-y}(\xi) \ell_{m}^{\alpha}(\xi) p_{\Xi}(\xi) \mathrm{d} \xi, \qquad \alpha = d/2-1.
\end{equation}
By definition \Cref{prop:f_bar_representation}, we have
\begin{equation}
 \forall \xi \in \mathbb{R}_{+}, \:\:   \bar{f}_{x-y}(\xi) = \sum\limits_{n=0}^{+\infty} \beta_{n}(d,x-y) \xi^n.
\end{equation}
In order to switch the order of integral and summation in \Cref{eq:gamma_m}, we claim that
\begin{align}
    \sum\limits_{n=0}^{+\infty}\int_{\mathbb{R}_{+}}\Big|\beta_{n}(d,x-y) \xi^n\ell_{m}^{\alpha}(\xi) p_{\Xi}(\xi)\Big| \mathrm{d} \xi =: \sum\limits_{n=0}^{+\infty}\eta^n &< \infty.
\end{align}
To prove the claim, we use Cauchy Schwarz inequality
\begin{align}
    \int_{\mathbb{R}_{+}}\Big|\xi^n\ell_{m}^{\alpha}(\xi) p_{\Xi}(\xi)\Big| \mathrm{d} \xi&\leq\left(\int_{\mathbb{R}_{+}}\xi^{2n}p_{\Xi}(\xi)\mathrm{d} \xi\right)^{1/2}\left(\int_{\mathbb{R}_{+}}[\ell_{m}^{\alpha}(\xi)]^2 p_{\Xi}(\xi)\mathrm{d} \xi\right)^{1/2}\\
    &=\left[(\alpha+2n)!\right]^{1/2}
\end{align}
Plug in $\beta_n$, we obtain a sequence $\zeta^n$ that bounds $\eta^n$ from above
\begin{align}
    \eta_n &\leq |\beta_{n}(d,x-y)|\int_{\mathbb{R}_{+}}\Big|\xi^n\ell_{m}^{\alpha}(\xi) p_{\Xi}(\xi)\Big| \mathrm{d}\xi\\
    &=\frac{\mathcal{V}_{d-2}}{\mathcal{V}_{d-1}} \frac{\sqrt{\pi}(\|x-y\|^2/(2\sigma^{2}))^n\Gamma\left(\frac{d-1}{2}\right)}{\Gamma(n+1)\Gamma\left(d/2+n\right)}\left[(\alpha+2n)!\right]^{1/2}=:\zeta_n.
\end{align}
We can always find $N$ large enough such that the ratio between consecutive $\zeta_n$ satisfies
\begin{align}
    \Big|\zeta_{n+1}/\zeta_n\Big| & =\frac{(\|x-y\|^2/(2\sigma^{2}))\Gamma(n+1)\Gamma\left(d/2+n\right)}{\Gamma(n+2)\Gamma\left(d/2+n+1\right)}\frac{\left[(\alpha+2n+2)!\right]^{1/2}}{\left[(\alpha+2n)!\right]^{1/2}}\\
    &=\frac{\|x-y\|^2/(2\sigma^{2})}{(n+1)(d/2+n)}(\alpha+2n+2)^{1/2}(\alpha+2n+1)^{1/2}\\
    &\leq\frac{\|x-y\|^2/(2\sigma^{2})}{n^2}(\alpha+2n+2)\leq 1/2
\end{align}
for any $n\geq N$. The ratio test shows that the sequence $\{\zeta_n\}_n$ is summable. $\{\eta_n\}_n$ is positive and dominated by $\{\zeta_n\}_n$ hence is also summable, which proves the claim.

With the claim, we employ Fubini-Tonelli Theorem to move the summation out
\begin{equation}
  \gamma_m (d, x-y) = \sum\limits_{n=0}^{+\infty} \beta_{n}(d,x-y) \int_{\mathbb{R}_{+}} \xi^n \ell^{\alpha}_{m}(\xi) p_{\Xi}(\xi) \mathrm{d}\xi = \sum\limits_{n=m}^{+\infty} \beta_{n}(d,x-y) \int_{\mathbb{R}_{+}} \xi^n \ell^{\alpha}_{m}(\xi) p_{\Xi}(\xi) \mathrm{d}\xi, 
\end{equation}
where the last equality follows from the fact that $\ell_{m}^{\alpha}$ is orthogonal to all polynomials with degree less than $m$ with respect to the inner product $\langle .,. \rangle_{p_{\Xi}}$.

Now, by \Cref{lem:moment_laguerre} the coefficient $\int_{\mathbb{R}_{+}} \xi^n \ell^{\alpha}_{m}(\xi) p_{\Xi}(\xi) \mathrm{d}\xi$ writes 
\begin{align}\label{eq:r_n_l_m_w_formula}
    \int_{\mathbb{R}_{+}} \xi^n \ell^{\alpha}_{m}(\xi) p_{\Xi}(\xi) \mathrm{d}\xi &=\Gamma(\alpha+1)^{-1}\int_{\mathbb{R}_{+}} \xi^n \ell_m^{\alpha}(\xi) \xi^{\alpha}e^{-\xi}\mathrm{d} \xi \nonumber\\
    &= \binom{m-n -1  }{m}\sqrt{\frac{\Gamma(m+1)}{\Gamma(m+\alpha +1)}} \frac{\Gamma(\alpha+1+n) }{\Gamma(\alpha+1)}
\end{align}
where
\begin{equation}\label{eq:combin_number_negative}
    \binom{m-n -1  }{m} := \frac{(m-n-1)(m-n-2)\dots(-n)}{m!}  
\end{equation}
is equal to $(-1)^m\binom{n}{m}$ when $n \geq m$.


By definition of $\beta_{n}(d,x-y)$ in \eqref{eq:def_beta}, we have 
\begin{equation}\label{eq:beta_n_x_y}
\beta_n(d, x-y) = \frac{(-1)^n\Gamma\left(d/2\right)}{\Gamma\left(d/2+n\right)\Gamma(n+1)}c^{2n}, \:\: c:= \frac{\|x-y\|}{\sqrt{2}\sigma}.
\end{equation}

Combining \eqref{eq:r_n_l_m_w_formula}, \eqref{eq:combin_number_negative}, and \eqref{eq:beta_n_x_y}, we get
\begin{align}
    &\beta_n(d, x-y) \int_{\mathbb{R}_{+}} \xi^n \ell^{\alpha}_{m}(\xi) p_{\Xi}(\xi) \mathrm{d}\xi \\
    &= \frac{(-1)^nc^{2n}\Gamma\left(d/2\right)}{\Gamma\left(d/2+n\right)\Gamma(n+1)} (-1)^{m} \binom{n}{m}\sqrt{\frac{\Gamma(m+1)}{\Gamma(m+d/2) } } \frac{\Gamma(d/2+n)}{\Gamma(d/2)} \nonumber \\
    &= \frac{(-1)^{n+m}c^{2n}}{(n-m)!} \sqrt{\frac{1}{m!\,\Gamma(m+d/2)}}.
\end{align}
Thus, by summing over $n$, we get
\begin{align}
    \big|\gamma_m\left(d,x-y\right)\big| &=\bigg|\sum_{n=m}^{\infty}\beta_n(d, x-y) \int_{\mathbb{R}_{+}} \xi^n \ell^{\alpha}_{m}(\xi) p_{\Xi}(\xi) \mathrm{d}\xi \bigg|\\
    &= \bigg|\sum_{n=m}^{\infty}\frac{(-1)^{n+m}c^{2n}}{(n-m)!} \sqrt{\frac{1}{m!\,\Gamma(m+d/2)}}\bigg|\\
    &=\sqrt{\frac{1}{m!\,\Gamma(m+d/2)}}\bigg|\sum_{n=m}^{\infty}\frac{(-1)^n c^{2n}}{(n-m)!}\bigg|\\
    &=\sqrt{\frac{1}{m!\,\Gamma(m+d/2)}} c^{2m}\bigg|\sum_{n=0}^{\infty}\frac{(-c^2)^n}{n!}\bigg|\\
    &=\sqrt{\frac{1}{m!\,\Gamma(m+d/2)}}c^{2m}e^{-c^2}.
\end{align}
Thus, for $M \in \mathbb{N}$, we have
\begin{align}
    \sum\limits_{m =2M}^{+\infty} \sqrt{m}|\gamma_m\left(d,x-y\right)| &\leq  \sum\limits_{m =2M}^{+\infty} m|\gamma_m\left(d,x-y\right)|\\ & = \sum\limits_{m=2M}^{+\infty}m \sqrt{\frac{1}{m!\,\Gamma(m+d/2)}}c^{2m}e^{-c^2} \nonumber \\
    & \leq \frac{c^2}{\sqrt{\Gamma(d/2)}}\sum\limits_{m=2M}^{+\infty} \frac{1}{(m-1)!} c^{2(m-1)} e^{-c^2}, 
\end{align}
where we have use the fact that 
\begin{equation}
    \Gamma(m+d/2) = \underbrace{(d/2 +m -1) \cdots (d/2+1) \cdot (d/2)}_{\text{m factor}} \Gamma(d/2) \geq m!  \Gamma(d/2).
\end{equation}
Now, by \Cref{lemma:exp_truncated}, we have 

\begin{equation}
    \sum\limits_{m=2M-1}^{+\infty} \frac{1}{m!} c^{2m} e^{-c^2} \leq \frac{1}{(2M-1)!} \big(c^2\big)^{2M-1}e^{c^{2}} e^{-c^{2}} \leq \Big( \frac{c^2}{2M-1}\Big)^{2M-1}.
\end{equation}

Now to prove~\eqref{eq:laguerre_quad_err_bound}, we use \Cref{lem:quad_err_for_Laguerre} which states that there exists a constant $L$ that is independent of $M_R$ such that 
\begin{equation}
\forall \varphi \in L_{2}(p_\Xi) ,\:\: \Big|\int_{\mathbb{R}_{+}} \varphi(\xi) p_{\Xi}(\xi) \mathrm{d}\xi - \sum\limits_{i=1}^{M_R} a_{i} \varphi(\xi_{i}) \Big| \leq  L \sum\limits_{m=2M_{R}}^{+\infty} \sqrt{m} |\langle \varphi, \ell^{\alpha}_{m} \rangle_{p_{\Xi}}|.
\end{equation}

\subsection{Proof of \Cref{prop:harmonic_decomposition_f}}\label{sec:proof_prop:harmonic_decomposition_f}

The function $f_{r(x-y)}$ is continuous on $\mathbb{S}^{d-1}$. Thus, by Theorem 2.30 in  \cite{AtHa12}, $f_{r(x-y)}$ decomposes in the spherical harmonics basis as follows 
\begin{align}\label{eq:harmonic_decompose_f_2}
    f_{r(x-y)}(n) &= \sum_{k=0}^{+\infty}\sum_{i=1}^{N(d,k)} \langle f_{r(x-y)}, Y_{k,i} \rangle_{\pi_{\mathbb{S}^{d-1}}}Y_{k,i}(n),
\end{align}
where the convergence holds uniformly on $\mathbb{S}^{d-1}$. In the following, we calculate the coefficients $\langle f_{r(x-y)}, Y_{k,i} \rangle_{\pi_{\mathbb{S}^{d-1}}}$.

When $r(x-y) = 0$, we have 
\begin{equation}
\forall n \in \mathbb{S}^{d-1}, \:\: f_{r(x-y)}(n) =\cos(\langle r(x-y),n \rangle) = 1 = Y_{0,1}(n),
\end{equation}
so that 
\begin{equation}
    \forall k \in \mathbb{N}, \:\: \forall i \in [N(d,k)], \:\: \langle f_{r(x-y)}, Y_{k,i} \rangle_{ \pi_{\mathbb{S}^{d-1}}} = \delta_{k,0}.
\end{equation}

Now, assume that $r(x-y) \neq 0$. Define $v:= (x-y)/\|x-y\|$, and $\beta:= r(x-y)$. Let $k \in \mathbb{N}$ and $i \in [N(d,k)]$. We have
\begin{equation*}
\langle f_{r(x-y)}, Y_{k,i} \rangle_{\pi_{\mathbb{S}^{d-1}}} = \int_{\mathbb{S}^{d-1}}  f_{r(x-y)}(n) Y_{k,i}(n) \dd \pi_{\mathbb{S}^{d-1}}(n) =  \int_{\mathbb{S}^{d-1}}  \varphi_{\beta}(\langle v, n \rangle) Y_{k,i}(n) \dd \pi_{\mathbb{S}^{d-1}}(n),
\end{equation*}
where the function $\varphi_{\beta}: \mathbb{R} \rightarrow \mathbb{R}$ is defined by $\varphi_{\beta}(t) = \cos(\beta t)$. Using the Hecke-Funk formula \eqref{eq:hecke_funk}, we get 
\begin{align}
\langle f_{r(x-y)}, Y_{k,i} \rangle_{\pi_{\mathbb{S}^{d-1}}} & =     \int_{\mathbb{S}^{d-1}} \varphi_{\beta}(\langle v,n \rangle) Y_{k,i}(n) \dd \pi_{\mathbb{S}^{d-1}}(n)  \\
    &= \frac{\mathcal{V}_{d-2}}{\mathcal{V}_{d-1}}  Y_{k,i}(v) \int_{-1}^{1} \varphi_{\beta}(t) P_{k}(t) (1-t^2)^{(d-3)/2} \mathrm{d}t.
\end{align}
Therefore, using the addition formula ~\eqref{eq:spherical_addition_formula}, we have for $n \in \mathbb{S}^{d-1}$
\begin{align}
    \sum_{i=1}^{N(d,k)} \langle f_{r(x-y)}, Y_{k,i} \rangle_{\pi_{\mathbb{S}^{d-1}}}Y_{k,i}(n)  & = \frac{\mathcal{V}_{d-2}}{\mathcal{V}_{d-1}}\int_{-1}^{1} \varphi_{\beta}(t) P_{k}(t) (1-t^2)^{(d-3)/2} \mathrm{d}t \sum_{i=1}^{N(d,k)} Y_{k,i}(v) Y_{k,i}(n) \\ 
    & =   N(d,k) P_{k}(\langle  v,n \rangle)\lambda_k  ,
\end{align}
where
\begin{equation}
\lambda_{k} := \frac{\mathcal{V}_{d-2}}{\mathcal{V}_{d-1}} \int_{-1}^{1} \cos(\beta t)  P_{k}(t) (1-t^2)^{(d-3)/2} \dd t = \frac{\mathcal{V}_{d-2}}{\mathcal{V}_{d-1}} \int_{-1}^{1} \cos(r \|x-y\| t)  P_{k}(t) (1-t^2)^{(d-3)/2} \dd t .
\end{equation}
Using the identity $\frac{\mathcal{V}_{d-2}}{\mathcal{V}_{d-1}} = \frac{\Gamma\left(d/2\right)}{\sqrt{\pi}\Gamma \left(d/2-1/2\right)}$ concludes the proof of \eqref{eq:harmonic_decompose_f}.


Now, we move to the proof of the upper bound~\eqref{eq:lambda_k_bound}. First, observe that when $k$ is an odd integer $P_{k}$ is an odd function, and we get $\lambda_{k} = 0$, and the bound ~\eqref{eq:lambda_k_bound} holds. In the following, assume that $k$ is an even integer. By Lemma 11 in \cite{Mul06}, we have for any $k$ times differentiable function $\varphi: \mathbb{R} \rightarrow \mathbb{R}$

\begin{equation}\label{eq:lemma_11_mul}
    \int_{-1}^{1} \varphi(t) P_{k}(t) (1-t^2)^{(d-3)/2} \mathrm{d}t = \Big(\frac{1}{2}\Big)^{k} \frac{\Gamma((d-1)/2)}{\Gamma(k+(d-1)/2)} \int_{-1}^{1} \varphi^{(k)}(t) (1-t^2)^{k+\frac{d-3}{2}}  \mathrm{d}t.
\end{equation}
By taking $\varphi(t) = \varphi_{\beta}(t) = \cos(\beta t)$, we have $\varphi^{(k)}(t) = (-1)^{k/2}\beta^{k}\cos(\beta t)$, and we get
\begin{equation}\label{eq:lambda_k_alpha_formula}
    \lambda_{k} = (-1)^{k/2} 
    \frac{\Gamma\left(d/2\right)}{\sqrt{\pi}\Gamma \left(d/2-1/2\right)}\Big(\frac{\beta}{2}\Big)^{k} \frac{\Gamma((d-1)/2)}{\Gamma(k+(d-1)/2)} \int_{-1}^{1}\cos(\beta t) (1-t^2)^{k+\frac{d-3}{2}}  \mathrm{d}t.
\end{equation}
Now, observe that 
\begin{equation}\label{eq:upper_bound_int_cos_poly_0}
\Big|\int_{-1}^{1}  \cos(\beta t)(1-t^2)^{k+\frac{d-3}{2}} \mathrm{d}t  \Big| \leq \int_{-1}^{1}  \big|\cos(\beta t) \big|(1-t^2)^{k+\frac{d-3}{2}} \mathrm{d}t  \leq \int_{-1}^{1} (1-t^2)^{k+\frac{d-3}{2}} \mathrm{d}t .
\end{equation}
Moreover, we have
\begin{equation}\label{eq:upper_bound_int_cos_poly}
    \int_{-1}^{1} (1-t^2)^{k+\frac{d-3}{2}} \mathrm{d}t = B(1/2,k+(d-1)/2) =  \frac{\Gamma(1/2)\Gamma(k+\frac{d-1}{2})}{\Gamma(k+\frac{d}{2})},
\end{equation}
where $B$ is the Beta function. Since $k$ is an even integer, we have
\begin{equation}
    \frac{\Gamma(k+\frac{d-1}{2})}{\Gamma(k+\frac{d}{2})} = \frac{\big(k+(d-1)/2 -1 \big) \dots \big((d-1)/2+1 \big)\big((d-1)/2\big)}{\big(k+d/2 -1 \big) \dots \big(d/2-1 \big)\big(d/2\big)} \frac{\Gamma((d-1)/2)}{\Gamma(d/2)} \leq \frac{\Gamma((d-1)/2)}{\Gamma(d/2)}. 
\end{equation}
Thus
\begin{equation}\label{eq:bound_on_int_k_d_3}
   \int_{-1}^{1} (1-t^2)^{\frac{k+d-3}{2}} \mathrm{d}t \leq \Gamma(1/2) \frac{\Gamma((d-1)/2)}{\Gamma(d/2)} = \frac{\sqrt{\pi} \Gamma((d-1)/2)}{\Gamma(d/2)}.
\end{equation}
Combining \eqref{eq:lambda_k_alpha_formula}, \eqref{eq:upper_bound_int_cos_poly_0}, \eqref{eq:upper_bound_int_cos_poly} and \eqref{eq:bound_on_int_k_d_3}, we get
\begin{equation}
\big|\lambda_{k} \big| \leq  \Big(\frac{\beta}{2}\Big)^{k} \frac{\Gamma((d-1)/2)}{\Gamma(k+(d-1)/2)},
\end{equation}
which concludes the proof of \eqref{eq:lambda_k_bound}.

\subsection{Proof of \Cref{prop:iid_spherical_quadrature}}\label{proof:prop_iid_spherical_quadrature}
Define the approximation error 
\begin{equation}
    \mathcal{E}_{S}(f_{r(x-y)}) := \Big| \int_{\mathbb{S}^{d-1}} f_{r(x-y)}(n) \mathrm{d} \pi_{\mathbb{S}^{d-1}}(n) - \frac{1}{M_S}\sum\limits_{j = 1}^{M_S}  f_{r(x-y)}(\theta_{j}) \Big|.
\end{equation}

Let $r \in \mathbb{R}_{+}^{*}$ and $x,y \in \mathbb{R}^{d}$. 
When $r(x-y) = 0$, $f_{r(x-y)} = Y_{0,1} \equiv 1$ and
\begin{equation}
     \E\mathcal{E}_{S}(f_{r(x-y)})^2 \Big|_{r(x-y)=0} = 0 = \frac{1}{M_S}   
     \sum\limits_{k = 2}^{+\infty} N(d,k) \lambda_{k}^2,
\end{equation}
since $\lambda_{k} = 0$ for $k \in \mathbb{N}^{*}$.

Now, assume that $r (x-y) \neq 0$. By \Cref{prop:harmonic_decomposition_f}, we have 
\begin{equation}
\forall n \in \mathbb{S}^{d-1}, \:\:        f_{r(x-y)}(n) = \sum_{k=0}^{+\infty} N(d,k) \lambda_{k} P_{k}(\langle v, n \rangle),\:\: v:= \frac{x-y}{\|x-y\|}.
\end{equation}
For $k=0$, we have by Hecke-Funk formula \eqref{eq:hecke_funk}
\begin{equation}
    \lambda_{0} = \frac{\mathcal{V}_{d-2}}{\mathcal{V}_{d-1}} \int_{-1}^{1} \cos(r\|x-y\|t) P_{0}(t) (1-t^2)^{(d-3)/2} \mathrm{d} t = \int_{\mathbb{S}^{d-1}}f_{r(x-y)}(n) \mathrm{d}\pi_{\mathbb{S}^{d-1}}(n),
\end{equation}
since $Y_{0,1}(v) = 1$ for $v \in \mathbb{S}^{d-1}$.
Moreover, since $P_{1}$ is an odd function, we have
\begin{equation}
    \lambda_{1} =    \frac{\mathcal{V}_{d-2}}{\mathcal{V}_{d-1}} \int_{-1}^{1} \cos(r\|x-y\|t) P_{1}(t) (1-t^2)^{(d-3)/2} \mathrm{d} t = 0.
\end{equation}
Thus 
\begin{equation}
    \mathcal{E}_{S}(f_{r(x-y)}) = \frac{1}{M_S}  \Big|\sum\limits_{k=2}^{+\infty} N(d,k) \lambda_{k} \sum\limits_{j=1}^{M_{S}} P_{k}\big(\langle v, \theta_{j} \rangle \big)\Big|.
\end{equation}
Therefore 
\begin{equation}\label{eq:error_spherical_r_xy_1}
        \mathcal{E}_{S}(f_{r(x-y)})^2 = \frac{1}{M_{S}^2}  \sum\limits_{k_1=2}^{+\infty}\sum\limits_{k_2=2}^{+\infty} N(d,k_1)N(d,k_2) \lambda_{k_1}\lambda_{k_2} \sum\limits_{j_1=1}^{M_{S}} \sum\limits_{j_2=1}^{M_{S}}P_{k_1}\big(\langle v, \theta_{j_1} \rangle \big) P_{k_2}\big(\langle v, \theta_{j_2} \rangle \big).
\end{equation}
Now, let $k_1,k_2 \in \mathbb{N}\setminus\{0,1\}$. 
When $k_2 \neq k_1$, either $j_1 \neq j_2$ so that $\theta_{j_1}$ and $\theta_{j_2}$ are independent, and we get
\begin{equation}
    \mathbb{E} P_{k_1}\big(\langle v, \theta_{j_1} \rangle \big) P_{k_2}\big(\langle v, \theta_{j_2} \rangle \big) = \mathbb{E}_{\theta_{j_1}} P_{k_1}\big(\langle v, \theta_{j_1} \rangle \big) \mathbb{E}_{\theta_{j_2}}P_{k_2}\big(\langle v, \theta_{j_2} \rangle \big) = 0,
\end{equation}
or $j_1 = j_2$, and we get
\begin{equation}
    \mathbb{E} P_{k_1}\big(\langle v, \theta_{j_1} \rangle \big) P_{k_2}\big(\langle v, \theta_{j_2} \rangle \big) = \int_{\mathbb{S}^{d-1}} P_{k_1}\big(\langle v, n \rangle \big) P_{k_2}\big(\langle v, n \rangle \big) \mathrm{d} \pi_{\mathbb{S}^{d-1}}(n) = 0,
\end{equation}
using the addition formula \eqref{eq:spherical_addition_formula} and the fact that $\int_{\mathbb{S}^{d-1}} Y_{k_{1},i_{1}}(n) Y_{k_{2},i_{2}}(n) \mathrm{d}\pi_{\mathbb{S}^{d-1}}(n) = 0$, for $i_{1}~\in~[N(d,k_1)]$ and $i_{2} \in [N(d,k_2)]$ when $k_{1} \neq k_{2}$. Thus
\begin{equation}\label{eq:proof_MC_2}
        \mathbb{E} \sum\limits_{j_1=1}^{M_S} \sum\limits_{j_2=1}^{M_S}P_{k_1}\big(\langle v, \theta_{j_1} \rangle \big) P_{k_2}\big(\langle v, \theta_{j_2} \rangle \big) = 0 .
\end{equation}

When $k_1 = k_2 = k$, we use 
the addition formula \eqref{eq:spherical_addition_formula} to prove that 
\begin{equation*}
     P_{k_1}(\langle v,\theta_{j_1} \rangle) P_{k_2}(\langle v,\theta_{j_2} \rangle)  = \frac{1}{N(d,k_1)N(d,k_2)}\sum\limits_{i_1 =1}^{N(d,k_1)} \sum\limits_{i_2=1}^{N(d,k_2)} Y_{k_1,i_1}(v) Y_{k_2,i_2}(v)Y_{k_1,i_1}(\theta_{j_1})   Y_{k_2,i_2}(\theta_{j_2}).
\end{equation*}
Moreover, since $k$ is non-negative, we have
\begin{equation}
    \mathbb{E} Y_{k_1,i_1}(\theta_{j_1})   Y_{k_2,i_2}(\theta_{j_2}) = \mathbb{E} Y_{k,i_1}(\theta_{j_1})   Y_{k,i_2}(\theta_{j_2}) =  \delta_{i_1,i_2} \delta_{j_1,j_2}.
\end{equation}
Therefore,
\begin{equation}
    \mathbb{E} P_{k_1}\big(\langle v, \theta_{j_1} \rangle \big) P_{k_2}\big(\langle v, \theta_{j_2} \rangle \big) = \frac{1}{N(d,k)}  \delta_{j_1,j_2},
\end{equation}
since 
\begin{equation}
  \sum\limits_{i=1}^{N(d,k)} Y_{k,i}(v) Y_{k,i}(v) = N(d,k)P_{k}(1) = N(d,k).
\end{equation}
Thus
\begin{equation}\label{eq:proof_MC_3}
        \mathbb{E} \sum\limits_{j_1=1}^{M_S} \sum\limits_{j_2=1}^{M_S}P_{k_1}\big(\langle v, \theta_{j_1} \rangle \big) P_{k_2}\big(\langle v, \theta_{j_2} \rangle \big) = \frac{M_S}{N(d,k)}  .
\end{equation}
Finally, observe that $|\frac{1}{M_{S}^2}  \sum_{k_1=2}^{K_1}\sum_{k_2=2}^{K_2} N(d,k_1)N(d,k_2) \lambda_{k_1}\lambda_{k_2}\sum_{j_1=1}^{M_{S}} \sum_{j_2=1}^{M_{S}}P_{k_1}\big(\langle v, \theta_{j_1} \rangle \big) P_{k_2}\big(\langle v, \theta_{j_2} \rangle \big)|$ is bounded by 
\begin{equation}
\sum_{k_1=2}^{K_1}\sum_{k_2=2}^{K_2} N(d,k_1)N(d,k_2) \lambda_{k_1}\lambda_{k_2} \leq \Big(\sum\limits_{k=2}^{+\infty} N(d,k) \lambda_{k}\Big)^2 <+\infty,
\end{equation}
for $K_{1}, K_{2} \in \mathbb{N} \setminus \{0,1\}$. Indeed,  by \eqref{eq:bound_polys_by_1}, for $k_{1}, k_{2} \in \mathbb{N} \setminus \{0,1\}$, we have
\begin{align}
\big| \sum_{j_1=1}^{M_{S}} \sum_{j_2=1}^{M_{S}}P_{k_1}\big(\langle v, \theta_{j_1} \rangle \big) P_{k_2}\big(\langle v, \theta_{j_2} \rangle \big) \big| & \leq \sum\limits_{j_1=1}^{M_{S}} \sum\limits_{j_2=1}^{M_{S}}|P_{k_1}\big(\langle v, \theta_{j_1} \rangle \big) P_{k_2}\big(\langle v, \theta_{j_2} \rangle \big)| \\ & \leq \sum\limits_{j_1=1}^{M_{S}} \sum\limits_{j_2=1}^{M_{S}} 1 = M_{S}^2.
\end{align}
Combining \eqref{eq:proof_MC_2} and \eqref{eq:proof_MC_3}, and using dominated convergence theorem, we get \eqref{eq:spherical_MC_bound}.


Next, we prove that 
\begin{equation}
    \sum\limits_{k=2}^{+\infty} N(d,k)\lambda_{k}^2 \leq  \frac{1}{2}\Big(\frac{r^2 \|x-y\|^2}{d-1} \Big)^2 e^{\frac{r^2 \|x-y\|^2}{d-1}}.
\end{equation}
For this we make use of \eqref{eq:lambda_k_bound} which states that 
\begin{equation}
|\lambda_{k}| \leq   \frac{\Gamma((d-1)/2)}{\Gamma(k+(d-1)/2)}\bigg(\frac{r \|x-y\|}{2}\bigg)^{k},  
\end{equation}
so that
\begin{equation}\label{eq:N_d_k_lambda_k_2}
    N(d,k)\lambda_{k}^2 \leq   N(d,k) \Big(\frac{\Gamma((d-1)/2)}{\Gamma(k+(d-1)/2)} \Big)^2 \bigg(\frac{r \|x-y\|}{2}\bigg)^{2k}.
\end{equation}
In the following, we derive an upper bound for 
\begin{equation}
    N(d,k) \frac{\Gamma^2((d-1)/2)}{\Gamma^2(k+(d-1)/2)}.
\end{equation}
For this purpose, we write
\begin{align}
    \frac{\Gamma((d-1)/2)}{\Gamma(k+(d-1)/2)} &=\frac{\Gamma((d-1)/2)}{\underbrace{(k+(d-1)/2-1)(k+(d-1)/2-2)\cdots((d-1)/2)}_{\text{k factors}}\Gamma((d-1)/2)}\\
    &=\frac{2^k}{(2k+d-3)(2k+d-5)\cdots (d+1) (d-1)}.
\end{align}
Thus, we get
\begin{align}
    \frac{\Gamma^2((d-1)/2)}{\Gamma^2(k+(d-1)/2)} &=\frac{2^{2k}}{(2k+d-3)^2(2k+d-5)^2\cdots (d+1)^2 (d-1)^2}\\
    &\leq\frac{2^{2k}}{(2k+d-3)(2k+d-4)(2k+d-5)\cdots (d-1) (d-1)} \\
     &\leq\frac{2^{2k} (d-2)!}{(d-1)(2k+d-3)!} .
\end{align}
Moreover, we have by~\eqref{eq:dim_spherical_harmonics_k}
\begin{equation}
N(d,k) = \frac{(2k+d-2)(k+d-3)!}{k!(d-2)!},
\end{equation}
so that
\begin{align}\label{eq:N_dk_d1_kd1}
N(d,k) \frac{\Gamma^2((d-1)/2)}{\Gamma^2(k+(d-1)/2)} & \leq \frac{(2k+d-2)(k+d-3)!}{k!(d-2)! } \frac{2^{2k} (d-2)!}{(d-1)(2k+d-3)!} \nonumber \\
& \leq \frac{2^{2k}}{k!(d-1)} \frac{(2k+d-2) }{(2k+d-3) \dots (k+d-2)} \nonumber \\
& \leq \frac{2^{2k}}{k!(d-1)^k},
\end{align}
since, when $k\geq 2$, we have
\begin{align}
    &\frac{(d+2k-2)}{(d-1)(d+2k-3)\dots(d+k-2)} \\
    &= \frac{(d+2k-2)(d+k-3)}{(d-1)(d+2k-3)\dots(d+k-2)(d+k-3)}\\
    &=\frac{1}{d-1}\underbrace{\frac{(d+2k-2)(d+k-3)}{(d+2k-3)(d+k-2)}}_{\leq1}\frac{1}{(d+2k-4)\dots(d+k-1)(d+k-3)}\\
    &\leq\frac{1}{d-1}\frac{1}{\underbrace{(d+2k-4)\dots(d+k-1)(d+k-3)}_{k-1 \text{ terms}}}\\
    &\leq\frac{1}{(d-1)^k}.
\end{align}
Therefore, by~\eqref{eq:N_d_k_lambda_k_2}, we have
\begin{align}
    \sum\limits_{k=2}^{+\infty} N(d,k)\lambda_{k}^2 \leq  \sum\limits_{k=2}^{+\infty}   \frac{2^{2k}}{k!(d-1)^k} \bigg(\frac{r \|x-y\|}{2}\bigg)^{2k}  & = \sum\limits_{k=2}^{+\infty}   \frac{1}{k!}\frac{(r \|x-y\|)^{2k}}{(d-1)^k} \\
     & =  \sum\limits_{k=2}^{+\infty}   \frac{1}{k!} \Bigg( \frac{r^2 \|x-y\|^2}{d-1} \Bigg)^k.
\end{align}
Finally, by \Cref{lemma:exp_truncated}, for $d\geq 2$ we get
\begin{equation}
    \sum\limits_{k=2}^{+\infty} N(d,k)\lambda_{k}^2 \leq  \frac{1}{2}\Big(\frac{r^2 \|x-y\|^2}{d-1} \Big)^2 e^{\frac{r^2 \|x-y\|^2}{d-1}}.
\end{equation}

\subsection{Proof of \Cref{prop:block_iid_spherical_quadrature}}\label{proof:block_iid_spherical_quadrature}
Let $r \in \mathbb{R}_{+}^{*}$ and $x,y \in \mathbb{R}^{d}$. When $r(x-y) = 0$, the proof of \eqref{eq:block_iid_spherical_quadrature} follows the same steps as in~\eqref{proof:prop_iid_spherical_quadrature}.

In the following, we assume that $r(x-y) \neq 0$, and we make use of \eqref{eq:error_spherical_r_xy_1}
\begin{equation}\label{eq:error_spherical_r_xy_2}
        \mathcal{E}_{S}(f_{r(x-y)})^2 = \frac{1}{M_{S}^2}  \sum\limits_{k_1=2}^{+\infty}\sum\limits_{k_2=2}^{+\infty} N(d,k_1)N(d,k_2) \lambda_{k_1}\lambda_{k_2} \sum\limits_{j_1=1}^{M_{S}} \sum\limits_{j_2=1}^{M_{S}}P_{k_1}\big(\langle v, \theta_{j_1} \rangle \big) P_{k_2}\big(\langle v, \theta_{j_2} \rangle \big),
\end{equation}
proved in \Cref{proof:prop_iid_spherical_quadrature}.


Now, let $k_{1}, k_{2} \in \mathbb{N}\setminus \{0,1\}$. We treat two cases: $k_1 \neq k_2$ and $k_1 = k_2$.

\paragraph{In the case $k_1 \neq k_2$}, we have
\begin{equation}
    \mathbb{E} \sum\limits_{j_1=1}^{M_{S}} \sum\limits_{j_2=1}^{M_{S}}P_{k_1}\big(\langle v, \theta_{j_1} \rangle \big) P_{k_2}\big(\langle v, \theta_{j_2} \rangle \big) = \sum\limits_{j_1=1}^{M_{S}} \sum\limits_{j_2=1}^{M_{S}} \mathbb{E} P_{k_1}\big(\langle v, \theta_{j_1} \rangle \big) P_{k_2}\big(\langle v, \theta_{j_2} \rangle \big).
\end{equation}
Let $(j_1, j_2) \in [M_{S}]^2$. The random variable $\theta_{j_1}$ follows the uniform distribution on $\mathbb{S}^{d-1}$, so that by the Hecke-Funk formula \eqref{eq:hecke_funk}, we have
\begin{equation}
    \mathbb{E}_{\theta_{j_1}} P_{k_1}(\langle v, \theta_{j_1} \rangle)  = \int_{\mathbb{S}^{d-1}} P_{k_1}(\langle v, n \rangle) \mathrm{d}\pi_{\mathbb{S}^{d-1}}(n) = \frac{\mathcal{V}_{d-2}}{\mathcal{V}_{d-1}} \int_{-1}^{1} P_{k_1}(t) (1-t^2)^{(d-3)/2} \mathrm{d}t = 0.
\end{equation}
Now, when $\theta_{j_1}$ and $\theta_{j_2}$ come from two different orthonormal matrices they are independent and we have
\begin{equation}
    \mathbb{E} P_{k_1}\big(\langle v, \theta_{j_1} \rangle \big) P_{k_2}\big(\langle v, \theta_{j_2} \rangle \big) = \mathbb{E}_{\theta_{j_1}} P_{k_1}\big(\langle v, \theta_{j_1} \rangle \big) \mathbb{E}_{\theta_{j_2}} P_{k_2}\big(\langle v, \theta_{j_2} \rangle \big) = 0.
\end{equation}
Otherwise, they are not independent, and
we make use of \Cref{lemma:formula_omc}, and we still get
\begin{equation}
    \mathbb{E} P_{k_1}\big(\langle v, \theta_{j_1} \rangle \big) P_{k_2}\big(\langle v, \theta_{j_2} \rangle \big) = 0.
\end{equation}
Thus
\begin{equation}\label{eq:exp_P_k_neq_k}
    \mathbb{E} \sum\limits_{j_1=1}^{M_{S}} \sum\limits_{j_2=1}^{M_{S}}P_{k_1}\big(\langle v, \theta_{j_1} \rangle \big) P_{k_2}\big(\langle v, \theta_{j_2} \rangle \big) = 0.
\end{equation}

\paragraph{In the case $k_1 = k_2 = k \neq 2$.} Let $j_1, j_2 \in [M_{S}]^2$. When $\theta_{j_1}$ and $\theta_{j_2}$ come from two different orthonormal matrices, they are independent, and we have
\begin{equation}
   \mathbb{E} P_{k}\big(\langle v, \theta_{j_1} \rangle \big) P_{k}\big(\langle v, \theta_{j_2} \rangle \big) = 0.
\end{equation}
Otherwise, we use again \Cref{lemma:formula_omc}, and we get
\begin{equation}
   \mathbb{E} P_{k}\big(\langle v, \theta_{j_1} \rangle \big) P_{k}\big(\langle v, \theta_{j_2} \rangle \big) \leq \frac{2}{d-1} \int_{\mathbb{S}^{d-1}}  P_{k}\big(\langle v, \theta \rangle \big)^2 \mathrm{d}\pi_{\mathbb{S}^{d-1}}(\theta).
\end{equation}
Thus
\begin{align}\label{eq:exp_P_k_k_neq_2}
    \mathbb{E}\sum\limits_{j_1=1}^{M_{S}} \sum\limits_{j_2=1}^{M_{S}}P_{k}\big(\langle v, \theta_{j_1} \rangle \big) P_{k}\big(\langle v, \theta_{j_2} \rangle \big) 
    & \leq \frac{M_{S}}{d} \Big(\frac{2d(d-1)}{d-1} +d \Big) \int_{\mathbb{S}^{d-1}} P_{k}(\langle v, \theta \rangle)^2 \mathrm{d} \pi_{\mathbb{S}^{d-1}}(\theta) \nonumber \\
    & \leq 3M_S\int_{\mathbb{S}^{d-1}} P_{k}(\langle v, \theta \rangle)^2 \mathrm{d} \pi_{\mathbb{S}^{d-1}}(\theta).
\end{align}

\paragraph{In the case $k_{1} = k_{2} = 2$.} We have 
\begin{equation}
    \mathbb{E}\sum\limits_{j_1=1}^{M_{S}} \sum\limits_{j_2=1}^{M_{S}}P_{k}\big(\langle v, \theta_{j_1} \rangle \big) P_{k}\big(\langle v, \theta_{j_2} \rangle \big)  = \mathbb{E}\Big(\sum\limits_{j=1}^{M_{S}} P_{2}\big(\langle v, \theta_{j} \rangle \big) \Big) ^2.
\end{equation}
Now, we prove that 
\begin{equation}\label{eq:sum_P_2_0}
 \forall v \in \mathbb{S}^{d-1},\:\:   \sum\limits_{j=1}^{M_{S}} P_{2}\big(\langle v, \theta_{j_1} \rangle \big) = 0,
\end{equation}
so that
\begin{equation}\label{eq:sum_P_2_0_bis}
    \forall v \in \mathbb{S}^{d-1},\:\: \mathbb{E}\sum\limits_{j_1=1}^{M_{S}} \sum\limits_{j_2=1}^{M_{S}}P_{k}\big(\langle v, \theta_{j_1} \rangle \big) P_{k}\big(\langle v, \theta_{j_2} \rangle \big)  = 0.
\end{equation}
For this purpose, observe that, by definition of $P_{2}$ given by \ref{def:P_k}, we have
\begin{equation}
    P_{2}(t) = \frac{1}{d-1} \Big(dt^2 -1 \Big),
\end{equation}
so that for $m \in [M_{S}/d -1]$, we have
\begin{equation}\label{eq:sum_P_2_orthormal}
   \sum\limits_{j=1 + m \cdot d}^{(1+m)\cdot d} P_{2}(\langle v, \theta_j \rangle) = \frac{1}{d-1} \Big(d \sum\limits_{j=1 + m\cdot d}^{(1+m) \cdot d} \langle v, \theta_j \rangle^2 - \sum\limits_{j=1 + m\cdot d}^{1+m }1 \Big) = \frac{1}{d-1} \Big(d \|v\|^2 -d \Big) = 0,
\end{equation}
where we have used the fact that for $m \in [M_{S}/d -1]$, $(\theta_{1+ m \cdot d}, \dots, \theta_{(1+m) \cdot d} )$ forms an orthonormal basis of $\mathbb{R}^{d}$. By summing~\eqref{eq:sum_P_2_orthormal} over $m \in [M_{S}/d -1]$, we prove \eqref{eq:sum_P_2_0}.

Finally, combining \eqref{eq:error_spherical_r_xy_2}, \eqref{eq:exp_P_k_neq_k}, \eqref{eq:exp_P_k_k_neq_2} and \eqref{eq:sum_P_2_0_bis}, and the fact that $\lambda_{k} = 0$ when $k$ is odd, we get

\begin{equation}
  \mathbb{E}  \mathcal{E}(f_{r(x-y)})^2 \leq \frac{3}{M_{S}}   
     \sum\limits_{k = 4}^{+\infty} N(d,k)^2 \lambda_{k}^2 \int_{\mathbb{S}^{d-1}} P_{k}\big(\langle v, \theta \rangle \big)^2 \mathrm{d} \pi_{\mathbb{S}^{d-1}}(\theta).
\end{equation}
Moreover, by \eqref{eq:int_Sd_P_k_2} we have 
\begin{equation}
   \int_{\mathbb{S}^{d-1}} P_{k}\big(\langle v, \theta \rangle \big)^2 \mathrm{d} \pi_{\mathbb{S}^{d-1}}(\theta) =  \frac{1}{N(d,k)}.
\end{equation}
Therefore  
\begin{equation}
     \mathbb{E}\mathcal{E}(f_{r(x-y)})^2 \leq \frac{3}{M_{S} }   
     \sum\limits_{k = 4}^{+\infty} N(d,k) \lambda_{k}^2.
\end{equation}

The proof of \eqref{eq:block_iid_spherical_quadrature_bis} follows the same steps as in \Cref{proof:prop_iid_spherical_quadrature} by making use of \eqref{eq:N_d_k_lambda_k_2} and \eqref{eq:N_dk_d1_kd1} and  \Cref{lemma:exp_truncated}.




\section{Lemmata}
\subsection{An upper bound on the truncation of the exponential function}
\begin{lemma}\label{lemma:exp_truncated}
For $M \in \mathbb{N}$, we have 
\begin{equation}
\forall x \in \mathbb{R}_{+}, \:\: \Big|\sum\limits_{m = M+1}^{+\infty} \frac{1}{m!} x^{m} \Big| \leq \frac{x^{M+1}}{(M+1)!} e^x.
\end{equation}
\end{lemma}
\begin{proof}
Let $\varphi$ be the $\exp$ function. Using the Taylor's Theorem with integral form of remainder, we get for $x \in \mathbb{R}_{+}$,
\begin{equation}
\varphi(x) = \sum\limits_{m = 0}^{M} \varphi^{(m)}(0) \frac{x^m}{m!} + \int_{0}^{x} \varphi^{(M+1)}(t) \frac{(x-t)^{M}}{M!} \mathrm{d} t.
\end{equation}
So that
\begin{align*}
    \Big| \varphi(x) - \sum\limits_{m = 0}^{M} \varphi^{(m)}(0) \frac{x^m}{m!} \Big| &= \Big| \int_{0}^{x} \varphi^{(M+1)}(t) \frac{(x-t)^{M}}{M!} \mathrm{d} t \Big| \\
    & \leq \sup\limits_{t \in [0,x]} |\varphi^{(M+1)}(t) | \Big| \int_{0}^{x} \frac{(x-t)^{M}}{M!} \mathrm{d} t \Big| \\
    & \leq  e^{x} \frac{x^{M+1}}{(M+1)!}.
\end{align*}
We conclude, by observing that 
\begin{equation}
    \Big| \sum\limits_{m=M+1}^{+\infty}  \frac{x^{m}}{m!} \Big| = \Big| \varphi(x) - \sum\limits_{m = 0}^{M} \varphi^{(m)}(0) \frac{x^m}{m!} \Big|. 
\end{equation}

\end{proof}

\subsection{Error bound for the generalized Gauss-Laguerre quadrature rule}
\begin{lemma}\label{lem:quad_err_for_Laguerre}
There exists $L>0$ such that for any $M \in \mathbb{N}^*$, we have for $\varphi \in L_{2}(p_{\Xi})$,
\begin{equation}
 \sum\limits_{m=2M}^{+\infty}\sqrt{m} |\langle \varphi, \ell_{m}^{\alpha} \rangle_{p_{\Xi}}| < +\infty \implies \Big| \int_{\mathbb{R}_{+}} \varphi(\xi) p_{\Xi}(\xi) \mathrm{d}\xi - \sum\limits_{i=1}^{M}a_{i} \varphi(\xi_i) \Big| \leq L \sum\limits_{m=2M}^{+\infty}\sqrt{m} |\langle \varphi, \ell_{m}^{\alpha} \rangle_{p_{\Xi}}|,
\end{equation}
where the $(\ell^{\alpha}_{n})_{n \in \mathbb{N}}$ is family of normalized generalized Laguerre polynomials, and the $a_{1}, \dots, a_{M}$ and the $\xi_{1}, \dots, \xi_{M}$ are respectively the weights and the nodes of the Gaussian quadrature associated to the weight function $p_{\Xi}$. 
 
\end{lemma}

\begin{proof}
First, we prove the existence of a universal constant $L>0$ that satisfies
\begin{equation}\label{eq:condition_on_l_m}
\forall m \in \mathbb{N}, \:\: |\sum\limits_{i=1}^M a_{i} \ell^{\alpha}_{m}(\xi_i)| \leq L.
\end{equation}

For this purpose, we use Theorem 2.2 in \cite{CaCr} that shows the existence of a universal constant $L'>0$ that does not depend on $M$ such that 
\begin{equation}\label{eq:laguerre_bound_L}
    \Big| \int_{0}^{+\infty} \varphi(\xi) p_{\Xi}(\xi) \mathrm{d}\xi - \sum\limits_{i=1}^{M} a_{i} \varphi(\xi_i) \Big| \leq L' \int_{0}^{+\infty} |\varphi^{'}(\xi) \sqrt{\xi} p_{\Xi}(\xi) | \mathrm{d} \xi .
\end{equation}
Now, take $\varphi(\xi) := \ell_{m}^{\alpha}(\xi)$. By using~\eqref{eq:laguerre_bound_L}, we get
\begin{align}
   \Big| \int_{0}^{+\infty} \ell_{m}^{\alpha}(\xi) p_{\Xi}(\xi) \mathrm{d}\xi - \sum\limits_{i=1}^{M} a_{i} \ell_{m}^{\alpha}(\xi_i)\Big| & \leq L' \int_{0}^{+\infty} | \ell_{m}^{' \alpha}(\xi)  \sqrt{\xi} p_{\Xi}(\xi) | \mathrm{d} \xi \nonumber \\
   & \leq L' \left(\int_{0}^{+\infty}  p_{\Xi}(\xi) \mathrm{d} \xi\right)^{1/2}\left(\int_{0}^{+\infty}  \ell_{m}^{' \alpha}(\xi)^2  \xi p_{\Xi}(\xi) \mathrm{d} \xi \right)^{1/2}\nonumber  \\
      & = L' \left(\int_{0}^{+\infty}  \ell_{m}^{' \alpha}(\xi)^2  \xi p_{\Xi}(\xi) \mathrm{d} \xi  \right)^{1/2}.
\end{align}
Now, for the un-normalized generalized Laguerre polynomial, we have
\begin{equation}
    L_{m}^{'\alpha}(x) = - L_{m-1}^{\alpha+1}(x) ,
\end{equation}
thus
\begin{equation}
    \ell_{m}^{'\alpha}(\xi) = - \sqrt{\frac{m!}{\Gamma(m+\alpha+1)}}L_{m-1}^{\alpha+1}(\xi)= - \sqrt{m} \ell_{m-1}^{\alpha+1}(\xi).
\end{equation}
\begin{equation}
    \int_{0}^{+\infty}  \ell_{m}^{' \alpha}(\xi)^2  \xi p_{\Xi}(\xi) \mathrm{d} \xi = \frac{\Gamma(\alpha+2)}{\Gamma(\alpha+1)} m = (\alpha+1)m.
\end{equation}
Therefore, by the fact that $\int_{0}^{+\infty} \ell^{\alpha}_{m}(\xi) p_{\Xi}(\xi) \mathrm{d}\xi = 0$ when $m > 0$, we have 
\begin{equation}
    \Big|\sum\limits_{i=1}^{M} a_{i} \ell_{m}^{\alpha}(\xi_i) \Big| \leq  L' \sqrt{(\alpha+1)m},
\end{equation}
and we take $L = \sqrt{\alpha +1}L'$.

Now, let $\varphi \in L_{2}(p_{\Xi})$, and write
\begin{equation}
    \varphi = \sum\limits_{m=0}^{+\infty} \langle \varphi , \ell_{m}^{\alpha} \rangle_{p_{\Xi}} \ell_{m}^{\alpha}.
\end{equation}
By the exactness of the quadrature for polynomials of order smaller than $2M-1$, we get
\begin{equation*}
    \Big| \int_{\mathbb{R}_{+}} \varphi(\xi) p_{\Xi}(\xi) \mathrm{d}\xi - \sum\limits_{i=1}^{M}a_{i} \varphi(\xi_i) \Big| \leq  \sum\limits_{m = 2M}^{+\infty} |\langle \varphi, \ell_{m}^{\alpha} \rangle_{p_{\Xi}}| |\sum\limits_{i=1}^{M} a_{i} \ell_{m}^{\alpha}(\xi_i)| \leq L \sum\limits_{m = 2M}^{+\infty} \sqrt{m}|\langle \varphi, \ell_{m}^{\alpha} \rangle_{p_{\Xi}}|.
\end{equation*}


\end{proof}

\subsection{A formula of covariance under orthogonal Monte Carlo}

\begin{lemma}\label{lemma:formula_omc}
Let $\theta_{1}, \dots, \theta_{d}$ be the columns of a random orthogonal matrix sampled from the Haar distribution on $\mathbb{O}_{d}(\mathbb{R})$. Let $k, k' \in \mathbb{N}^{*}$. When $k \neq k'$, we have
\begin{equation}
\mathbb{E} P_{k}(\langle v, \theta_{i} \rangle) P_{k'}(\langle v, \theta_{j} \rangle) = 0. 
\end{equation}
Otherwise,
\begin{equation}
\forall v \in \mathbb{S}^{d-1}, \:\:    \mathbb{E} P_{k}(\langle \theta_{i}, v \rangle )P_{k}(\langle \theta_{j}, v \rangle )  \leq \frac{2}{d-1} \int_{\mathbb{S}^{d-1}} P_{k}(v, \theta)^2 \mathrm{d} \pi_{\mathbb{S}^{d-1}}(\theta).
\end{equation}
\end{lemma}

\begin{proof}

    Let $i,j \in [d]$ such that $i \neq j$. Since $\theta_{i}$ and $\theta_{j}$ are the columns of an orthogonal matrix, they are orthogonal and they follow marginally the uniform distribution on $\mathbb{S}^{d-1}$. Now, given $\theta_{i}$, the random variable $\theta_{j}$ follows the uniform distribution on $\mathbb{S}^{d-1} \cap \theta_{i}^{\perp}$, denoted $\pi_{\mathbb{S}^{d-1}(\theta_i ^{\perp})}$. Therefore, given $\psi_{1}, \psi_{2}: \mathbb{S}^{d-1} \rightarrow \mathbb{R}$, we have
    \begin{align}\label{eq:ex_equal_operator}
        \mathbb{E} \psi_{1}(\theta_i) \psi_{2}(\theta_j) & = \int_{\mathbb{S}^{d-1}} \psi_{1}(\theta_{i}) \int_{\mathbb{S}^{d-1} \cap \theta_{i}^{\perp}} \psi_{2}(\theta_j) \mathrm{d} \pi_{\mathbb{S}^{d-1} \cap \theta_{i}^{\perp}}(\theta_j) \mathrm{d} \pi_{\mathbb{S}^{d-1}}(\theta_i)  \nonumber \\
         & = \int_{\mathbb{S}^{d-1}} \psi_{1}(\theta_{i}) (T \psi_{2} )(\theta_i) \mathrm{d} \pi_{\mathbb{S}^{d-1}}(\theta_i),
    \end{align}
where $T: L_{2}(\pi_{\mathbb{S}^{d-1}}) \rightarrow L_{2}(\pi_{\mathbb{S}^{d-1}})$ is defined by
\begin{equation}
    (T \psi)(\theta) := \int_{\mathbb{S}^{d-1} \cap \theta^{\perp}} \psi_{2}(\theta')  \mathrm{d} \pi_{\mathbb{S}^{d-1} \cap \theta^{\perp}}(\theta').
\end{equation}
Now, let $k,k' \in \mathbb{N}^{*}$. By~\eqref{eq:ex_equal_operator}, we have
\begin{equation}
\forall v \in \mathbb{S}^{d-1}, \:\:    \mathbb{E} P_{k}(\langle v, \theta_i \rangle) P_{k'}(\langle v, \theta_j \rangle ) = \int_{\mathbb{S}^{d-1}} P_{k}(v,\theta_{i}) (T P_{k'}(v,. )(\theta_i)) \mathrm{d} \pi_{\mathbb{S}^{d-1}}(\theta_i).
\end{equation}
By Theorem 3 in \cite{GoLeTuZa17}, the self-adjoint operator $T$ is bounded and has pure point spectrum. Moreover, $T$ is a multiple of the identity on $\mathcal{P}_{\ell}$ the space of harmonic homogeneous polynomials of degree $\ell$. 
Therefore $T P_{k'}(\langle v,. \rangle)$ is a homogenous polynomial of degree $k'$. Thus, when $k \neq k'$, we have
\begin{equation}
    \int_{\mathbb{S}^{d-1}} P_{k}(v,\theta_{i}) (T P_{k'}(v,. )(\theta_i) \mathrm{d} \pi_{\mathbb{S}^{d-1}}(\theta_i) = 0.
\end{equation}
On the other hand, when $k=k'$, we use again Theorem 3 in \cite{GoLeTuZa17}  which states that $P_{k'}(v,.)$ is an eigenfunction of $T$. Moreover, when $d \geq 4$  the corresponding eigenvalue is smaller than $(d-1)^{-1}$, so that  
\begin{equation}
\int_{\mathbb{S}^{d-1}} P_{k}(v,\theta_{i}) (T P_{k'}(v,. )(\theta_i) \mathrm{d} \pi_{\mathbb{S}^{d-1}}(\theta_i)  \leq \frac{1}{d-1} \int_{\mathbb{S}^{d-1}} P_{k}(v,\theta_{i})^2 \mathrm{d} \pi_{\mathbb{S}^{d-1}}(\theta_i).
\end{equation}
Finally, for $d =2$ and $d=3$, the corresponding eigenvalue is smaller or equal to $1$, so that 
\begin{equation}
\int_{\mathbb{S}^{d-1}} P_{k}(v,\theta_{i}) (T P_{k'}(v,. )(\theta_i) \mathrm{d} \pi_{\mathbb{S}^{d-1}}(\theta_i)  \leq \frac{2}{d-1} \int_{\mathbb{S}^{d-1}} P_{k}(v,\theta_{i})^2 \mathrm{d} \pi_{\mathbb{S}^{d-1}}(\theta_i).
\end{equation}


\end{proof}

\subsection{Integral involving generalized Laguerre polynomials}
\begin{lemma}\label{lem:moment_laguerre}
    The generalized Laguerre polynomials satisfies the following equation
    \begin{align}
        \int_{\mathbb{R}^+} x^{\alpha'-1}e^{-x}L_n^\alpha(x)\mathrm{d}x &= \binom{\alpha-\alpha'+n}{n}\Gamma(\alpha')
    \end{align}
    for arbitrary $\alpha, \alpha'\in\mathbb{N}$.
\end{lemma}
\begin{proof}
    The generalized Laguerre polynomial can be written into
    \begin{equation}
        L^\alpha_n(x) = \frac{x^{-\alpha}e^x}{n!}\frac{\mathrm{d}^n}{\mathrm{d}x^n}(e^{-x}x^{n+\alpha}).
    \end{equation}

    Integration by parts n times yields
    \begin{align}
        &\int_{\mathbb{R}^+} x^{\alpha'-1}e^{-x}L_n^\alpha(x)\mathrm{d}x=\frac{1}{n!}\int_{\mathbb{R}^+}x^{\alpha'-1-\alpha}\frac{\mathrm{d}^n}{\mathrm{d}x^n}(e^{-x}x^{n+\alpha})\mathrm{d}x\\
        &\stackrel{\text{IBP}}{=}\underbrace{\frac{1}{n!}\left[x^{\alpha'-1-\alpha}\frac{\mathrm{d}^{n-1}}{\mathrm{d}x^{n-1}}(e^{-x}x^{n+\alpha})\right]^\infty_0 }_{=0}-\frac{\alpha'-1-\alpha}{n!}\int_{\mathbb{R}^+}x^{\alpha'-2-\alpha}\frac{\mathrm{d}^{n-1}}{\mathrm{d}x^{n-1}}(e^{-x}x^{n+\alpha})\mathrm{d}x\\
        &\text{conduct integration by parts n times}\\
        &=0+(-1)^n\frac{(\alpha'-1-\alpha)\dots(\alpha'-n-\alpha)}{n!}\int_{\mathbb{R}^+}e^{-x}x^{\alpha'-1}\mathrm{d}x\\
        &=\binom{\alpha-\alpha'+n}{n}\Gamma(\alpha').
    \end{align}
    where
    \begin{equation}
        \binom{m-n -1  }{m} := \frac{(m-n-1)(m-n-2)\dots(-n)}{m!}
    \end{equation}
    is equal to $0$ when $0\leq n\leq m-1$ and $(-1)^m\binom{n}{m}$ when $n \geq m$. 
\end{proof}

\section{Additional experiments}
\subsection{Radial rule analysis}\label{sec:optimal_radial_nodes}
When designing the spherical-radial rule, a key part is to find the optimal number of radial quadrature nodes. To answer this question, we numerically computed the optimal number of radial nodes for different kernel bandwidth and dataset dimension, which is explained well by Theorem~\ref{thm:main_theorem_OMC}. When the number of nodes is too small, all the quadrature nodes concentrated in several hyperspheres and ignored other regions. On the other hand, when the number of nodes is too big, most of the quadrature weights are approaching zero and have trivially to the evaluation. 

In Figure~\ref{fig:radial_nodes_analysis} and Figure~\ref{fig:radial_nodes_analysis_real}, we showed the error in relative Frobenius norm for approximating $5000\times5000$ kernel matrices constructed from synthetic and real datasets. The result is averaged over 100 runs. When dimension is small (4 or 8) and total number of features is within $1000$, the optimal number of radial nodes is 2. When dimension $\approx 16$, the optimal number of radial nodes becomes 1 when feature number goes beyond $100$. When dimension is bigger ($>16$), the optimal number of radial nodes is 1. In kernel approximation and prediction tasks, we use this as a guide to choose radial design, which yields desirable result. 

Why we observe such change of optimal design when dimension increases? Recall that the upper bound for the approximation error given in Equation~\eqref{eq:upper_bound_error_thm_1} consists of two parts: the radial error and the spherical error. Keep $M_R$ fixed, the spherical error dominates when the number of spherical nodes $M_S$ is small; when $M_S$ grows, the spherical error decays with $1/M_S$ so we would see the error decays as feature number increases. But after some time the error reaches a plateau because the radial error stays unchanged. When dimension is higher, the spherical error is bigger and we need to take the number of features very large to observe the plateau.

When we keep the number of total features fixed and increases $M_R$, the increase in the spherical error is faster than the decay in the radial error, hence in general when $M_R\geq 2$, larger $M_R$ yields bigger error. The idea that we only need one radial node in high dimension is in line with the concentration of Gaussian measure in high dimension.

\begin{figure} 
\hfill small $\sigma$ \hfill\hfill medium $\sigma$ \hfill\hfill large $\sigma$ \hfill{} 
\medskip

\begin{minipage}{0.33\textwidth}
\includegraphics[width=\linewidth]{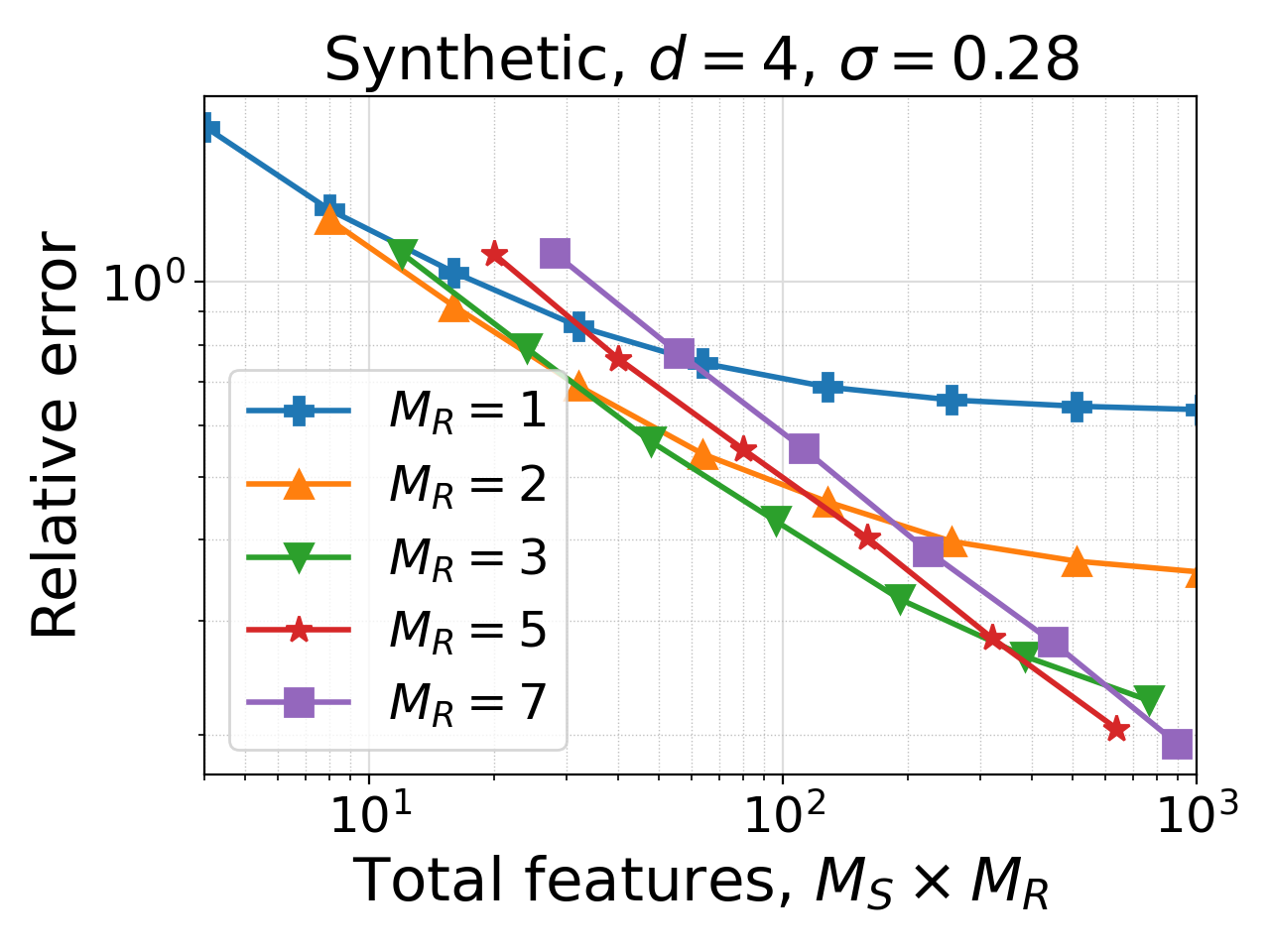} 
\includegraphics[width=\linewidth]{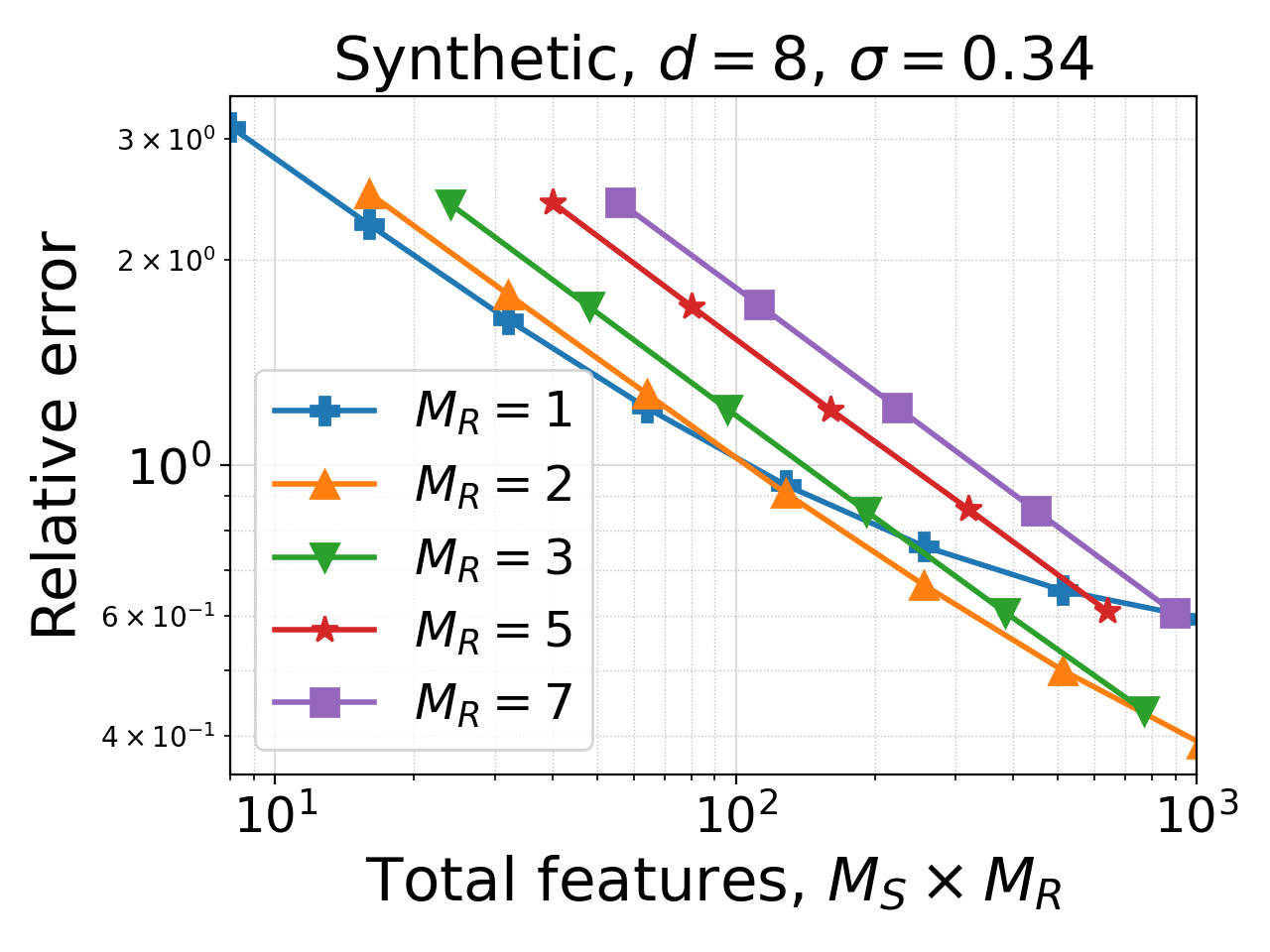} 
\includegraphics[width=\linewidth]{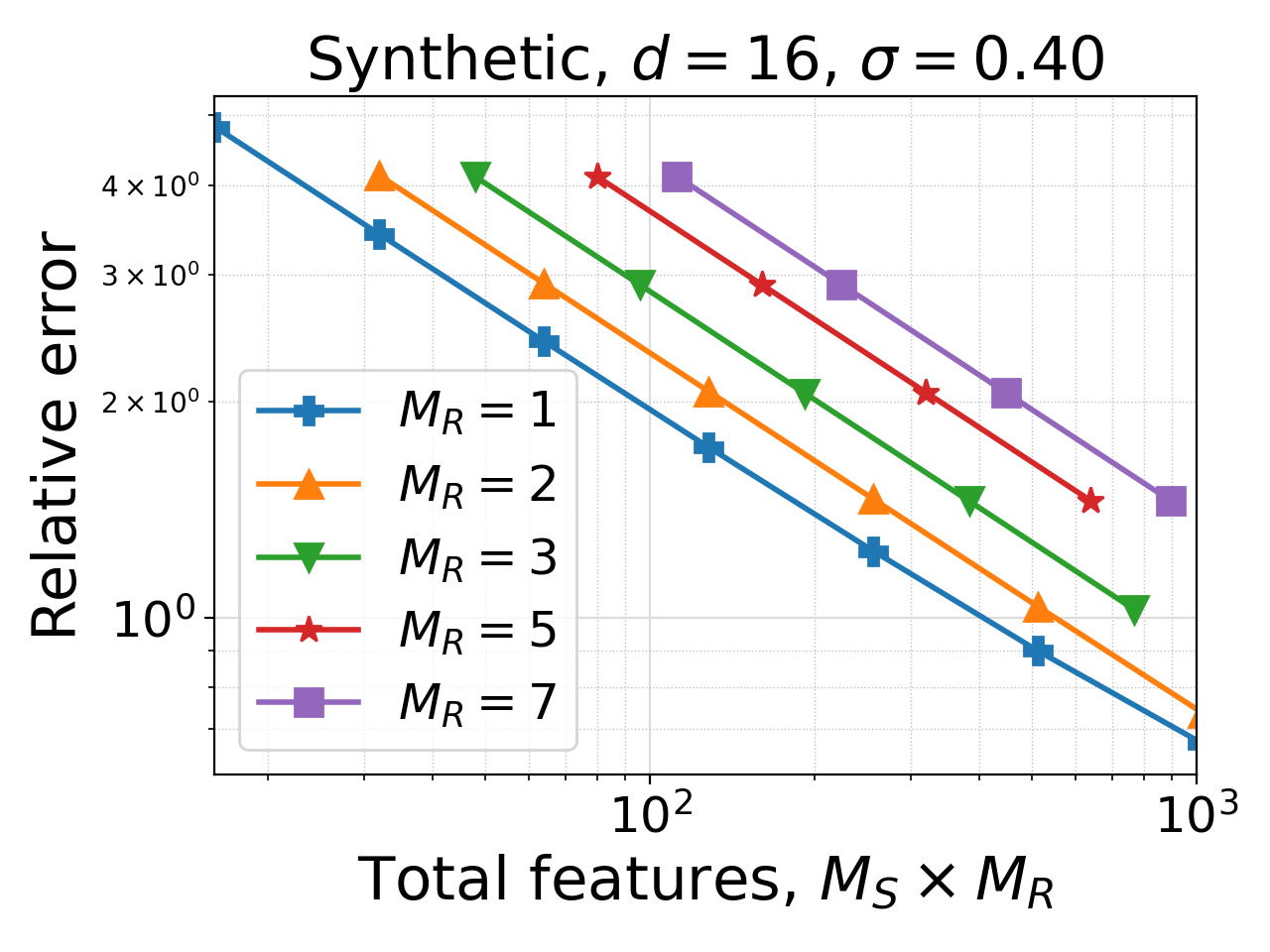} 
\includegraphics[width=\linewidth]{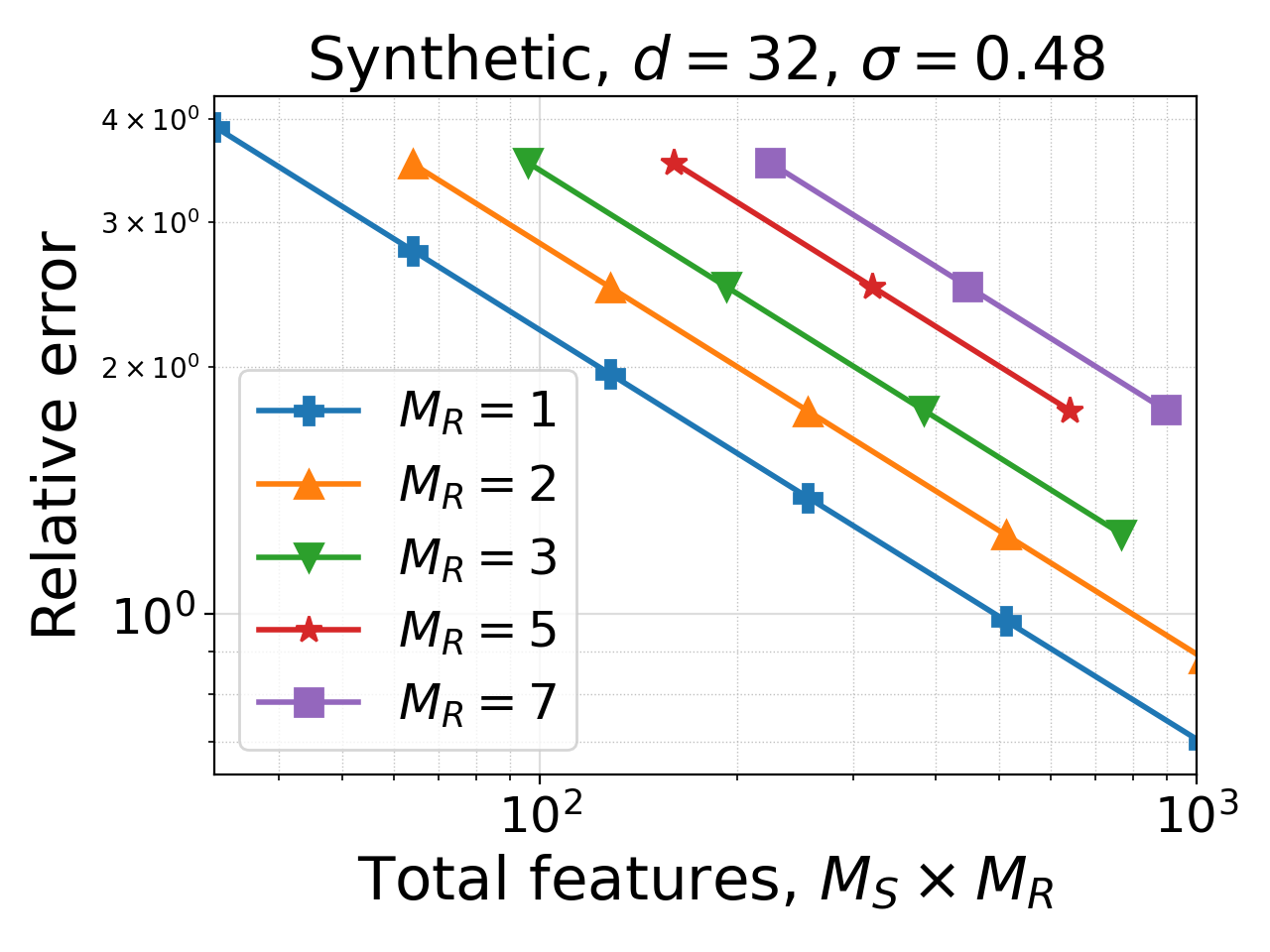}
\end{minipage}\hfill
\begin{minipage}{0.33\textwidth}
\includegraphics[width=\linewidth]{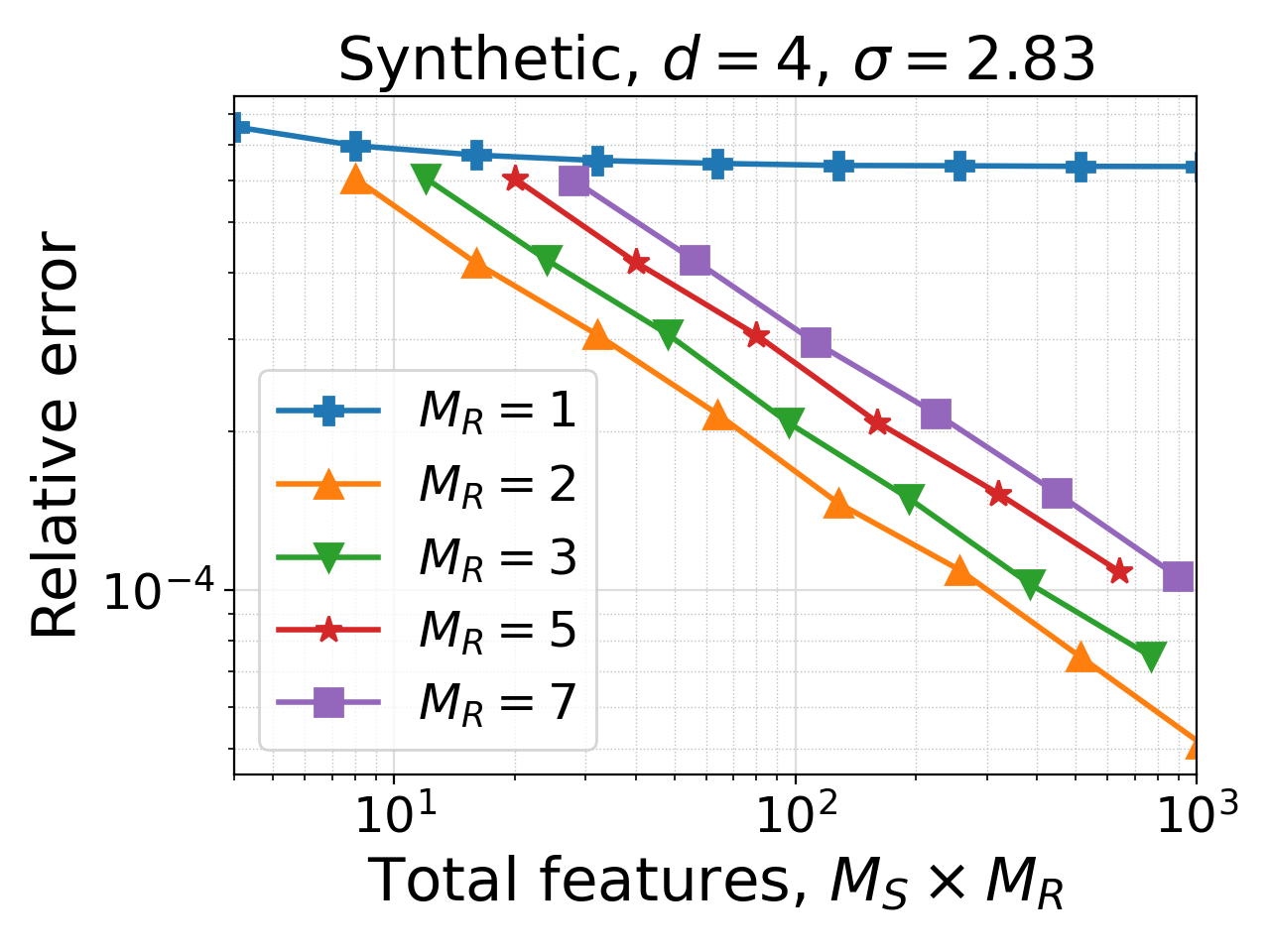}
\includegraphics[width=\linewidth]{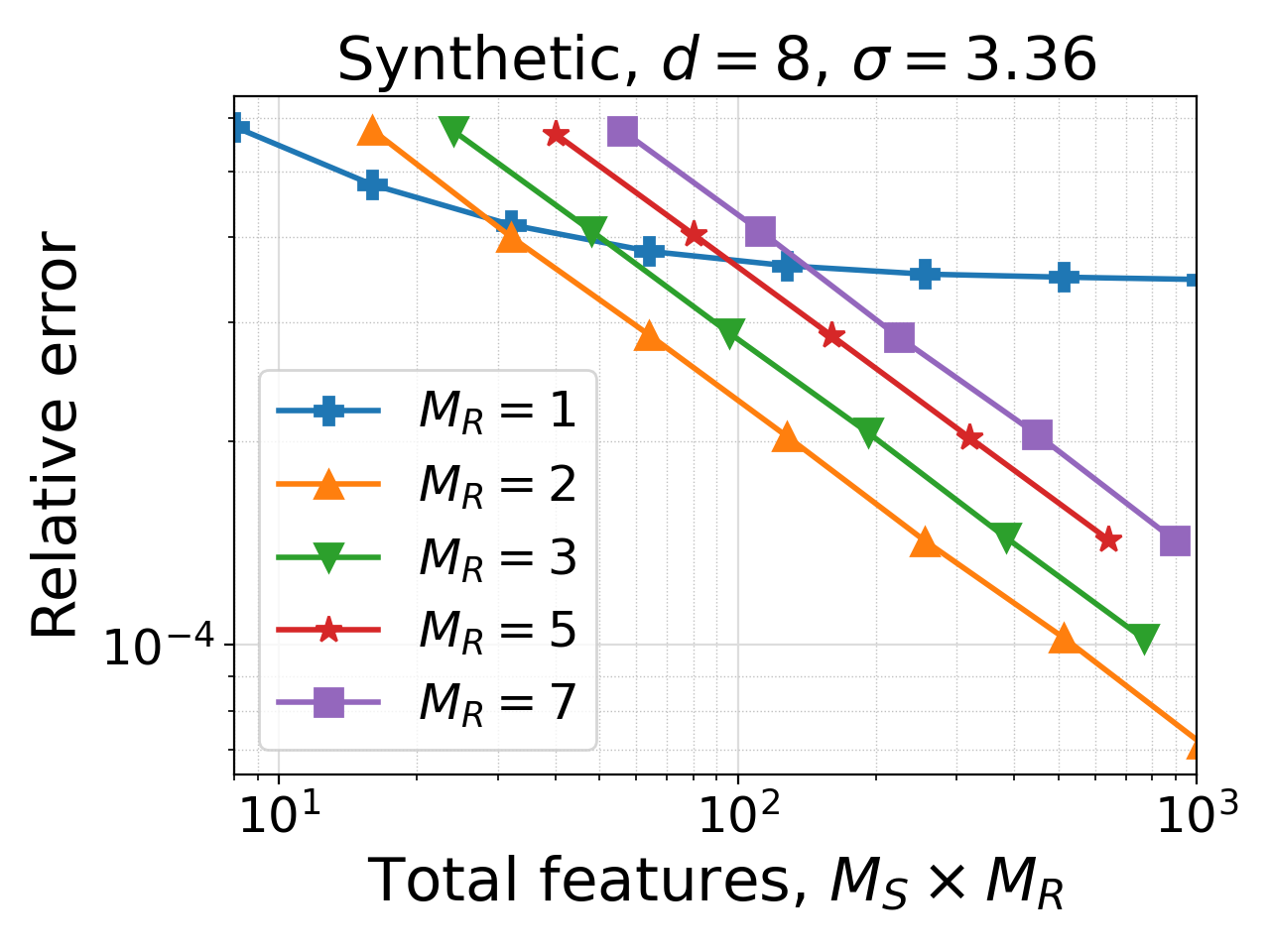}
\includegraphics[width=\linewidth]{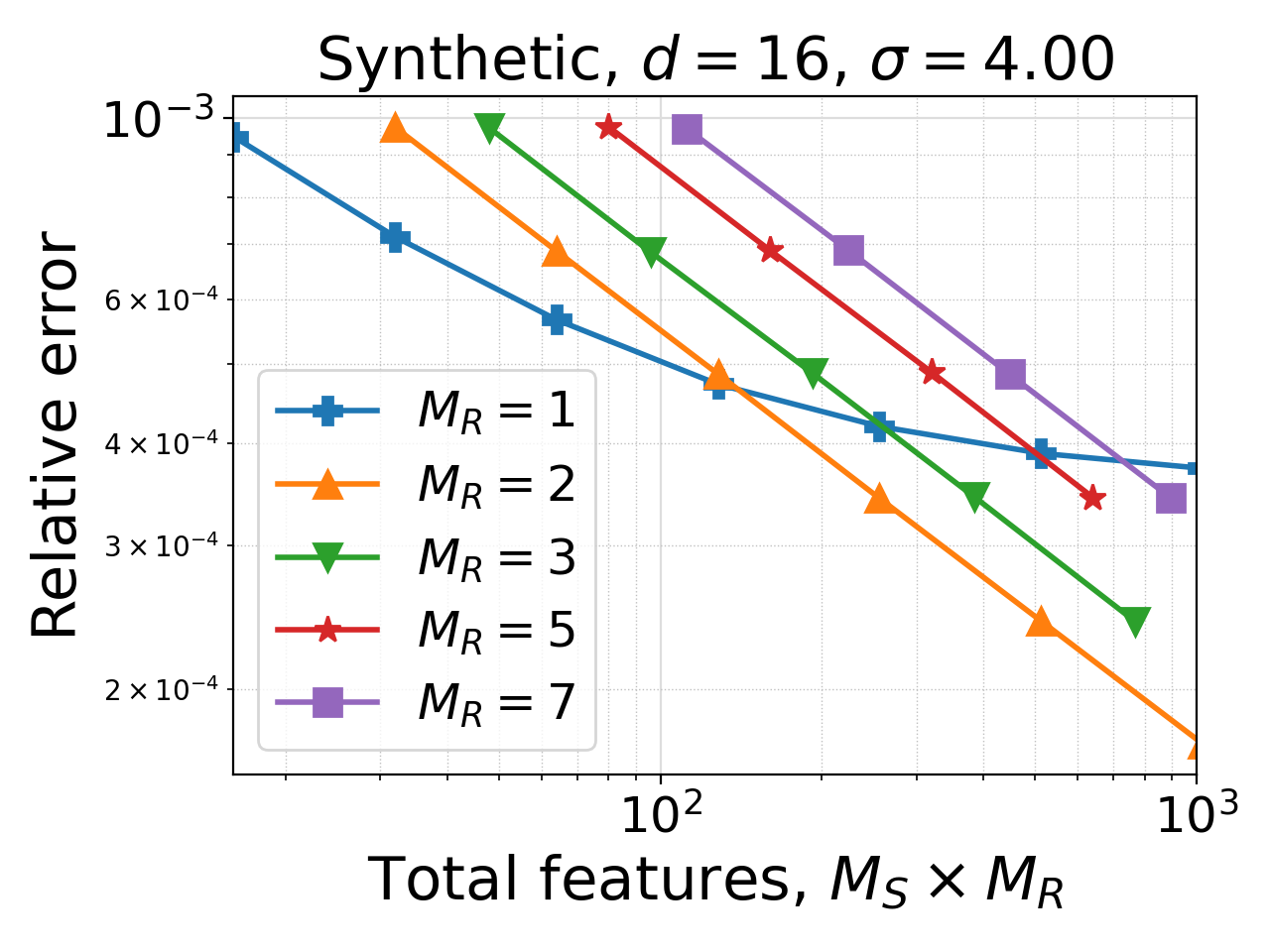}
\includegraphics[width=\linewidth]{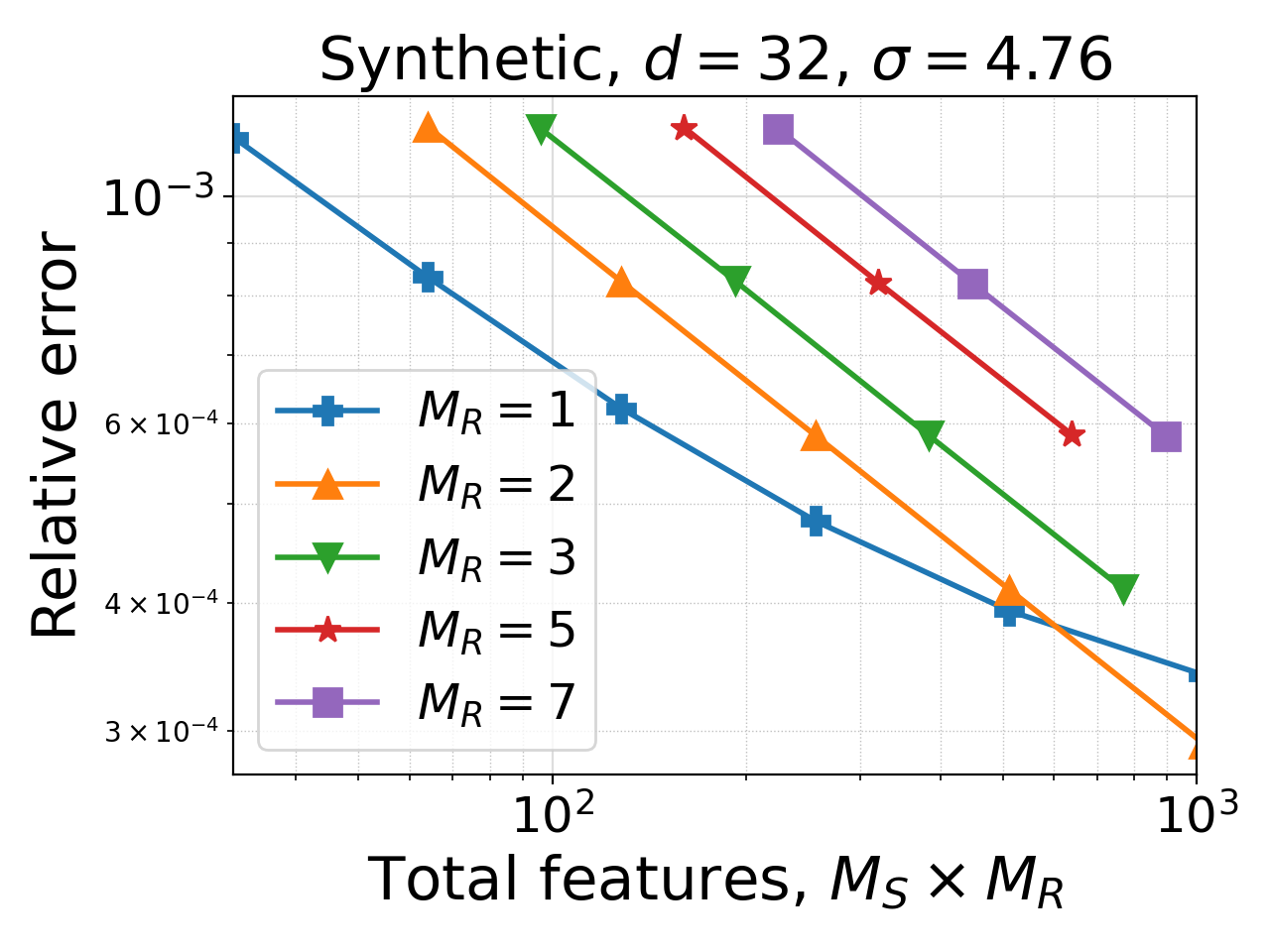}
\end{minipage}\hfill
\begin{minipage}{0.33\textwidth}
\includegraphics[width=\linewidth]{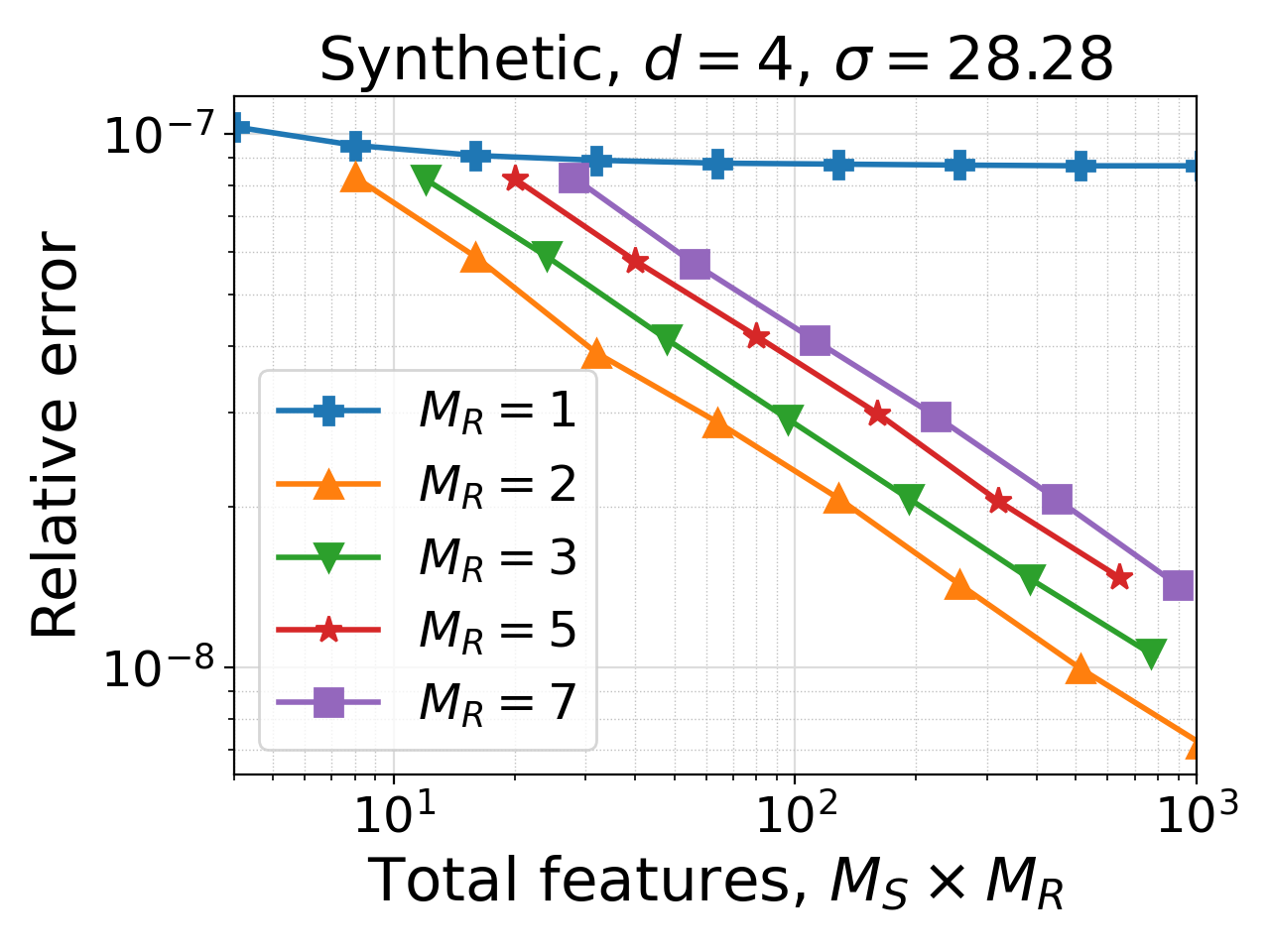}
\includegraphics[width=\linewidth]{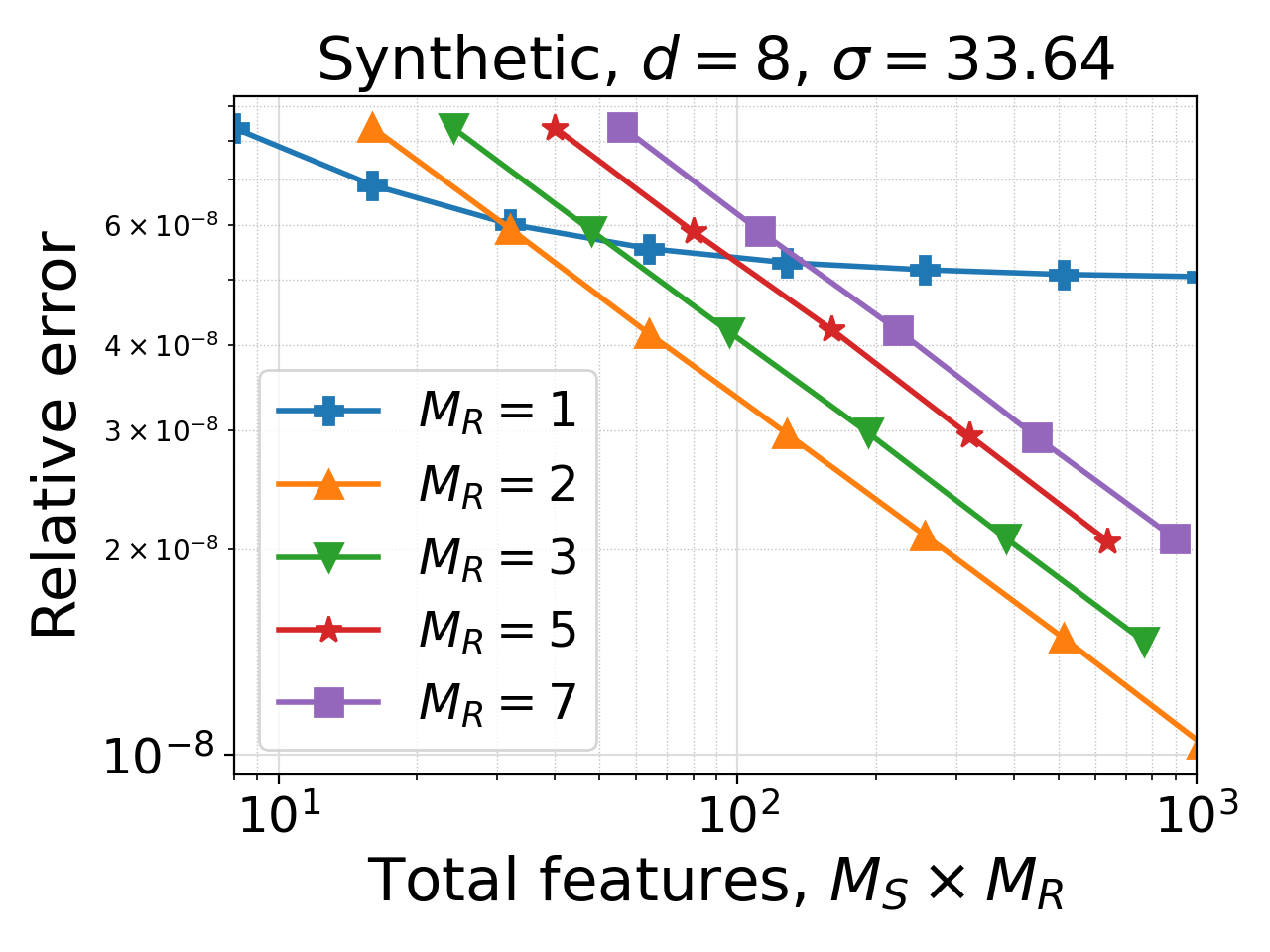}
\includegraphics[width=\linewidth]{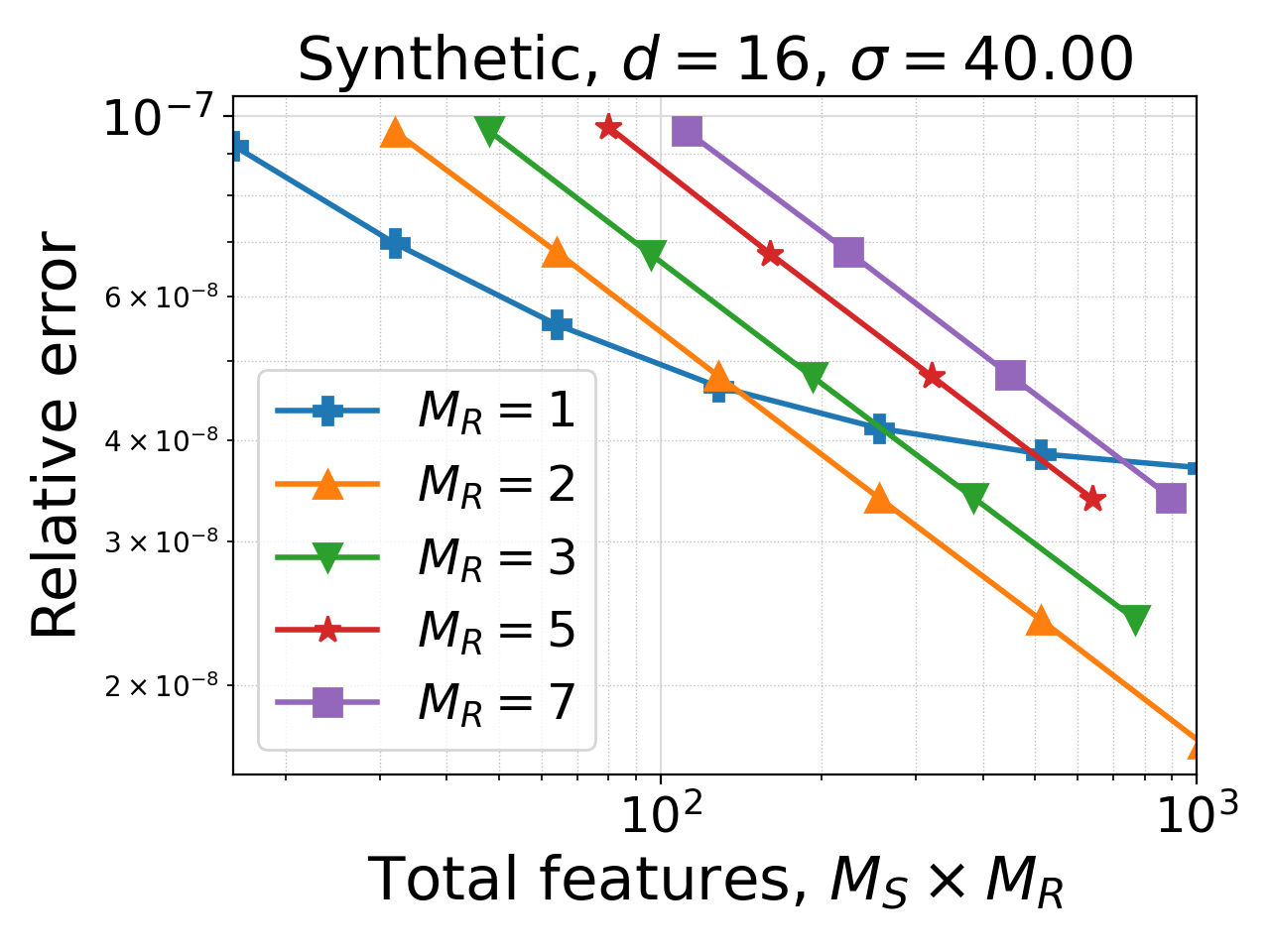}
\includegraphics[width=\linewidth]{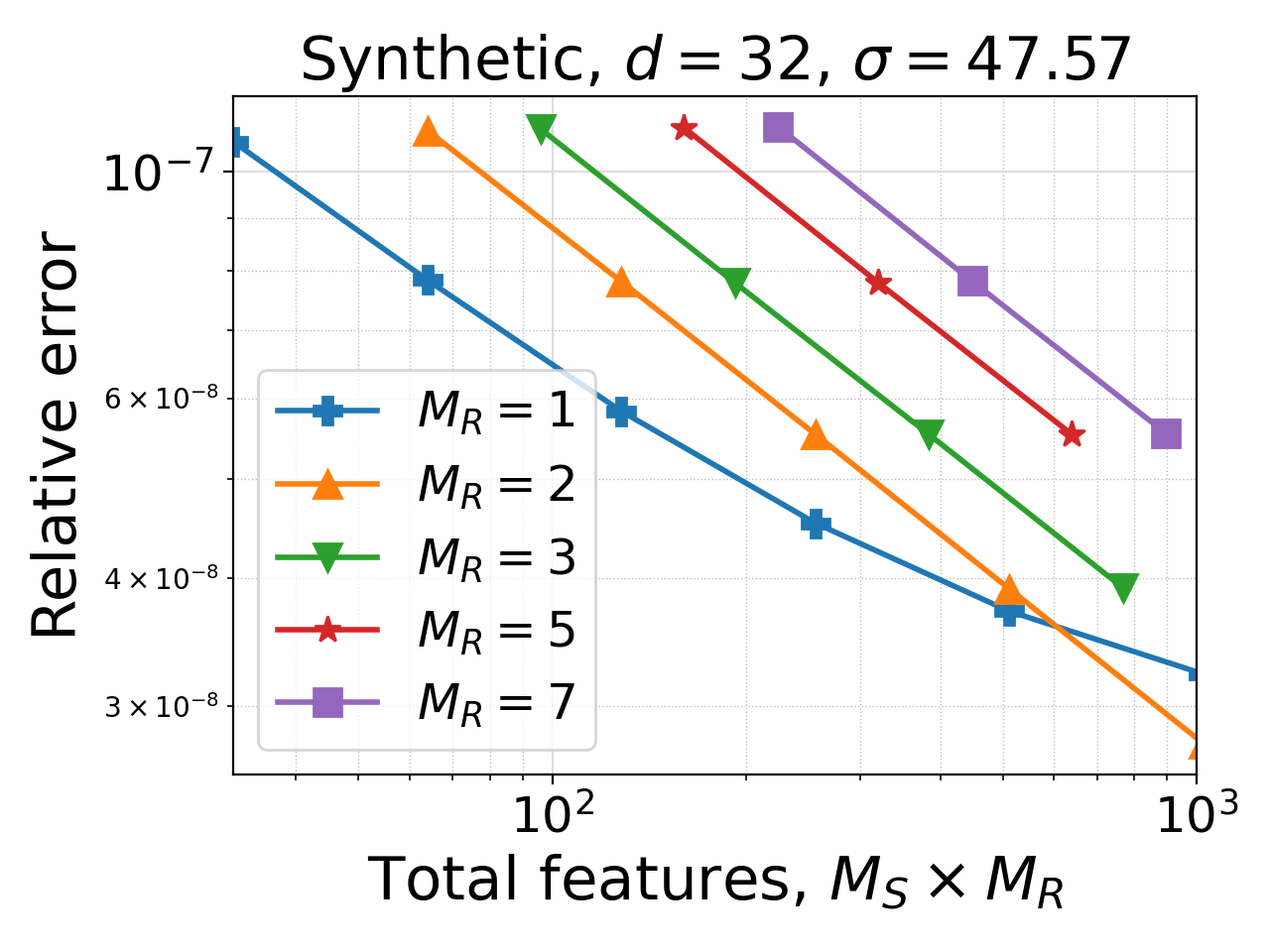}
\end{minipage}
\caption{The approximation error w.r.t. total number of features under different combination of kernel bandwidth and dataset dimension. From top to bottom, dimension increases from 4 to 32; from left to right, kernel bandwidth increases from $0.4*d^{1/4}$ to $20*d^{1/4}$.}
\label{fig:radial_nodes_analysis}
\end{figure}

\begin{figure} 
\hfill small $\sigma$ \hfill\hfill medium $\sigma$ \hfill\hfill large $\sigma$ \hfill{} 
\medskip

\begin{minipage}{0.33\textwidth}
\includegraphics[width=\linewidth]{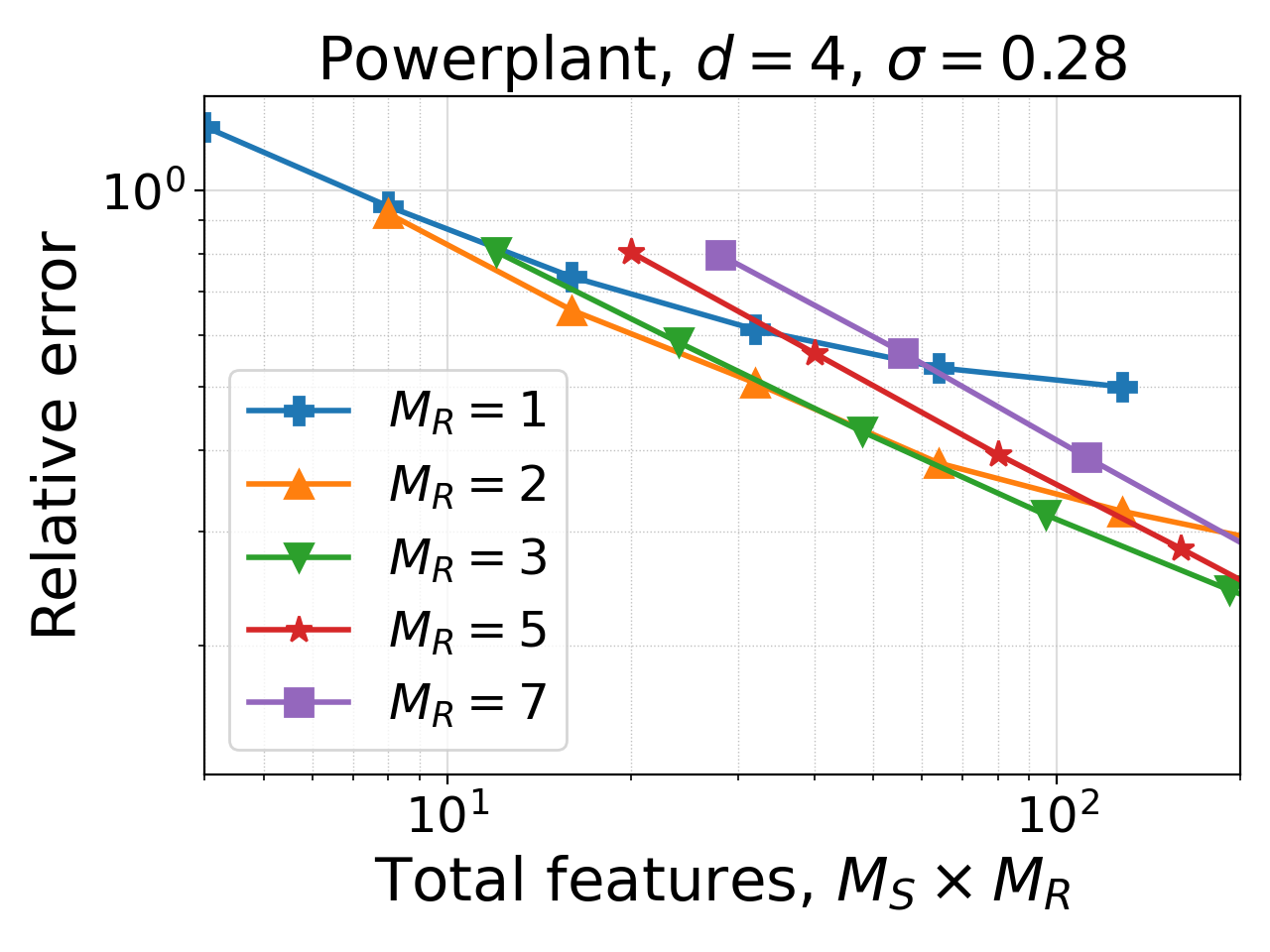} 
\includegraphics[width=\linewidth]{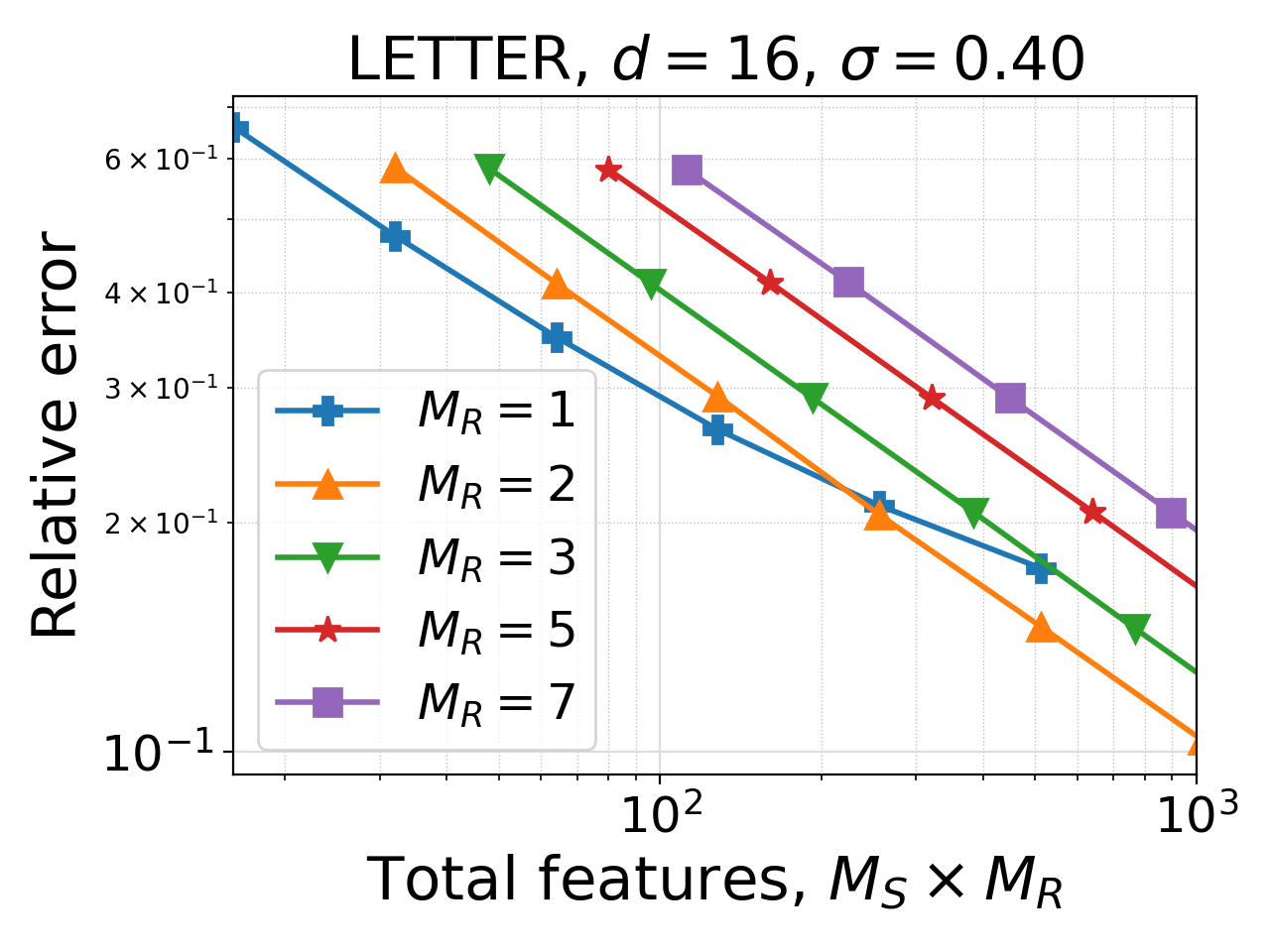} 
\includegraphics[width=\linewidth]{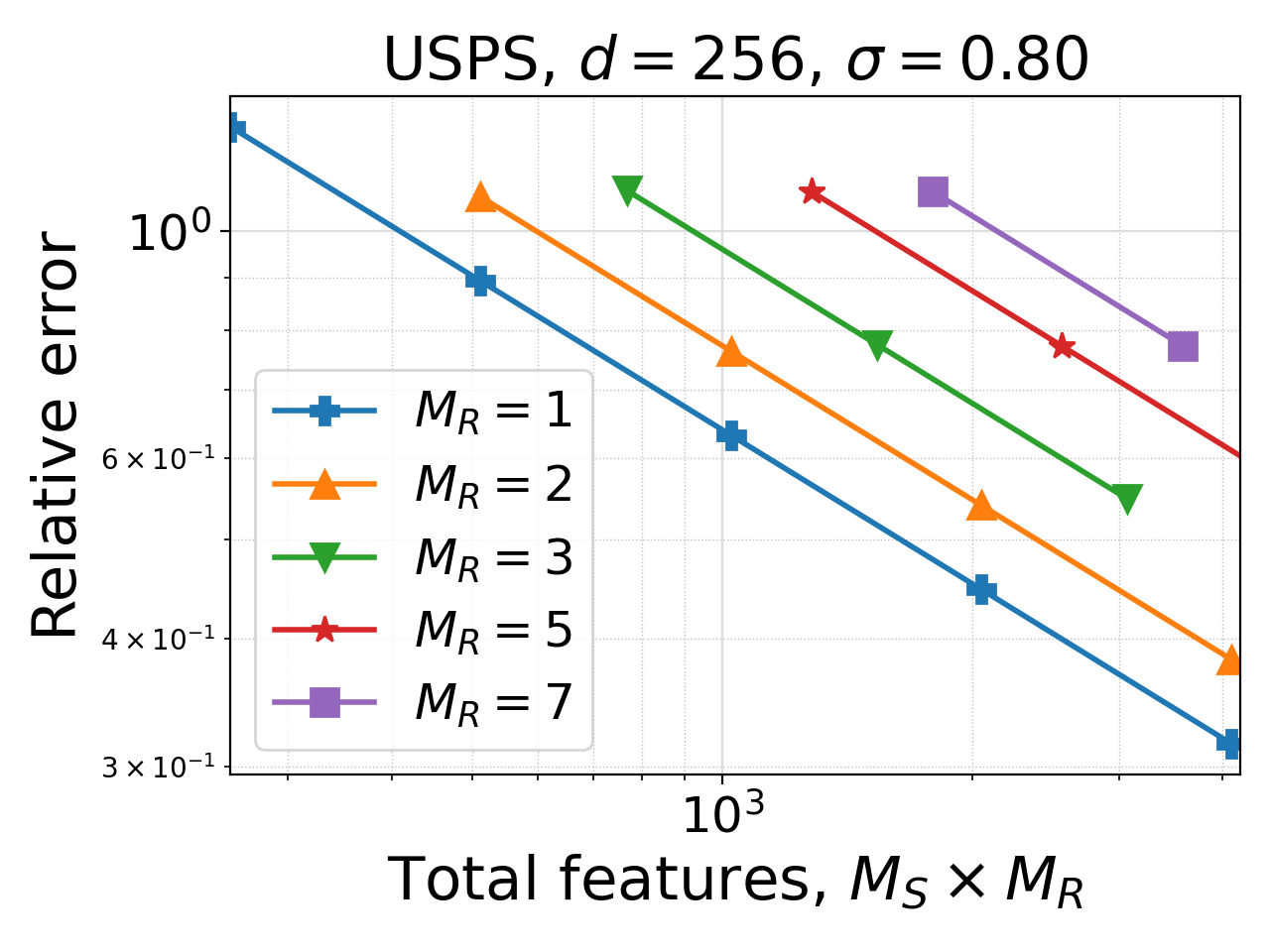} 
\includegraphics[width=\linewidth]{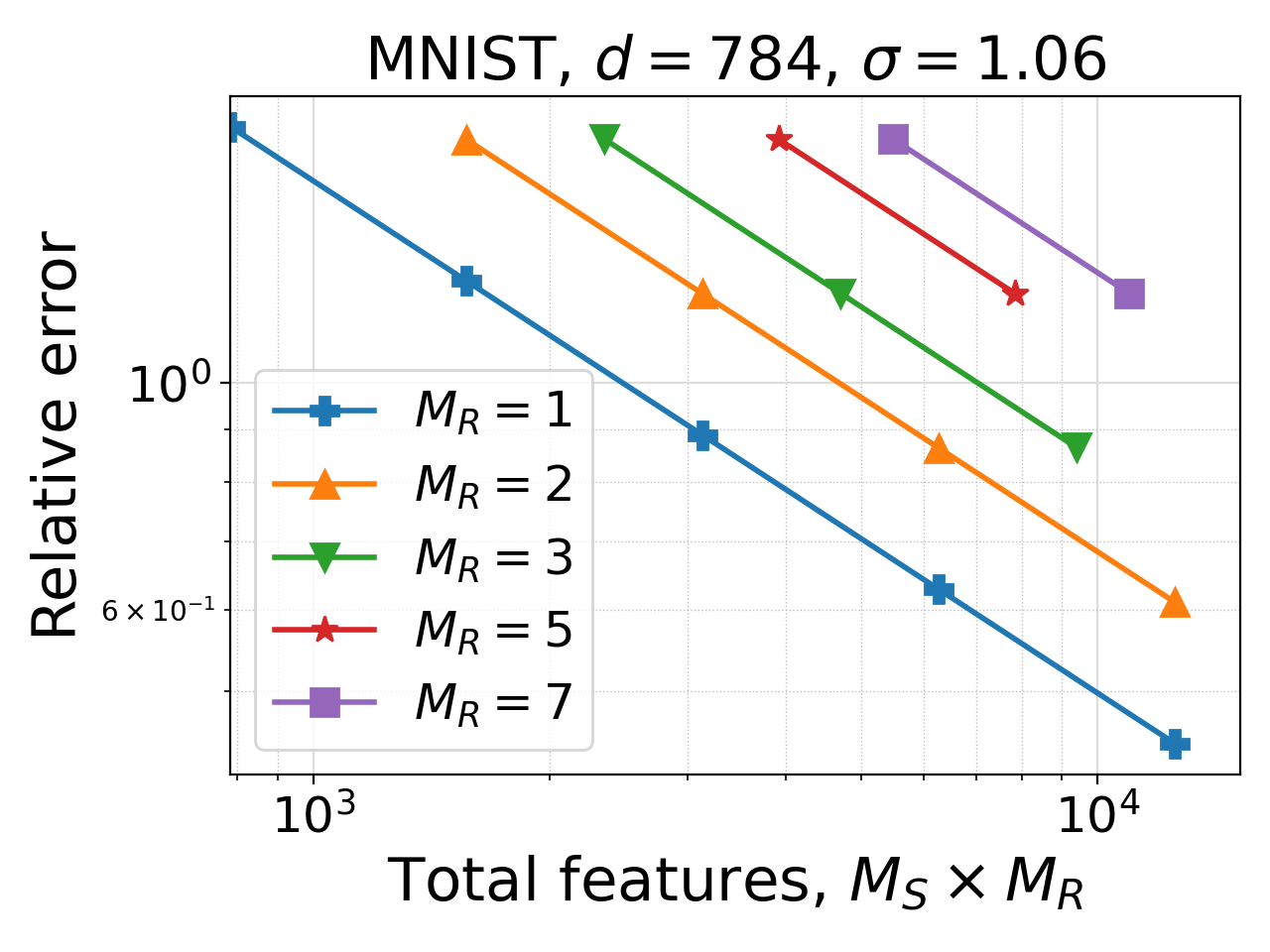}
\end{minipage}\hfill
\begin{minipage}{0.33\textwidth}
\includegraphics[width=\linewidth]{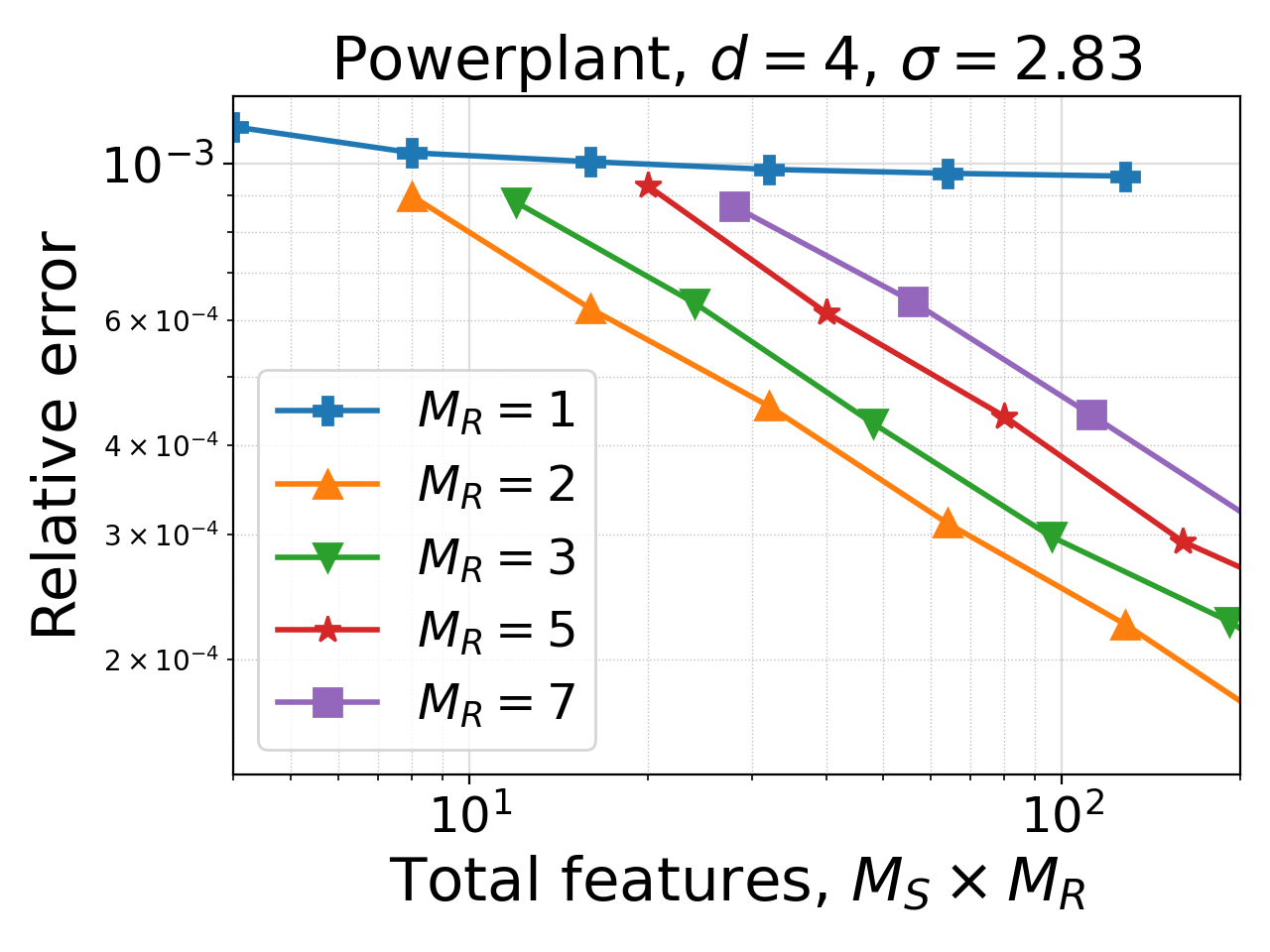}
\includegraphics[width=\linewidth]{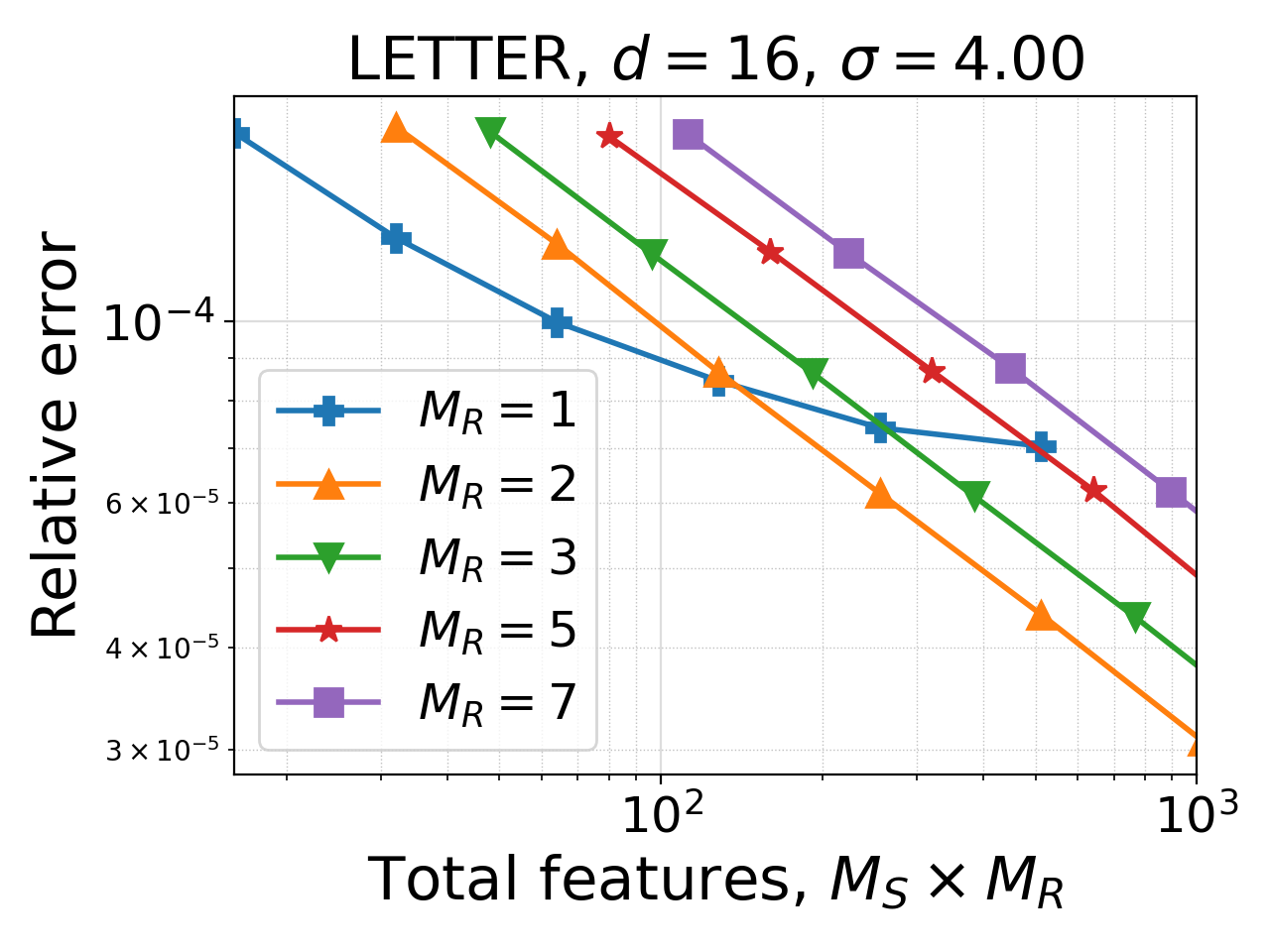}
\includegraphics[width=\linewidth]{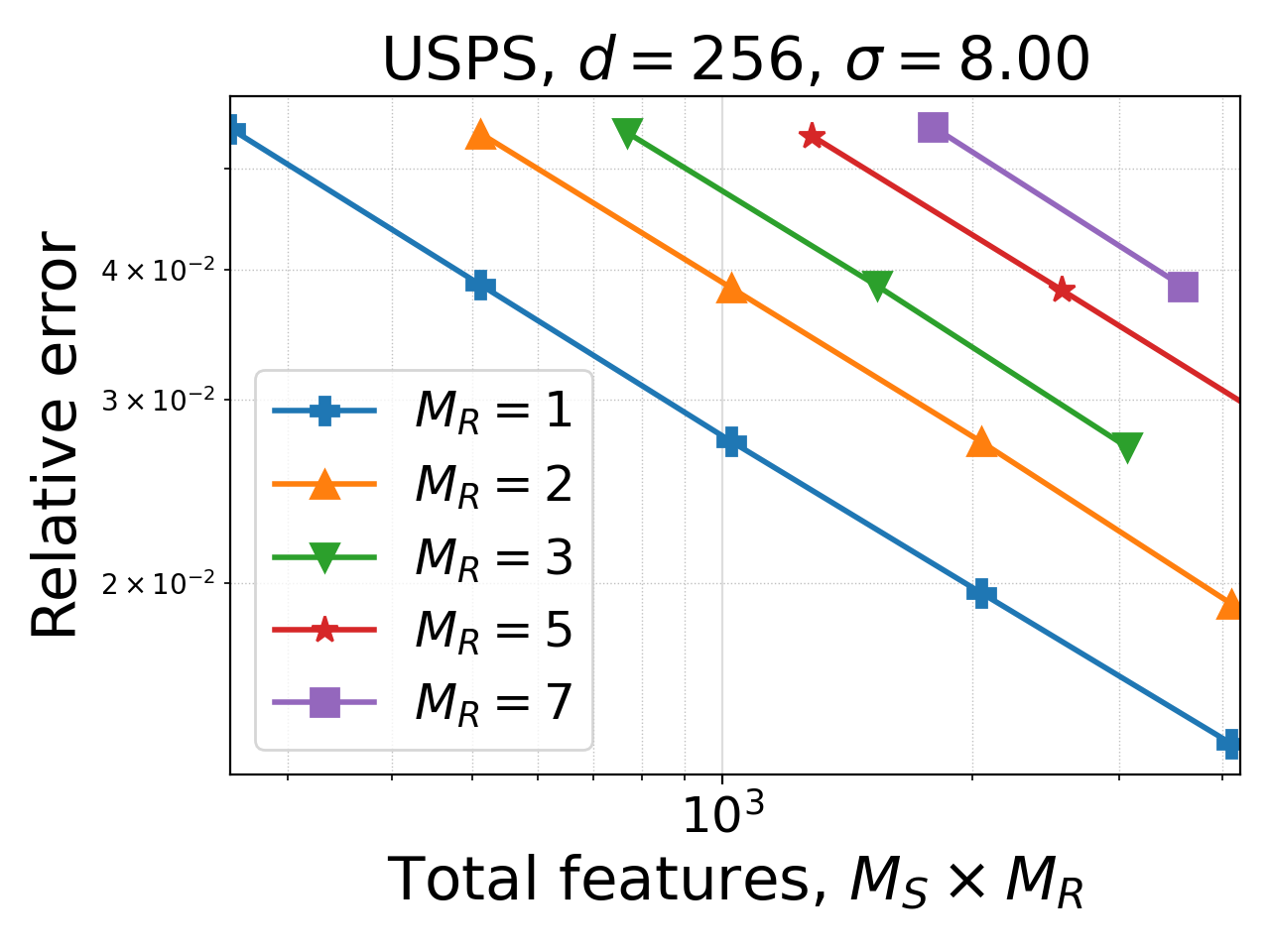}
\includegraphics[width=\linewidth]{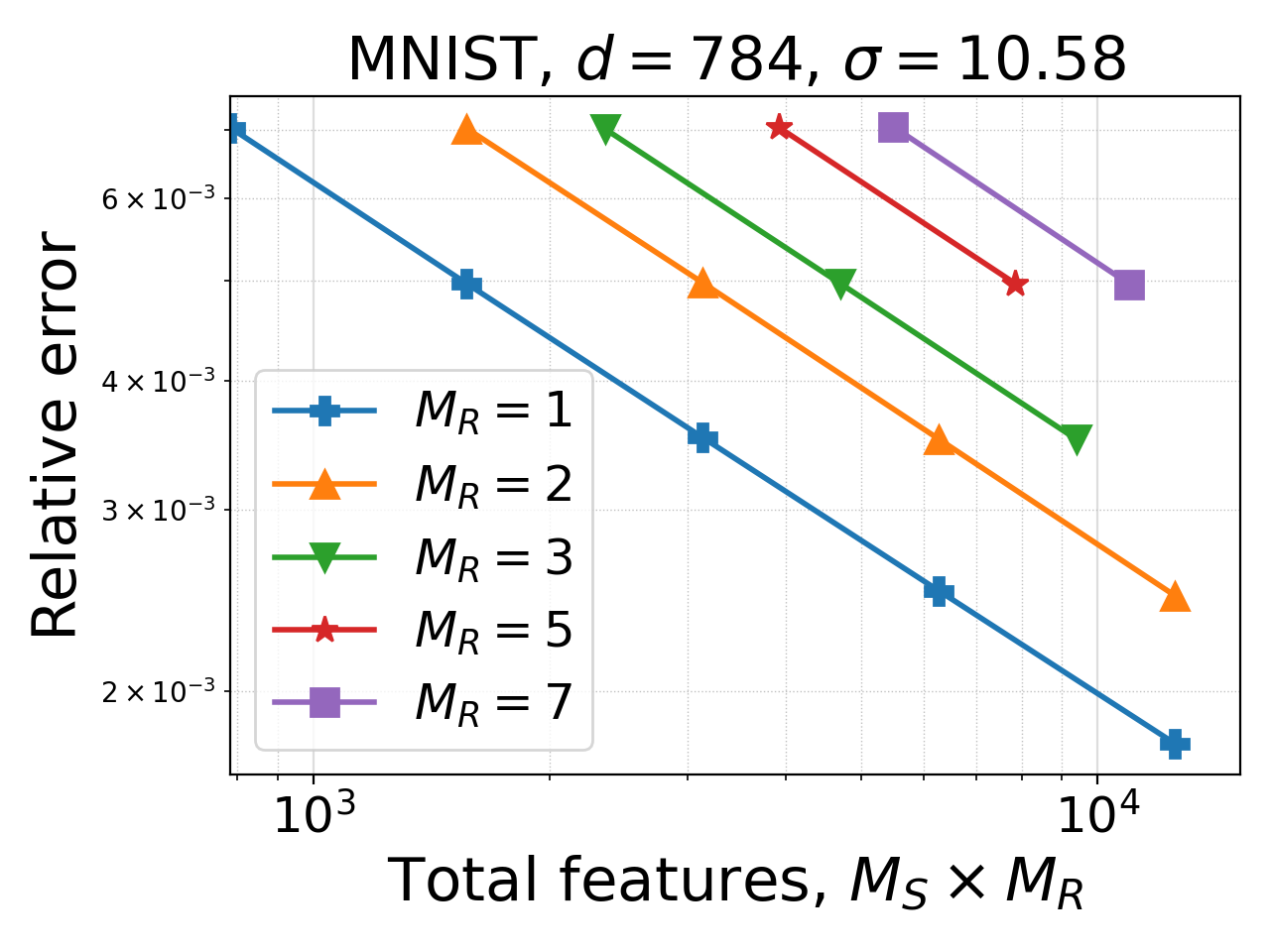}
\end{minipage}\hfill
\begin{minipage}{0.33\textwidth}
\includegraphics[width=\linewidth]{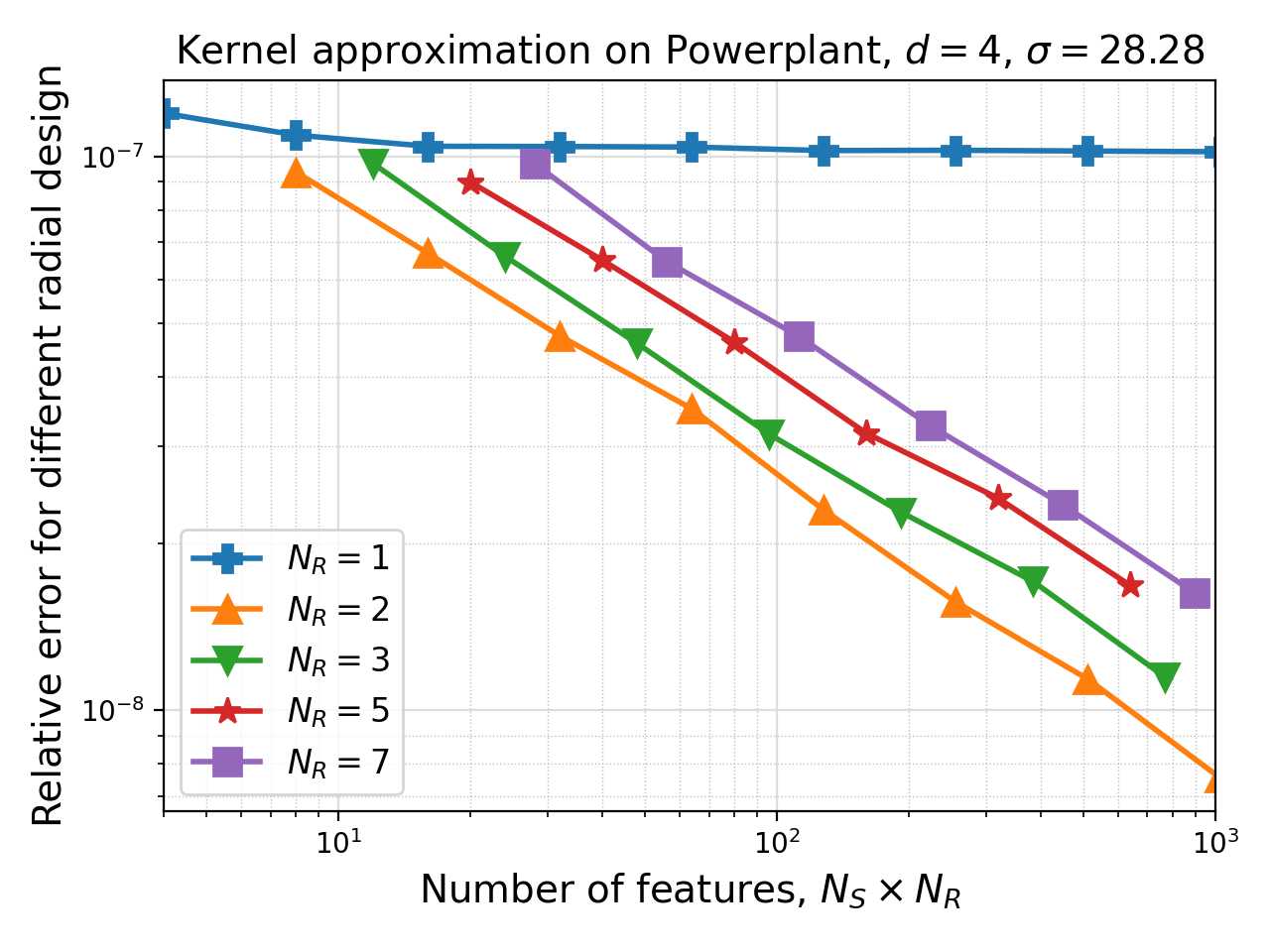}
\includegraphics[width=\linewidth]{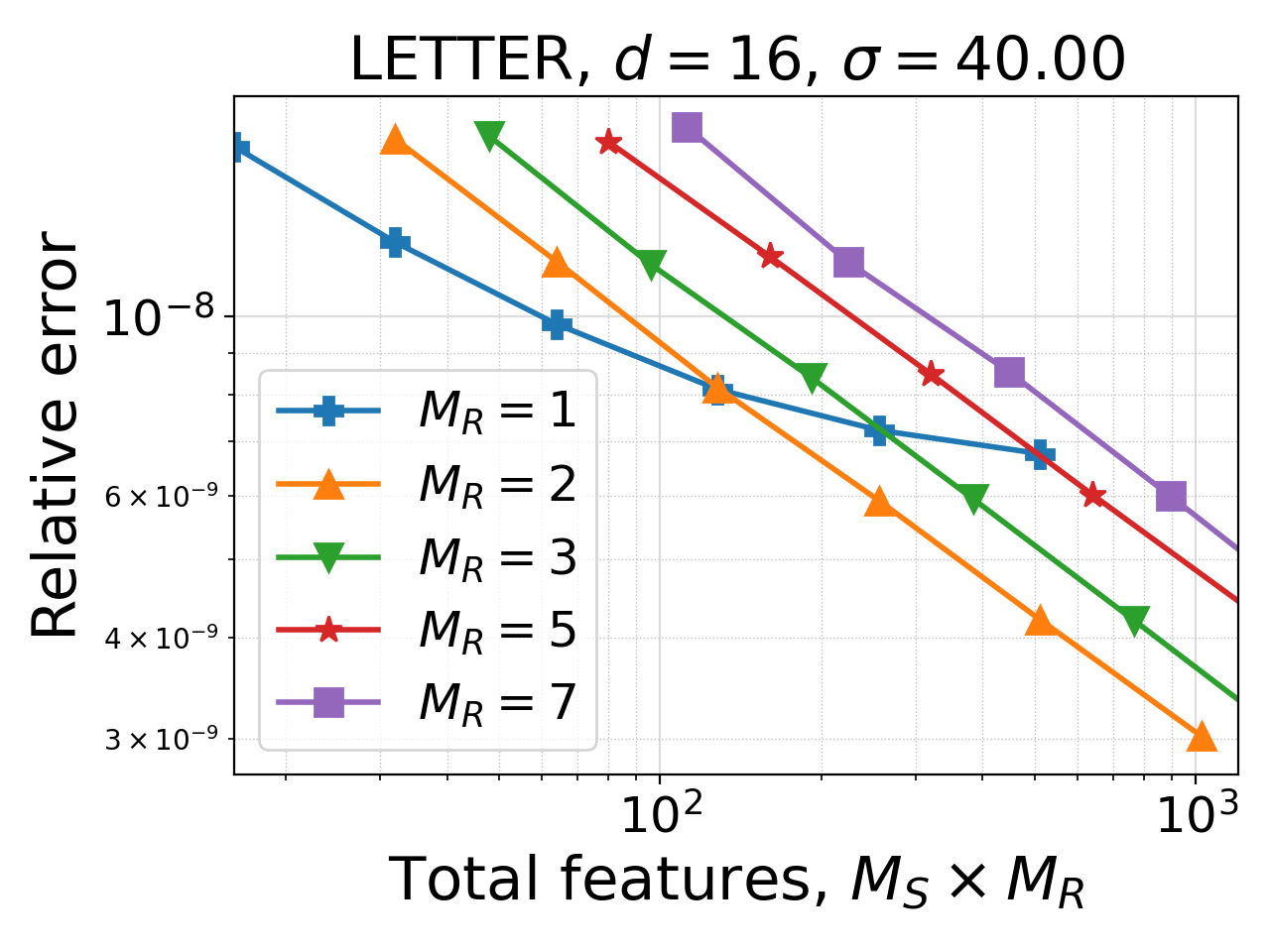}
\includegraphics[width=\linewidth]{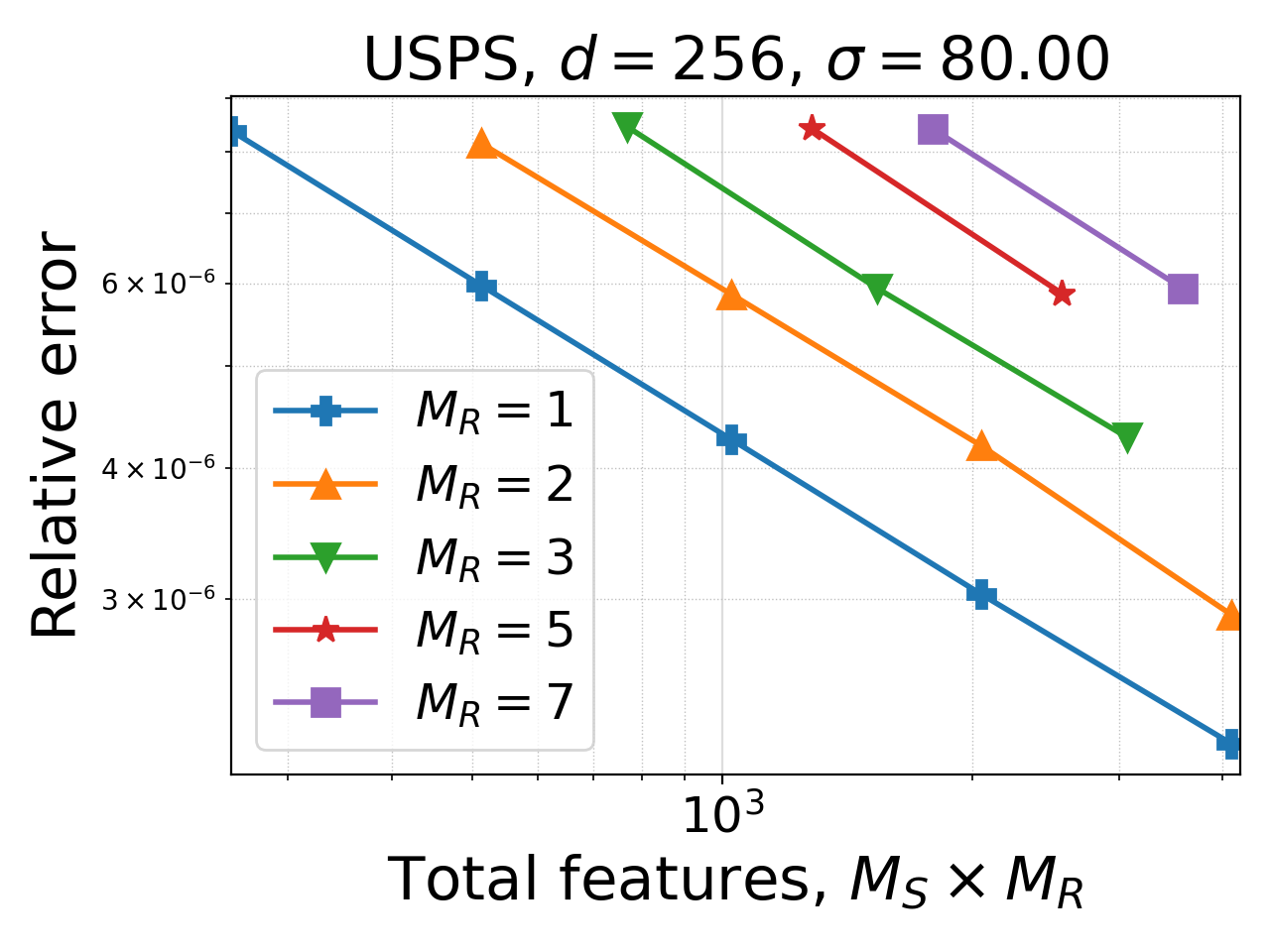}
\includegraphics[width=\linewidth]{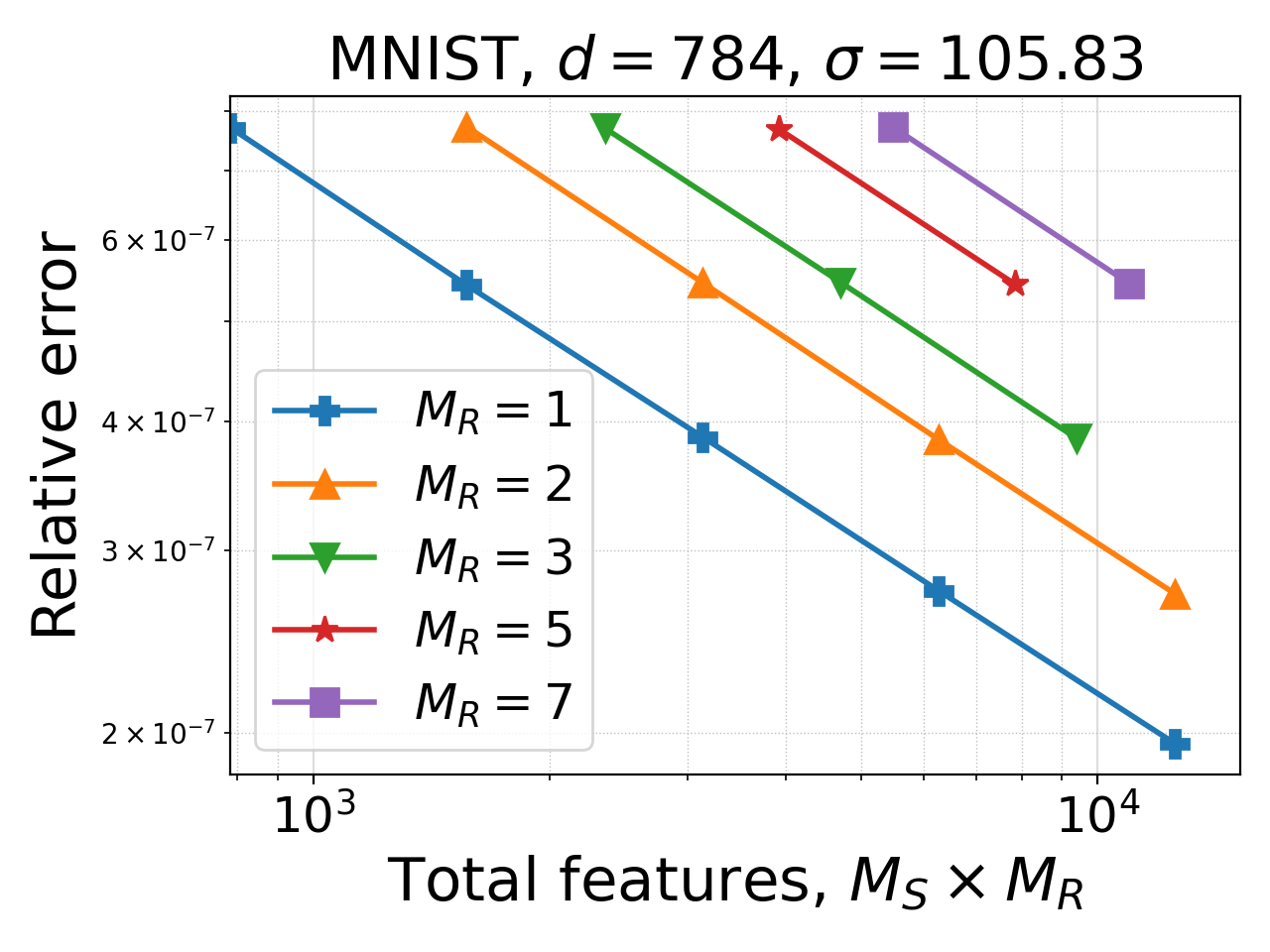}
\end{minipage}
\caption{The approximation error w.r.t. total number of features under different combination of kernel bandwidth and dataset dimension. From top to bottom are 4 real datasets.}
\label{fig:radial_nodes_analysis_real}
\end{figure}

\subsection{Maximum error and dataset diameter}
Figure~\ref{fig:diameter_analysis} shows the maximum error of a dataset of diameter $D$. The synthetic dataset consists of 1k uniform samples from the hypersphere of radius $D/2$. As we can see from the experiment, the error given by our spherical-radial algorithm (SR-OKQ-SOMC) with optimal kernel quadrature and symmetrized orthogonal Monte Carlo nodes (details explained in \Cref{sec:num_sim_1}) is smaller than other methods by a large margin and admits modest increase with dataset diameter when $d=4$. Without optimal kernel quadrature, our method SR-OMC behaves similarly to SSR, followed by QMC and ORF. The vanilla RFF has much bigger error.

In higher dimension, SR-OMC and RFF have the smallest error, followed by SR-OKQ-SOMC and ORF. Although QMC has $1/N$ convergence rate in theory, it has inferior behavior compared with other methods until the number of features is beyond a threshold, which hard to achieve in practice.

In the experiment we set the spherical kernel bandwidth $\sigma_{\mathbb{S}} = 1$ in SR-OKQ-SOMC for $d=4$ and $d=16$. This parameter can be tuned bigger in high dimension to improve the algorithm performance, and we leave the optimal selection of $\sigma_{\mathbb{S}}$ for future work.

\begin{figure}[h]
    \begin{subfigure}[b]{.49\linewidth}
  \centering
    \includegraphics[width=.99\linewidth]{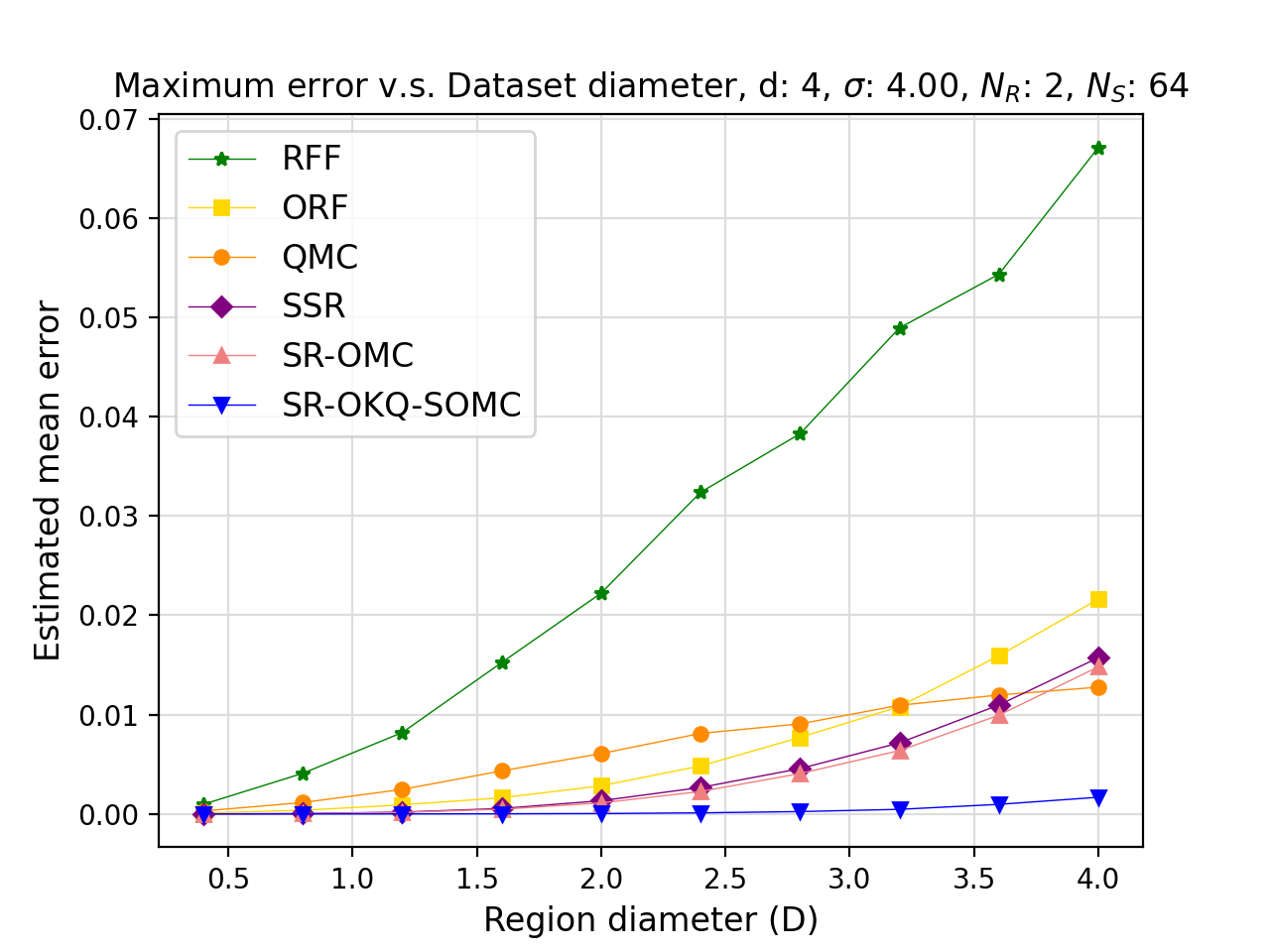}
    \vspace{-10pt}
    \caption{$d = 4$}
    \label{fig:max_err_d=4}
  \end{subfigure} 
  \begin{subfigure}[b]{0.49\linewidth}
  \centering
    \includegraphics[width=.99\linewidth]{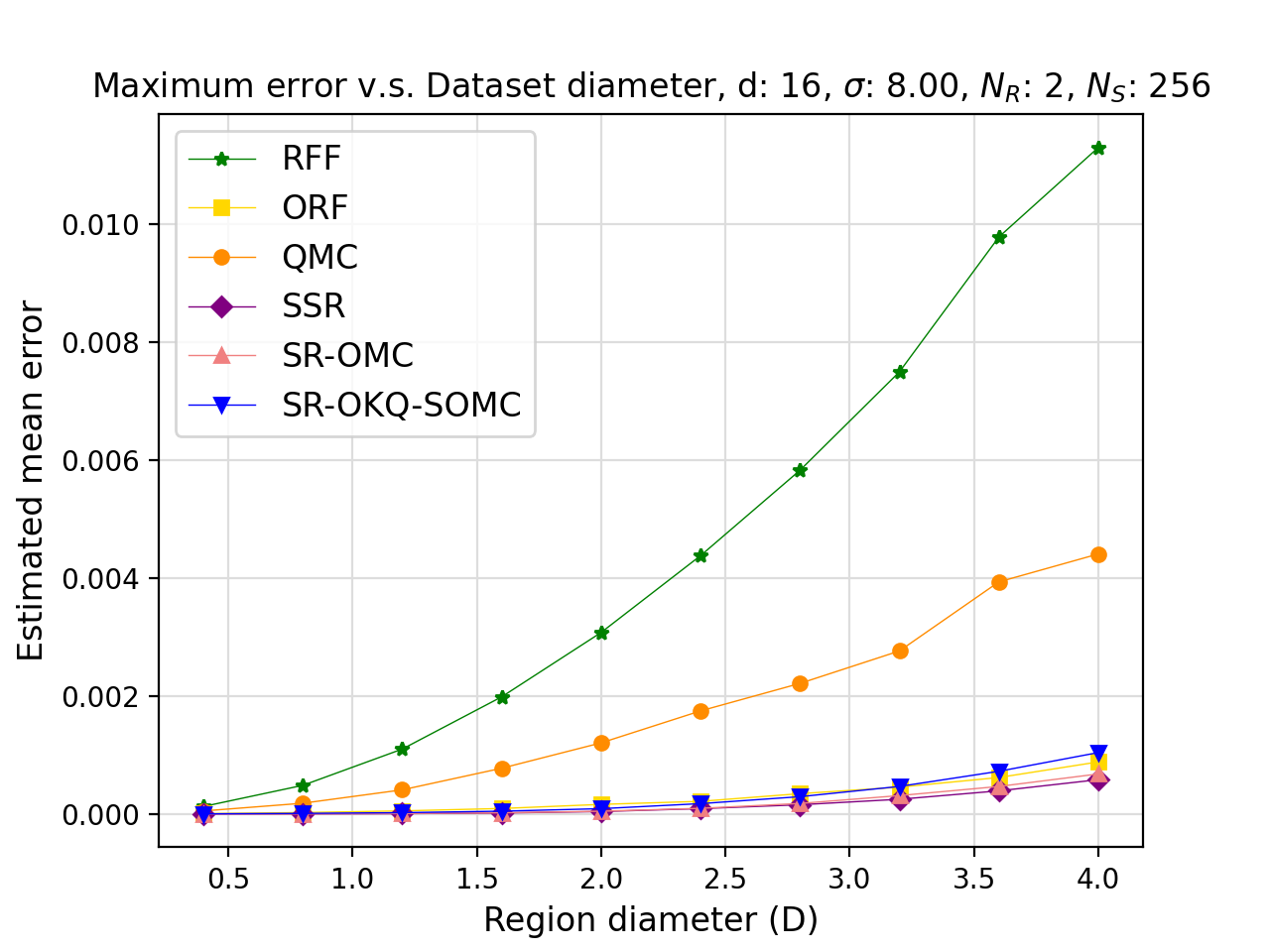}
    \vspace{-10pt}
    \caption{$d = 16$}
    \label{fig:max_err_d=16}
  \end{subfigure} 
\caption{Relation between maximum error and dataset diameter for different kernel approximation schemes. The maximum approximation error is computed over $10^2$ uniformly-ditsributed samples on the sphere of radius $D$. \label{fig:diameter_analysis}}
\end{figure}
\newpage

\end{document}